\documentclass[english,12pt]{cse-ucdenver-thesis-report}
\usepackage[T1]{fontenc}
\usepackage[utf8]{inputenc}
\usepackage{textcomp}
\usepackage{url}
\usepackage{paralist}
\usepackage{xparse}
\usepackage{enumitem}
\usepackage[table,xcdraw]{xcolor}
\usepackage{hyperref}
\usepackage{lipsum}
\usepackage{placeins}

\usepackage{ctable}
\usepackage{floatrow}
\usepackage{subcaption}

\usepackage{makecell}

\usepackage[numbers,sort&compress]{natbib}

\usepackage{amsmath,amssymb,amsfonts,mathrsfs}

\usepackage{amsthm}
\newtheorem{mylemma}{Lemma}
\newenvironment{lemma}{\begin{mylemma}}{\end{mylemma}}
\newtheorem{proposition}{Proposition}

\usepackage{subfiles}
\usepackage{graphicx}

\usepackage{amsmath,amssymb}

\usepackage{booktabs}

\usepackage{array}

\usepackage[ruled,vlined]{algorithm2e}

\usepackage{datetime}

\makeatletter
\@ifundefined{proglang}{%

  \newcommand\code{\bgroup\@makeother\_\@makeother\~\@makeother\$\@codex}
  \def\@codex#1{{\normalfont\ttfamily\hyphenchar\font=-1 #1}\egroup}
}
\makeatother

\usepackage{fancyhdr}
\usepackage{datetime}

\input{metainfo} 

\ucdabstract{

Dynamical systems describe how a physical system evolves over time. Physical processes can evolve faster or slower in different environmental conditions. We use time-warping as rescaling the time in a model of a physical system. This thesis proposes a new method of transfer learning for Recurrent Neural Networks (RNNs) based on time-warping. We prove that for a class of linear, first-order differential equations known as time-lag models, an LSTM can approximate these systems with any desired accuracy, and the model can be time-warped while maintaining the approximation accuracy. 

The Time-Warping method of transfer learning is then evaluated in an applied problem on predicting fuel moisture content (FMC), an important concept in wildfire modeling. An RNN with LSTM recurrent layers is pretrained on fuels with a characteristic time scale of 10 hours, where there are large quantities of data available for training. The RNN is then modified with transfer learning to generate predictions for fuels with characteristic time scales of 1 hour, 100 hours, and 1000 hours. The Time-Warping method is evaluated against several known methods of transfer learning. The Time-Warping method produces predictions with an accuracy level comparable to the established methods, despite modifying only a small fraction of the parameters that the other methods modify. 

}

\advisorapprovalpage{true}
\copyrightpage{true}  
\acknowledgements{

The research presented in this paper was partially funded by NASA grants 80NSSC22K1717, 80NSSC23K1118, 80NSSC22K1405, 80NSSC23K1344, and 80NSSC25K7276. Special thanks to my research collaborators at the Wildfire Interdisciplinary Research Center (WIRC) at San Jose State University, and the Cooperative Institute for Research in the Atmosphere (CIRA) at Colorado State University.

Computing resources were provided by the University of Colorado Denver supercomputing cluster Alderaan, partially funded by NSF Campus Cyberinfrastructure grant OAC-2019089.

A portion of this work used code generously provided by Brian Blaylock's Herbie Python packages Herbie (Version 2024.8.0) (\url{https://github.com/blaylockbk/SynopticPy}) and SynopticPy  (\url{https://github.com/blaylockbk/SynopticPy}). 

Large Language Models, specifically ChatGPT-5, were utilized for support with coding, literature search, and writing revisions.

}

\begin{document}
    \chapter{\uppercase{Introduction}} \label{chapter:01}

Many physical systems change over time, and these changes often follow patterns that reflect how quickly the system responds to new conditions. The characteristic time scale of a system refers to the typical amount of time over which the system responds to changes in its conditions. It is often useful from a modeling perspective to modify the time scale to represent different dynamics. In this thesis, we refer to the modification of the time scale of the dynamics as ``time-warping". 

Time-warping a physical system is used for the purposes of empirically fitting to data, where the observed rate of change of a system is different from what the model predicts and you need to adjust the dynamics to match reality. Time-warping is also useful when modeling a related system. It is common in applied time series modeling to assume that two different systems are described by the same underlying dynamics, with only the difference that one system is faster or slower than the other. Mathematical models attempt to identify dynamics that are invariant to different temporal changes. In a machine learning (ML) framework, transfer learning refers to a set of methods where aspects of one problem are applied to a related problem. Time-warping a dynamic system can be framed as a form of transfer learning. 

Traditional differential equations model the rate of change of a system. The dynamics can be directly sped-up or slowed-down. Modern ML methods for sequence learning typically learn a discrete update rule of the system. Recurrent neural networks (RNNs), including the popular Long Short-Term Memory (LSTM), are trained on a certain discrete time scale. These models are often viewed as black-boxes, and the characteristic time scale of the system is not obvious and not amendable to modification. 

In this thesis, we will show how the learned dynamics of an RNN can be time-warped. Under a certain set of assumptions about the underlying system and the RNN that represents it, this can be proven mathematically. For a class of first-order linear differential equations known as time-lag models, we prove that an LSTM can approximate the systems with arbitrary accuracy. Then for the type of temporal rescaling of the dynamic system that we define as a time-warp, we prove that specific parameters in the LSTM can be modified to approximate the time-warped system. The proofs provide the theoretical foundation of a practical method for time-warping a pretrained RNN. The time-warping method will be demonstrated in a simulation with idealized dynamics. Then, the time-warping method will be evaluated statistically in an applied transfer learning problem. 

The motivating application of this thesis is to improve the accuracy of modeling fuel moisture content (FMC) for fuels that have sparse observations. FMC is an important concept in wildfire science that affects wildfire risk and active wildfire behavior. Fuels of different sizes respond to environmental conditions differently, and the FMC evolution over time is faster or slower depending on the size of the fuel particles. Large quantities of data exist for fuels with a characteristic response time of 10 hours (10h fuels), but the observations for fuels with other response times are sparse. Fuels with characteristic response times of 1h, 100h, and 1000h are studied by researchers and used in operational wildfire modeling tools, but observations of these fuels are low resolution in both space and time. In this thesis, we will use an RNN with LSTM recurrent layers that was pretrained on observations from 10h fuels. We will apply a time-warping technique to generate FMC predictions for 1h, 100h, and 1000h fuels, then evaluate this method against known transfer learning techniques.

The pretrained RNN with LSTM recurrent layers was developed for 10h fuels using observations from a network of standardized sensors across the Continental United States (CONUS). The model was evaluated extensively and compared to several ML baselines~\citep{Hirschi-2026-RNN}. A small-scale study from 1996-1997 in Oklahoma has been extensively relied upon in the FMC modeling literature for predicting the FMC of 1h, 100h, and 1000h fuels~\citep{Carlson-2007-ANM}. Time-warping the pretrained RNN will be evaluated and compared to established methods within transfer learning using this influential FMC data set. 

This thesis fills several gaps within the RNN literature. 

First, formally framing time-warping as a transfer learning problem has not been studied to the best of our knowledge. Approaching time-warping as a transfer learning problem helps advance the field of transfer learning with RNNs. Then, while there has been some research into how RNNs learn a characteristic time scale, to our knowledge there exists no literature on post-hoc modification of the learned dynamics of an RNN. That is, taking a pretrained RNN and attempting to modify the learned temporal dynamics is a novel concept with interesting potential applications. The time-warping framework leverages knowledge about temporal dynamics to enable more effective transfer learning.

This thesis also contributes to improving the interpretability of RNN-based models. We develop a theoretical result showing that a class of dynamical systems can be approximated with an LSTM and then time-warped to represent different time scales with any desired accuracy. Directly modifying specific weights inside black-box RNNs based on theoretical considerations is rare within machine learning. Further, we use mathematical properties of the LSTM to develop theoretical expectations for how certain weights may change during transfer learning. Developing a priori expectations for how neural network parameters may respond to new observations can help build a deeper understanding of RNN behavior and clarify some aspects of their black-box nature.

Lastly, this thesis advances FMC modeling by introducing a new method of transfer learning that is highly applicable to this field. In recent years, ML techniques are beginning to augment or improve upon traditional physics-based approaches to predicting FMC. Physically-based models can be used to make informed predictions on how fuels will behave when there are sparse observations, while ML models tend to struggle when there are not large quantities of data to learn patterns from. Transfer learning has the potential to greatly improve FMC modeling in cases where observations are sparse. The time-warping method will be shown to be accurate compared to established transfer learning techniques, despite only modifying a small fraction of the parameters that the other techniques modify. 

All computational results in this thesis can be reproduced using publicly  available code. See Appendix~\ref{app:code} for more details.

    \chapter{\uppercase{Background}}\label{chapter:02}

\section{Time-Warping Dynamic Systems}

Throughout this work, the term ``time series'' refers to discrete-time models and data. The term ``dynamical systems'' refers to models defined by an explicit evolution operator, most commonly in continuous time. Differential equations are a subset of dynamical systems that are defined in continuous time and describe instantaneous rates of change. In practice, differential equations are discretized with respect to some temporal grid. Digital technologies like computers and measurement devices have led to the implementation of differential equations in discrete time. Many differential equations are formulated in a way that does not fix an intrinsic temporal resolution, allowing the same model to be evaluated on different time scales. A modeler can choose to run a differential equation on an hourly-resolution temporal grid or a daily-resolution grid or other resolutions, though this raises issues related to accuracy and stability. Whether the model accurately describes reality on those different scales is a separate question. Many statistical approaches to time series are defined in discrete time and describe the change in a system over a single time step. RNNs describe a discrete update to a system on a fixed temporal grid. Classical statistical approaches to time series work similarly, like autoregression (AR) and moving average methods. For statistical and ML models in discrete time, it is often intractable to deploy such models on a temporal grid with different spacing. Particularly with neural networks, which are often viewed as black boxes, there are few established methods that allow for this. 

Consider a general discrete-time model of the form $z_{t} = f(z_{t-1}, X_t)$ with initial condition $z_0$. Let $N$ be a positive integer, and define the discrete time index $t=0,1,\dots, N$. The model evolves according to the update rule for $t=1,\dots, N$. We denote the sequence $(z_t)_{t=1}^N$ and the inputs $(X_t)_{t=1}^N$. This indexing convention will be used throughout the thesis. A single time step might represent an hour. To run the model on daily resolution, you could generate $z_0, z_{24}, z_{48}, \dots$ by iterating the model forward 24 steps at a time. However, to evolve the system forward from the initial condition $z_0$ to $z_{24}$, you would need to input $X_1, X_2, \dots, X_{24}$. In practice, the input data might not be measured hourly, and then there is no obvious way to translate a one-step model into a 24-step model. So, an RNN trained on daily resolution data cannot be trivially translated to a produce hourly resolution predictions.

We denote a time-warp by $\gamma$ as a mapping $t \mapsto \gamma t$, for some positive real number $\gamma >0$. Consider a general first-order system in continuous time, 
\[
\frac{dy(t)}{dt} = f\!\left(y(t), X(t)\right)
\]
where $X(t)$ denotes external inputs. The function $f$ also has a set of parameters that are fixed over time. Define a time-warp by \(\tilde t = \gamma t\) with \(\gamma>0\), and identify
\[
\tilde y(\tilde t) := y(t), \qquad \tilde X(\tilde t) := X(t), \quad t=\tilde t/\gamma.
\]
Applying the chain rule,
\[
\frac{d\tilde y(\tilde t)}{d\tilde t}
= \frac{dy(t)}{dt}\frac{dt}{d\tilde t}
= \frac{1}{\gamma} f\!\left(\tilde y(\tilde t), \tilde X(\tilde t)\right).
\]

At this point, there are two mathematically equivalent ways to use this formulation that have different interpretations. Temporal rescaling treats $\tilde t$ as a new time variable with different units. If $t$ represents hours, then $\gamma = 24$ would mean that $\tilde t$ represents days. So, temporal rescaling fixes the dynamics but changes the units of time the dynamics are evaluated on. Alternatively, if we keep $t$ and $\tilde t$ as the same time scale, and rescale the system after the time-warp, then we have modified the dynamics of the original system while keeping time fixed. The first argument of function $f$ has been scaled by $1/\gamma$, with time being unchanged. If $\gamma >1$, we have sped up the system, and if $\gamma < 1$, we have slowed down the system. 

This thesis focuses on time-warping in the context of transfer learning for RNNs, where an RNN is pretrained on data at one timescale, then time-warped to model a related system that operates at a faster or slower timescale. We refer to time-warping in this thesis to distinguish it from standard temporal rescaling for this reason. 

\section{Time-Lag Systems}

Consider a linear first-order differential equation with initial conditions,
\[
\frac{dz(t)}{dt} = \lambda(X-z(t)), \quad z(0) = z_0,\;\lambda >0
\]
where $\lambda$ is a fixed positive parameter and $X$ is constant in time. A closed-form analytical solution can be obtained by separation of variables,
\begin{align*}
    \frac{dz}{X-z} &= \lambda\;dt\\
    \int \frac{1}{X-z} dz &= \int \lambda\;dt\\
    -\ln|X-z| &= \lambda t+\mathcal C\\
    z &= X - \mathcal C e^{-\lambda t}
\end{align*}
where $\mathcal C$ denotes a constant of integration. Using the initial condition, we have $\mathcal C = X-z_0$, and the analytical solution is
\begin{equation}
\label{eq:tl_sol_continuous}
\begin{aligned}
    z(t) = X + (z_0 - X)e^{-\lambda t}.
\end{aligned}
\end{equation}
The solution exhibits exponential relaxation. The state $z(t)$ converges to the equilibrium value $X$, with the rate of convergence determined by the parameter $\lambda$.

Next, we discretize the solution in time by introducing a constant sampling interval $\Delta t$. We allow the input to vary in time under a zero-order hold assumption, meaning that $X(t)$ is held constant over each sampling interval. Specifically, over the interval 
$[t_k, t_k+\Delta t)$, we take $X(t) = X_k$. Over this interval, the differential equation becomes
\[
\frac{dz}{dt} = \lambda(X_k - z),\qquad z(t_k) = z_k.
\]
Since $X_k$ is constant on the interval, the continuous-time solution gives
\[
z(t) = X_k + (z_k - X_k)e^{-\lambda(t-t_k)}.
\]
Evaluating at $t_{k-1} + \Delta t = t_{k}$, we obtain
\[
z_{k}=X_k + (z_{k-1} - X_k)e^{-\lambda\Delta t}.
\]
Rewriting,
\[
z_{k}=(1 - e^{-\lambda\Delta t})X_k+e^{-\lambda\Delta t} z_{k-1}.
\]
For $\Delta t = 1$, this simplifies to
\begin{equation}
\label{eq:tl_discrete}
z_{k}=(1 - e^{-\lambda})X_k+e^{-\lambda} z_{k-1}.
\end{equation}
Note that for $\lambda>0$, $e^{-\lambda}\in (0,1)$. Define $a:=e^{-\lambda}$. In this thesis, we define a time-lag system as a class of discrete-time linear models with the recurrence relation 
\begin{equation}
\label{eq:timelag}
\begin{aligned}
    z_{t} = az_{t-1} + (1-a)X_t, \quad 0<a<1,
\end{aligned}
\end{equation}
where $X_t$ is constant over $[t, t+1]$ under the zero-order hold assumption. 

While the form of \eqref{eq:timelag} is simple, many physical systems are described this way, and it appears in many applied contexts. In this thesis, we will demonstrate a simulation study based on Newton's Law of Cooling, which in its simple form corresponds to \eqref{eq:timelag}. The principles of Newton's Law of Cooling are used in applied areas of radiative heat transfer~\citep{Restuccia-2019-EWW}. Systems of the form~\eqref{eq:timelag} are also studied in the classical time series literature and control theory under the name of transfer function models~\citep[Sect.~1.2.2]{Box-2016-TSA}, and they are more commonly called order-1 ARX models. Simple RNNs with linear activation have the same form, and LSTMs can be shown to approximate \eqref{eq:timelag} with any desired accuracy, see Section~\ref{sec:math} below. 

The main application in this thesis is on predicting fuel moisture content (FMC) for fuels of various characteristic sizes. To predict FMC, simple physics-based models assume that the dynamics are invariant to the size of the stick. Only the characteristic time scale changes. That is, under the assumption that the same dynamic system describes a small twig and a large log, the dynamics are sped-up for the twig relative to the log. Section~\ref{sec:fmc} gives a background on FMC and some of the current modeling approaches. More complicated physics-based models use the stick radius as an additional parameter.

Other physical systems are described by time-lag models, either as approximations or through boundary conditions. For example, resistor–capacitor (RC) circuits describe the evolution of voltage over time and, under standard assumptions, correspond to \eqref{eq:timelag} when expressed in discrete time. Physical problems that correspond to time-lag systems are directly connected to the time-warping method developed in this thesis, but the time-warping method could be applied to more general modeling problems. 

A time-warping by $\gamma$, as defined in the previous section, modifies the discretized time-lag system by raising the parameter $a$ to the power $\gamma$,

\begin{equation}
\label{eq:timewarp}
\begin{aligned}
    \widetilde z_{t} = a^\gamma \widetilde z_{t-1} + (1-a^\gamma)X_t, \quad 0<a<1, \; \gamma>0.
\end{aligned}
\end{equation}

\section{Sequence Learning with Recurrent Neural Networks}

Sequence learning is an area of ML that deals with ordered sequences of observations. This includes time series modeling, where observations are ordered in time, and other areas that are not time-ordered. Examples of sequence learning problems that are not time-ordered include natural language processing (NLP) and DNA sequencing. There are some important theoretical differences between time series and more general sequence modeling tasks. In time series modeling, there is a notion of causality in the dynamics. Observations from the past can influence observations in the future, but not vice-versa. In NLP and DNA sequencing, the sequences do not have an intrinsic direction of causality. In NLP, for instance, word order is often a subjective stylistic choice, and the conjugation of a word can be influenced by words that occur after it in a sequence of text. Many of the same statistical techniques are used throughout the different sub-fields of sequence modeling, but in specific applications some techniques are more appropriate. This thesis will consider time series modeling, but many of the principles will extend to general sequence modeling. 

Recurrent neural networks (RNNs) are a class of neural network architectures developed for sequence modeling tasks. They build upon the architecture of classical neural networks. The classic neural network unit is the perceptron, but in modern practice it is referred to as the ``dense'' or fully-connected unit. Multiple dense units are stacked into a dense layer. It has connections to the previous layer, possibly the input layer, that result in a linear combination that is passed through an activation function. The activation function is typically nonlinear. In modern neural network practice, a sensible default activation function for dense layers is the Rectified Linear Unit (ReLU) function, which is defined
\[
\text{relu}(x) = \max(0, x).
\]
The output of a dense layer feeds into the next layer of the neural network, or it is the output of the network if it is the final layer. A neural network that only consists of dense layers is referred to as a feed-forward network, since information flows through the input layer to the output in only one direction. This type of architecture is commonly called a multilayer perceptron (MLP).

In contrast to standard feed-forward neural networks, RNNs maintain a state over time that is used to modify the output at the next time step. For many RNN architectures, this is a single value for the recurrent unit and is called the hidden state. We will refer to this as the ``recurrent state" to unify the terminology for the LSTM cell, which uses a two-dimensional recurrent state. A single cell, or unit, is the basic recurrent computational unit. Collections of recurrent cells are arranged in a recurrent layer, and then multiple layers can be arranged sequentially. Typically, the recurrent layers feed into one or more dense layers to structure the output of the network. The number of units in the final output layer and the activation function of the output layer is chosen by the nature of the problem. The number of units in the output layer typically corresponds to the dimensionality of the output. In this thesis, the target output will be a continuous scalar value, so the final output layer will be a single dense cell with linear activation. If a single probability was the target output of the network, such as when predicting a binary class probability, the output layer would be a single dense cell with an activation function that mapped to the interval $[0,1]$. If the target output of the neural network was a real-valued vector of length two, the final output layer would be two dense cells with linear activation. Other modeling problems will lead to different choices of output layer structure. We will refer to an RNN as any neural network that contains one or more recurrent layers. 

There are different types of RNN architectures. The simple RNN is the most basic version, but it suffers from stability issues during training which will be discussed later in this thesis. The Long Short-Term Memory (LSTM) augments the simple RNN in complex ways and addresses many of the stability issues with training. This thesis focuses on a new method for transfer learning with LSTMs. Other RNN architectures exist, like the Gated Recurrent Unit (GRU), but they will not be considered in this thesis. 

\subsection{The Simple RNN}

The simple RNN augments the dense cell by adding a recurrent connection. A recurrent layer will have one or more simple RNN units. Each RNN unit has a scalar-valued recurrent state that is updated over time. The vector $h_t$ represents the recurrent states for the entire recurrent layer. To produce an output at time $t$, a recurrent layer receives a vector of external predictors at the same time $X_t$ and the vector of recurrent states at the previous time $h_{t-1}$. Denote the output of the layer as $y_t$. In the case of the simple RNN, the output is equal to the hidden state. But for the LSTM, it is useful to label these separately. The inputs are passed into a linear combination and then a univariate activation function $\phi: \mathbb R \to \mathbb R$ is applied \citep[Ch.~15, Eq.~15--1]{Geron-2019-HOM},
\begin{equation}
\label{eq:simplernn_eqs}
\begin{aligned}
    h_t &= \phi(W_x\,X_t + W_h\,h_{t-1} + b) \\
    y_t &= h_t,
\end{aligned}
\end{equation}
where $W_x$ contains the weights connecting the inputs to the recurrent layer, $W_h$ contains the recurrent weights connecting the previous states to the current states, and $b$ is a bias vector. The vectors are treated as column vectors so the products $W_x X_t$ and $W_h h_{t-1}$ are standard matrix multiplications. The bias $b$ is a constant contribution to the argument of the activation function. The left side of Figure~\ref{fig:rnn_cell} shows the computational flow of the simple RNN. 

There are many possible choices for the activation function in the standard software libraries \citep{Keras-2025-Activation}. The most popular choice of activation function is the hyperbolic tangent ("tanh"), which maps from the real numbers to the range $(-1,1)$,
\begin{equation}
\label{eq:tanh}
\begin{aligned}
    \text{tanh}(x) = \frac{e^x - e^{-x}}{(e^x + e^{-x})}
\end{aligned}
\end{equation}

The weights and biases of the recurrent unit are kept constant during prediction and applied iteratively at each time step. For a finite sequence of length $N$ of external inputs $(X_t)_{t=1}^N$, and an initial recurrent state of $h_0$, the RNN generates a sequence of outputs for each time $(y_t)_{t=1}^N$. The right side of Figure~\ref{fig:rnn_cell} shows the RNN ``unrolled" in time.

\begin{figure}[htbp]
\centering
\begin{minipage}[c]{0.25\textwidth}
  \centering
  \includegraphics[width=\textwidth]{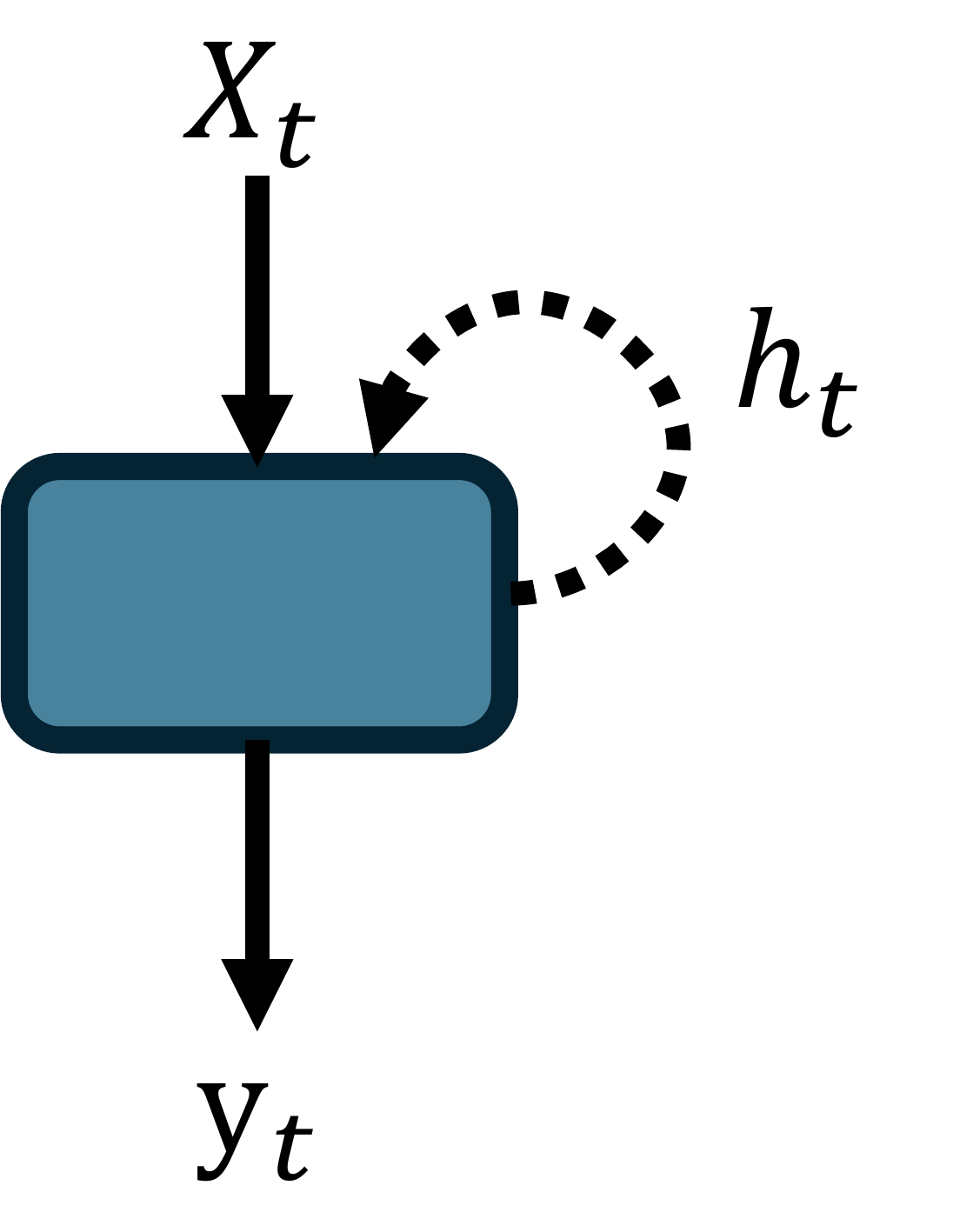}
\end{minipage}\hfill
\begin{minipage}[c]{0.65\textwidth}
  \centering
  \includegraphics[width=\textwidth]{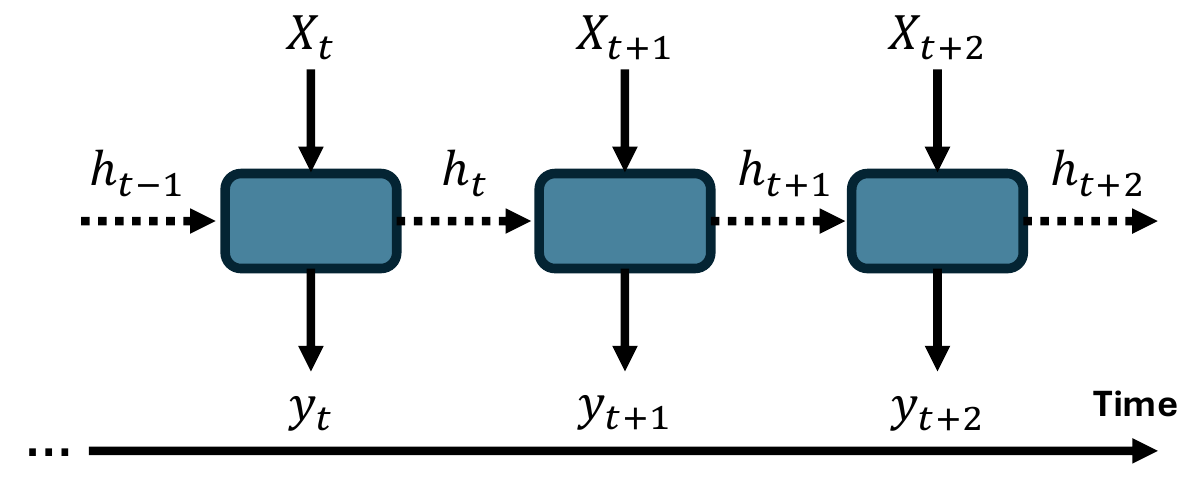}
\end{minipage}
\caption[
RNN cell schematic and information flow
]{
Two complementary views of an RNN.
(a) Left: a single RNN cell, following the formulation in \citep[p.~539]{Geron-2019-HOM}.
(b) Right: the RNN unrolled in time.
}
\label{fig:rnn_cell}
\end{figure}

\subsection{Long Short-Term Memory}

The Long Short-Term Memory (LSTM) augments the simple RNN by maintaining the hidden state and an additional long-term cell state. The combination of these two states will be referred to as the recurrent state. The LSTM was originally developed by \citet{Hochreiter-1997-LST} to address shortcomings of the simple RNN. First, the LSTM is capable of learning longer-term relationships than the simple RNN. Second, the LSTM addresses instability in the simple RNN when calculating gradients, an issue known as the exploding and vanishing gradient problem, which will be discussed in greater detail in Section~\ref{sec:opt}. The standard LSTM equations make use of the sigmoid function, commonly denoted $\sigma$,
\begin{equation}\label{eq:sigmoid}
\begin{aligned}
    \sigma : \mathbb{R} \to (0,1), \quad
    \sigma(x) = \frac{1}{1+e^{-x}}
\end{aligned}
\end{equation}

The LSTM unit maintains two separate states over time that are interpretable as memory. The hidden state, denoted $h_t$, is interpretable as the short-term memory and is also the output of the LSTM. It is recomputed at each time step from the current input and the cell state, so it can change quickly and does not preserve past information unless it is maintained in the cell state. The cell state, denoted $c_t$, is interpretable as the long-term memory. It is incrementally modified over time. The forget gate controls how much of the previous state is retained, while the input gate controls how new information is added. The cell state persists over time and evolves gradually.

The pair of these two states $[c_t, h_t]$ will be referred to as the recurrent state from now on. The recurrent state is modified through a series of computational gates. At time $t$, each of the computational gates receives as an input the hidden state of the recurrent layer at the previous time $h_{t-1}$ and a vector of the external input at the same time $X_t$. The computations of an LSTM, as formulated in \citet[Ch.~15, Eq.~15--4]{Geron-2019-HOM} are shown in Figure~\ref{fig:lstm_cell}, 

\vspace{1em}
\begin{samepage}
\begin{figure}[ht]
\centering
\includegraphics[width=12 cm]{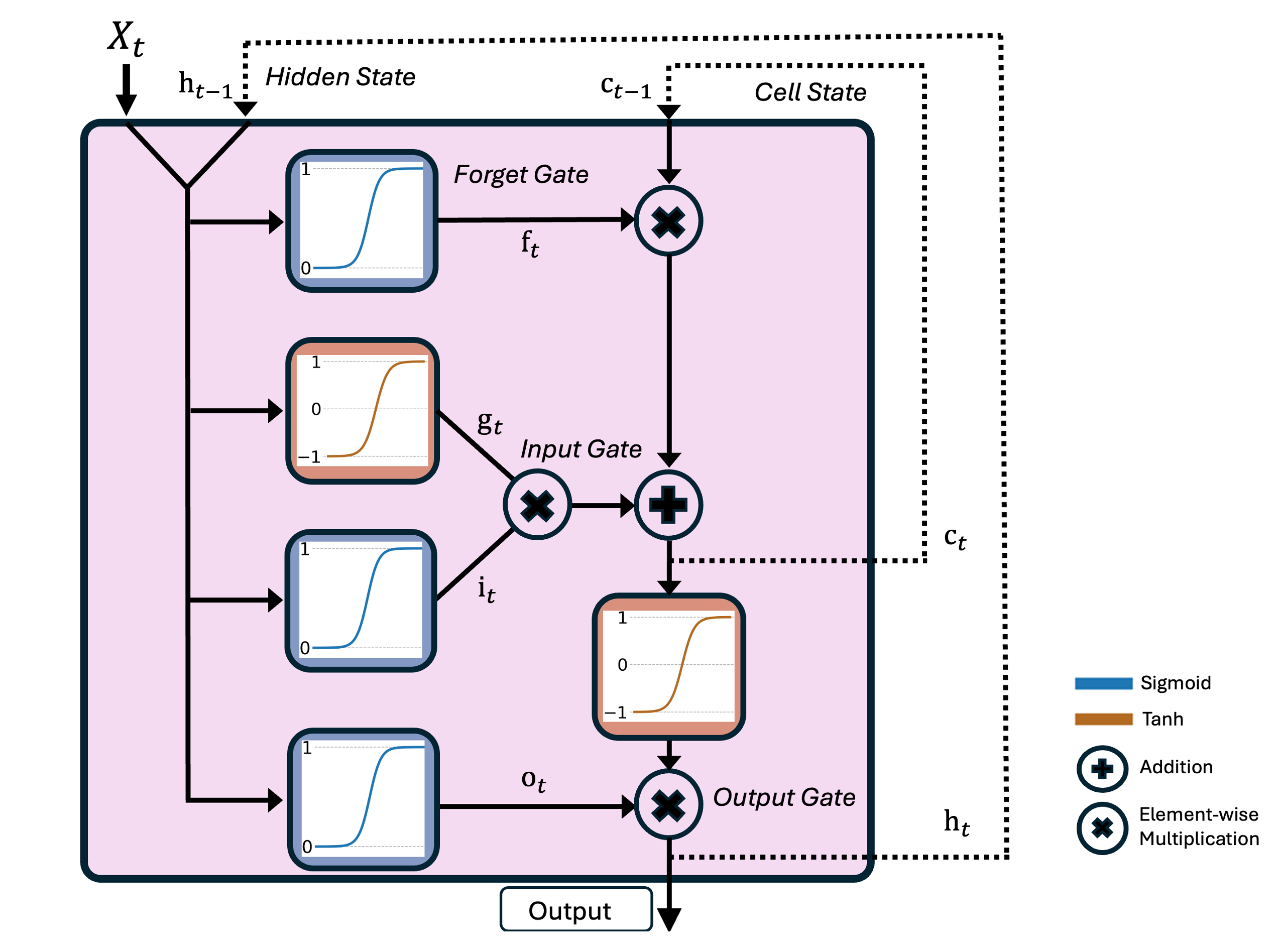}
\caption[A single LSTM unit.]{A single LSTM unit. This is an original graphic, developed following the formulation in \citep[p.~569]{Geron-2019-HOM}. The computations are described in Equations~\eqref{eq:lstm_eqs}.}
\label{fig:lstm_cell}
\end{figure} 

\begin{equation}
\label{eq:lstm_eqs}
\begin{aligned}
f_t &= \sigma(W_{xf}\,X_t + W_{hf}\,h_{t-1} + b_f) \\
i_t &= \sigma(W_{xi}\,X_t + W_{hi}\,h_{t-1} + b_i) \\
g_t &= \tanh(W_{xg}\,X_t + W_{hg}\,h_{t-1} + b_g) \\
o_t &= \sigma(W_{xo}\,X_t + W_{ho}\,h_{t-1} + b_o) \\
c_t &= f_t \otimes c_{t-1} + i_t \otimes g_t \\
h_t &= o_t\otimes \tanh(c_t)
\end{aligned}
\end{equation}
\end{samepage}

Again, the products are standard matrix multiplications. The symbol ``$\otimes$" is the Hadamard product, or element-wise multiplication. The trainable weights are described as follows. $W_{xf}, W_{xi}, W_{xg}$, and $W_{xo}$ are weight matrices that connect the set of external inputs to each gate. $W_{hf}, W_{hi}, W_{hg}$, and $W_{ho}$ are weight matrices that connect each gate to the hidden state of the LSTM layer at the previous time step $h_{t-1}$. Lastly, $b_{f}, b_{i}, b_{g}$, and $b_{o}$ are the vectors of bias terms for each gate. 

The recurrent state of an LSTM cell consists of two components, the hidden state and the cell state, formulated as the matrix $[h_t, c_t]$. The computational flow over time in Figure~\ref{fig:rnn_cell} still applies with an LSTM as the recurrent layer. The full recurrent state is maintained over time, but only the hidden state $h_t$ is passed to downstream layers in the neural network. In other words, the cell state $c_t$ is not produced as an output of the recurrent layer.

The gates can be interpreted heuristically as follows. The forget gate $f_t$ controls how much of the previous cell state to retain. A larger forget gate activation increases the fraction of the previous cell state that is retained, so the network retains more past information and therefore ``forgets'' less. Conversely, a smaller forget gate activation reduces the retained fraction of the previous cell state, causing the network to retain less information from the past. The input gate $i_t$ controls how much information from the current input to add to the cell state. The candidate memory $g_t$ represents a bounded nonlinear transformation of the current input and previous hidden state, and serves as the proposed update to the cell state. The output gate $o_t$ modulates the extent to which the cell state contributes to the hidden state. The recurrence of $c_t$ is linear, consisting of multiplication by $f_t$, which lies in (0, 1), and addition of the candidate memory times the input gate, which lies in (-1, 1). This leads to more stable dynamics and slower evolution over time, rather than repeatedly applying full linear transformations and nonlinear activation functions.

For an LSTM layer with $N$ units, each unit retains its own cell state which is not directly shared with the other units. However, each gate receives the full hidden state $h_{t-1}$ at time $t$, which contains information from the cell state of each unit. So for an RNN with $N>1$, the cell states $c_t$ are coupled to each other.

Figure~\ref{fig:lstm_cell} displays the LSTM architecture for a single unit using the standard activation functions, with three sigmoid activations and two tanh. However, software implementations of the LSTM allow for different activations functions to be used, including Keras~\citep{Keras-2025-LSTM}.

\subsection{Optimization with Backpropagation through Time}\label{sec:opt}

The parameters of neural networks are learned through a procedure that attempts to minimize a loss function. The loss with respect to model parameters for complicated models like neural networks cannot be assumed to be convex, and finding analytical solutions to minimize the loss is infeasible in all but the most contrived examples. The standard approach is a method known as gradient descent, which iteratively modifies model parameters by moving in the opposite direction of the gradient of the loss with respect to model parameters, e.g.~\citep[Ch.~4]{Geron-2019-HOM}. Calculating the gradient of the loss across the entire training dataset tends to lead to the algorithm getting stuck in local minima. Batch gradient descent is an augmentation that groups the training samples into batches and calculates gradients with respect to a subset of the overall training data. By randomizing the order of samples used in the batching, stochasticity is introduced to the optimization algorithm. In practice, adding randomness to the parameter search helps to avoid local minima and leads to better solutions.

Backpropagation is an efficient algorithm for calculating the gradient of a loss with respect to model parameters through successive applications of the chain rule \citep[Sect.~6.5]{Goodfellow-2016-DL}. The algorithm involves a forward pass, where inputs flow through the neural network architecture to generate an output. A loss is calculated by comparing the output of the neural network to a known target value for a specified loss function. The backward pass of the algorithm calculates the gradient of the loss with respect to each model parameter by applying the chain rule of calculus. Finally, the trainable parameters of the neural network are updated by stepping in the opposite direction of the gradient of the loss with respect to the model parameters.

For RNNs, the model parameters are shared across time while the recurrent state evolves recursively to generate outputs. The gradient of the loss with respect to the model parameters must account for this temporal dependence, a modification known as Backpropagation Through Time (BPTT) \citep[Ch.~9.7]{Zhang-2023-DDL}. Optimization via gradient descent requires computing gradients with respect to these parameters, including the weights and biases associated with recurrent connections. As an intermediate step, this involves computing gradients of the loss with respect to the recurrent state. The recurrent state $h_t$ depends on all previous states $h_{t-1}, h_{t-2}, \dots, h_0$, where $h_0$ is the initial state. Repeated application of the chain rule leads to a product of terms: for a single-unit recurrent layer, this product consists of partial derivatives with respect to $h_{t-1}, h_{t-2}, \dots, h_0$, while for multi-unit layers it consists of Jacobians of the state transitions. Computing this product is computationally linear in sequence length but highly unstable. In practice, gradients may grow or decay exponentially, leading to the exploding and vanishing gradient problem.

The most common way to address the exploding and vanishing gradient problem is to approximate the full derivatives by truncating the recursive application of the chain rule. Shorter sequences avoid the numerical issues of exploding and vanishing gradients. This introduces a hyperparameter known as sequence length that needs to be tuned. This method is known as Truncated Backpropagation through time (TBTT)~\citep[Sect.~9.7]{Zhang-2023-DDL}. 

Combining TBTT and batch gradient descent leads to a three-dimensional structure for the training data for RNNs. Training samples are arranged in a matrix of the form \texttt{(batch size, sequence length, features)}, where features are the number of unique inputs to the model. A single sample is a time-ordered matrix of the shape \texttt{(sequence length, features)}. The sample passes through the RNN layers to generate an output, which is compared to a target vector of observed data to calculate a loss. It is a modeling choice whether all the outputs of length equal to the sequence length is used to calculate a loss, or just the last element in the sequence. In the Keras and TensorFlow set of software tools, the function argument \texttt{return\_sequences} controls this behavior~\citep{Keras-2025-LSTM}. When \texttt{return\_sequences} is true, the loss calculation for a single sample is the loss averaged over the entire sequence length window. The algorithm proceeds by calculating losses for each sample in a batch, then aggregating the loss calculations, typically through averaging~\citep{Keras-2025-Loss}, and then the model parameters are updated in the opposite direction of the gradient in a step proportional to the learning rate hyperparameter. A single epoch of training is complete when all samples in all batches have been processed. 

The number of epochs used for training is a hyperparameter that requires tuning. The method known as ``early-stopping'' dynamically halts training when there are signs that the model predictive accuracy stagnates. Early-stopping utilizes a validation set during training. After each epoch of training, the loss on the validation set is calculated. Training is halted once the accuracy on the validation set fails to improve after a predetermined number of epochs, which is known as the ``patience'' hyperparameter. Then the final set of model weights is extracted after the epoch that had the lowest loss on the validation set. In practice, you can set the number of epochs to some very large number and utilize an early-stopping patience of 5 to 10, which alleviates the need to precisely tune the number of epochs. It is still possible that the loss would have continued to decrease after stagnating for 5 to 10 epochs, so early-stopping does not guarantee that the resulting model is as accurate as possible without overfitting. 

Training RNNs is notoriously tricky. The solution can be highly sensitive to hyperparameter choices. Though many methods exist to mitigate the issues, the exploding and vanishing gradient problem remains and can lead to computational errors or slow convergence. The sequential way that RNNs generate predictions makes RNNs less amenable to parallelization as other neural network methods. For these reasons, modern NLP has largely moved away from RNNs in favor of attention mechanisms and transformer architectures. We will briefly discuss why RNNs are still useful for applied modeling of physical systems in the next section. 

\subsection{Machine Learning Alternatives to RNNs}

Convolutional neural networks (CNNs) have also been applied to time series analysis. While CNNs are widely used for image processing and other spatially-distributed data, networks can also be constructed using one-dimensional temporal convolutions. In these architectures, convolutions act as temporal feature extractors. This approach has become increasingly common in environmental science in recent years. For example, \citet{Zhu-2021-LFM} and the follow-up study by \citet{Miller-2023-PLF} use one-dimensional temporal convolutions to predict live fuel moisture content point-wise across CONUS. 

Attention mechanisms are an ML approach to sequence learning that differs fundamentally from recurrent neural networks (RNNs). In an attention layer, each element of the output sequence is represented as a weighted sum of linear combinations of all elements in the input sequence. The model learns the coefficients of the linear combinations, while the weights in their sum change depending on the element currently being processed. As the current element changes, the weighting of the other elements changes accordingly, which is commonly interpreted as representing the context within the sequence. In a seminal 2017 paper, researchers from Google introduced the transformer architecture, which is built around attention mechanisms for sequence learning~\citep{Vaswani-2017-AAY}. Transformers proved highly effective for NLP tasks, and modern large language models are primarily based on the transformer architecture. Transformers are highly parallelizable in a way that RNNs are not. RNNs maintain a memory over time and generate outputs sequentially, which limits opportunities for parallel computation.

One of the main motivations for the transformer architecture was to model very long-term dependencies that RNNs struggle to capture. These long-range dependencies are particularly relevant in natural language processing, where context and logical coherence may need to be maintained across thousands of elements in a sequence. RNNs, by contrast, tend to gradually forget the past. This tendency proved ill-suited to NLP~\citep{Vaswani-2017-AAY}, but it can be appropriate and even desirable for many physical systems. Placing low importance on the beginning of a conversation or document can lead to nonsensical language modeling, but for physical problems such as weather forecasting it often makes sense to place greater importance on more recent conditions than on the distant past. 

\section{Autocorrelation and Autoregression}

Autocorrelation and autoregression are concepts from classical time series analysis. While they are commonly used for modeling temporal processes, they can also be used to characterize the empirical temporal dependencies of an arbitrary time series. In Chapter~\ref{chapter:04}, these concepts are used to analyze the temporal dependencies present in predictions generated by RNN models.

Autocorrelation is the correlation between a discrete time series and itself shifted in time. For a sequence $(z_t)$, where the index set may be finite or infinite, let $k$ be an integer time interval. The covariance between $(z_t)$ and $(z_{t+k})$ is called the autocovariance at ``lag $k$''~\citep[Sect.~2.1]{Box-2016-TSA}. Note that this use of the term ``lag'' is different from the concept of time-lag systems discussed above, but there is a mathematical relationship between these concepts that we will discuss shortly. The autocovariance at lag $k$ is defined as, 
\begin{equation}
\label{eq:acv}
    \text{cov}(z_t, z_{t+k}) = \mathbb E[(z_t - \mu)(z_{t+k}-\mu)],
\end{equation}
where $\mathbb E$ is the expected value, and $\mu$ is the mean of the sequence. If $(z_t)$ is an infinite sequence, it needs a well-defined mean and variance for the autocovariance to be defined. If $(z_t)_{t=1}^N$ is a finite sequence, then $\mu$ can be estimated as the empirical sample mean, and the autocovariance can be estimated empirically with,
\begin{equation}
\label{eq:acv_empirical}
    \hat{\text{cov}}(z_t, z_{t+k}) = \frac{1}{N}\sum_{t=1}^{N-k}(z_t-\hat\mu)(z_{t+k}-\hat\mu), \qquad \hat\mu = \frac{1}{N}\sum_{t=1}^{N}z_t.
\end{equation}
The normalization uses $N$ rather than $N-k$ so that the same scaling is applied for all lags, which simplifies the definition since the normalization constants cancel in the ratio~\citep[Eq.~2.1.12]{Box-2016-TSA}.

The autocorrelation function (ACF) scales the autocovariance between -1 and 1 by normalizing by the variance of the process. The autocorrelation at lag $k$ is denoted $\rho_k$ and is defined as,
\begin{equation}
\label{eq:acf}
    \rho_k= \frac{\mathbb E[(z_t - \mu)(z_{t+k}-\mu)]}{\mathbb E[(z_t - \mu)^2]} = \frac{\text{cov}(z_t, z_{t+k})}{\text{var}(z_t)}
\end{equation}
A time series with no temporal dependencies would be expected to have zero autocorrelation for all lags. The white noise process, denoted $\epsilon_t$, is a time series of i.i.d. random terms with no temporal structure, and the ACF is zero for all lags. 

For a finite sequence $(z_t)_{t=1}^N$, the autocorrelation can be estimated empirically as,
\begin{equation}
\label{eq:acf_empirical}
    \hat\rho_k= \frac{\sum_{t=1}^{N-k}(z_t-\hat\mu)(z_{t+k}-\hat\mu)}{\sum_{t=1}^{N}(z_t-\hat\mu)^2} = \frac{\hat{\text{cov}}(z_t, z_{t+k})}{\hat{\text{var}}(z_t)}
\end{equation}

Autoregressive processes are discrete-time stochastic processes defined in terms of previous values of the series~\citep[Sect.~3.2]{Box-2016-TSA}. An autoregressive process of order $p$, denoted AR($p$), and is defined as,
\begin{equation}
\label{eq:ark}
    z_t = \sum_{k=1}^p \beta_k\,z_{t-k} + \epsilon_t,
\end{equation}
where $\beta_k$ are fixed parameters and $\epsilon_t$ is a white-noise process. For the ACF of an AR($p$) process to be well-defined, the parameters must satisfy certain constraints that ensure the process is stable~\citep[Sect.~3.2.3]{Box-2016-TSA}. 

The simplest autoregressive process is the AR(1), which is defined by a recursive relationship of the form
\begin{equation}
\label{eq:ar1}
    z_t = \beta\,z_{t-1} + \epsilon_t
\end{equation}
The theoretical ACF at lag $k$ can be derived analytically~\citep[Eq.~3.2.12]{Box-2016-TSA} and is given by
\[
\rho_k = \beta^k, \quad k\geq 0.
\]
Under the condition $|\beta|<1$, the ACF of the AR(1) process decays exponentially to zero as the lag increases.  

It is possible to fit the parameters of an AR($p$) process to data, but properties of the AR($p$) process can also be used to analyze the temporal structure of an observed time series. Autoregressive processes of different orders $p$ have different theoretical forms for their ACF. By comparing the empirical ACF of an observed time series to the theoretical ACF of an AR($p$) process, conclusions can be drawn about the underlying temporal dependencies. 

The plot of lag $k$ versus the ACF at lag $k$ is an important diagnostic tool for analyzing temporal dependencies in a time series. The partial autocorrelation function (PACF) measures the correlation between $z_t$ and $z_{t-k}$ after accounting for the influence of the intermediate lags $1,\dots,k-1$~\citep[Sect.~3.2.5]{Box-2016-TSA}. For positive $\beta$, the ACF of the AR(1) process begins at $\rho_0 = 1$ and then decreases monotonically toward zero as the lag increases. The PACF takes the value $\beta$ at lag 1 and is zero for all lags $k>1$. 

For a finite sample from a white-noise process with no temporal dependence, the true ACF and PACF are zero at all nonzero lags, but empirical estimates may deviate from zero due to sampling variability. In particular, the ACF at a single lag may appear nonzero purely by chance. Approximate 95$\%$ confidence bounds for the ACF under the white-noise assumption can be calculated and are the same for all lags. These bounds are commonly used as a statistical significance threshold to determine whether an observed ACF value at a given lag is larger than would be expected under the null hypothesis of zero temporal dependence. In diagnostic ACF plots, these approximate 95$\%$ bounds are typically shown as two horizontal lines. The same significance threshold is often applied to PACF plots.

We previously defined the time-lag system as $z_t = a z_{t-1} + (1-a)X_t$, with $0<a<1$. This system can be interpreted as an autoregressive process with an external input and is commonly referred to as an ARX(1) process. Unlike the pure AR(1) process, the ACF and PACF of an ARX(1) process depend on the statistical properties of the external input $X_t$. Without placing additional assumptions on $X_t$, there is therefore no general theoretical form for these functions. 

Nevertheless, the empirical ACF and PACF can still provide qualitative information about the persistence of the system. If the observed ACF and PACF resemble those of a pure AR(1) process, this suggests that the recursive dependence on the past state $z_{t-1}$ plays a substantial role in the temporal structure of the series. In this sense, similarity to the theoretical AR(1) ACF can be interpreted as evidence of persistence or memory in the system. This interpretation should be viewed as a heuristic guideline. For an empirical time series it is generally not possible to disentangle the contributions of the external input $X_t$ and the past state $z_{t-1}$ to the observed ACF and PACF. However, comparing the empirical autocorrelation structure to that of the AR(1) process provides a useful reference point for characterizing the degree of temporal persistence in the system.

\section{Transfer Learning}\label{sec:tl}

Transfer learning is a general term in ML for reusing knowledge learned in one context in another. The main practical motivation for using transfer learning is when a certain modeling problem has insufficient observed data to adequately train an ML model, but there is more available data for a related modeling problem. Another motivation is that by utilizing transfer learning, an ML model can learn deeper aspects of a system that are assumed to be general across different tasks. As a motivating example, consider the task of classifying images on whether they contain dogs and the related task of classifying images on whether they contain foxes. There will be many more available images of dogs to train ML models to classify the images. For a particularly rare species of fox, there may be too few available images to reliably train an ML classifier. Pretraining a classifier on images of dogs could help train a classifier for foxes. Dogs and foxes are relatively closely related and share many physical features. A classifier trained on images of dogs could learn a general set of features that are useful to identify foxes. 

\subsection{Transfer Learning Methods}

\citet{Pan-2010-STL} formulated transfer learning, and their terminology is widely used today. Let the domain of a problem be $\mathcal D = \{\mathcal X, P\}$, where $\mathcal X$ is a space of features and $P$ is the marginal distribution over $\mathcal X$. Let a task be $\mathcal T = \{\mathcal Y, f\}$, where $\mathcal Y$ is a space of response values and $f$ is a predictive function. Transfer learning defines a source domain and source task, which we will respectively denote $\mathcal D^{(s)}$ and $\mathcal T^{(s)}$, and a target domain and target task, which we will respectively denote $\mathcal D^{(r)}$ and $\mathcal T^{(r)}$. 

Transfer learning categorizes problems according to which aspects of the domain and task are held consistent between the source and the target~\citep{Pan-2010-STL}. In this thesis, we will briefly analyze a ``transductive'' transfer learning case for the empirical example with FMC. We will primarily consider transfer learning problems in which the domain is identical for both the source and the target, that is, $\mathcal D^{(s)} = \mathcal D^{(r)}$. This is common in applied environmental science, where features like the weather and geography are used to model many different local conditions. Further, we will consider cases in which the space of response variables is identical for the source and the target, that is, $\mathcal Y^{(s)} = \mathcal Y^{(r)}$. The difference lies in the predictive functions, with $f^{(s)} \neq f^{(r)}$. In the transfer learning literature, this is referred to as an inductive transfer learning problem. In our setting, $f$ is parameterized by an RNN of the form
\[
y_{t} = f(y_{t-1}, X_t).
\]
The source $f^{(s)}$ and the target $f^{(r)}$ are RNNs defined on the same input and response spaces. The predictive functions represent different dynamical systems. 

Additionally, this thesis focuses on parameter-based transfer learning~\citep{Pan-2010-STL}. The predictive functions are assumed to have the same structural form, but differ in their parameter values. In the neural network context, this implies that the predictive functions will have the same number and type of hidden layers with the same number of units. Hyperparameters will also be shared, such as the activation functions used and the learning rate and other optimization-related hyperparameters. These constraints are not necessary in general, since different tasks could be best modeled by different architectures. We fix the network architecture to isolate the effects of parameter transfer and avoid confounding architectural variation.

We use the phrase ``\textbf{pretrained}'' to describe a model that is fit using the source domain and task. Some researchers refer to this as ``initialization'', but this phrase already describes how the parameters of an ML model are initially set before gradient descent optimization proceeds. ``\textbf{Fine-tuning}'' refers to updating the parameters of a pretrained model by using data realizations from the target response space. Fine-tuning is often performed with certain constraints related to which parameters are allowed to be updated and which are held constant. This will be discussed in greater detail in the next section. The phrase ``\textbf{zero-shot}'' refers to a transfer learning case with no fine-tuning or parameter fitting , so predictions for the target task are generated without seeing any data from the target response space.

\subsection{Transfer Learning Neural Networks with Fine-Tuning}

The main parameter-based transfer learning technique in neural networks is through fine-tuning. Fine-tuning in neural networks consists of optimization of model parameters through gradient descent using training samples from the target learning problem. A model is pretrained by optimizing parameters using data from the source problem, and then the model is fine-tuned by retraining with data from the target. Full fine-tuning refers to the situation where all parameters are allowed to be updated with no restrictions. 

\citet{Fawaz-2018-TLT} evaluated transfer learning for classification of time series with CNNs. A CNN was pretrained on a source learning task, then fine-tuned on a target learning task. The output layer of the CNN used softmax activation to map to different class labels, and needed to be re-initialized for the target learning task. 

Freezing layers in neural networks is a technique that expands on fine-tuning to try to identify parameters that are generic across learning tasks, e.g.~\citep{Yosinski-HTF-2014}. Freezing a layer fixes the parameter values after pretraining, and parameters are not updated during fine-tuning. Freezing layers works on the assumption that the layers have knowledge that is general between the source and target tasks. Transfer learning has been extensively studied in image processing with CNNs. A CNN is trained to classify images for a certain set of classes for the source learning task. The target learning task introduces new image classes that the CNN has not been trained to classify. 

It has been shown to be effective to freeze the input or first few layers of a CNN~\citep{Yosinski-HTF-2014}. The weights of layers downstream from the frozen layers are then fine-tuned. The interpretation is that the earlier convolutional layers extract a set of general features, such as general edge detection similar to Gabor filters or features that identify patches of colors. These general features transfer well between tasks. The layers downstream and closer to the output layer are interpreted to be learning task-specific features that do not transfer well between tasks. However, freezing weights and fine-tuning can raise an issue that the authors refer to as ``fragile co-adaptation'' of learned patterns. The weights of neural networks have complex dependencies, and by fixing one set of weights, there is no guarantee that gradient descent can continue to effectively navigate the parameter space to find optimal solutions. 

Freezing weights in neural networks with LSTM layers has been used in applied environmental science, including in the streamflow modeling literature, where predictive models are constructed to forecast stream water volume. \citet{Ma-2021-THD} pretrained an RNN with an LSTM recurrent layer on data from streams in CONUS, then fine-tune and generate predictions in parts of China that have far sparser data. They tested freezing different layers of the RNN in three configurations. One configuration used no freezing and allowed all weights to be fine-tuned, another froze the weights from the input layer, and the last froze the weights in both the input and output layers. They did not find that freezing layers systematically improved performance relative to a full fine-tuning of all the parameters. The streamflow modeling literature uses many of the same ML approaches as we employ with FMC, as will be discussed in Section~\ref{sec:fmc}. As with FMC, time series of observations from a sparse spatial grid are combined with weather and static geographic features and used to train RNNs. 

\section{Time-Warping ML Sequence Models}

In this section, we provide a literature review of results and techniques within ML related to RNNs and time-warping. Some of these research areas provide a theoretical foundation for this thesis, and some of these research areas are related techniques that help show the novelty of the work in this thesis. To our knowledge, this thesis provides the first work that involves modifying the dynamics of a pretrained network.

\subsection{Learning Characteristic Time Scales with RNNs}

Recurrent neural networks provide a mechanism for learning characteristic time scales through their gating dynamics. The Chrono-Initialization method for LSTMs focuses on initializing bias parameters of an RNN based on the assumed range of the memory time constants~\citep{Tallec-2018-CRN}. In an LSTM, the authors identify the forget and input gates as controlling the effective time scale of the dependencies. They initialize the forget gate bias so that, under certain assumptions, the resulting characteristic time scale of the dynamics matches the dynamic system under consideration. They initialize the forget and input gate bias terms as opposites of each other, or $b_f = -b_i$. \citet{Ohno-2021-RNN} further analyze the time scale of the hidden state under more general conditions and show that the forget gate continues to play a dominant role in controlling the time scale of the learned dynamics. \citet{Can-2020-GCS} analyze the Jacobian of an LSTM with arbitrary weights and show that, in the appropriate limit, the forget gate plays a central role in determining the effective time scale of the dynamics.

This work motivates the idea that an RNN learns a set of parameters tied to a characteristic time scale, and that manipulating the bias terms of the forget and input gates provides a way to control the temporal dependencies of the learned dynamics. While their work focuses on initialization and training behavior, it does not examine post hoc time-warping of a pretrained RNN. This thesis investigates how a trained recurrent model may be adjusted to operate under changes in temporal resolution.

\citet{Quax-2020-ATS} add learnable parameters to an RNN that allow each unit to adapt the rate at which the hidden state evolves. The authors interpret this as learning the intrinsic time scales of continuous-time dynamic systems. The adaptive time scales in the RNN are inspired by research into the human brain that shows a hierarchy of time scales. Empirically, they show that the architecture learns time constants that reflect the temporal structure of different learning tasks. While their work demonstrates that RNNs can internally learn task-relevant time scales, it does not address post-hoc manipulation or warping of those learned characteristic time scales, which is the focus of this thesis. 

\subsection{Neural Ordinary Differential Equations}

Neural ODEs are a class of methods introduced by \citet{Chen-2018-NOD}. Instead of specifying a fixed form of an ODE, they use a neural network to learn the instantaneous rate of change of a system defined by
\[
\frac{dy}{dt} = f_\theta(t, y),
\]
where $f_\theta$ is typically an MLP with a set of parameters $\theta$. This method allows for using standard ODE integrators to evolve a system forward with dynamic step sizes. They build a theoretical connection to residual neural networks. As the number of hidden layers increases and the step size goes to zero, in the limit, a ResNet approaches an ODE.

Hybrid approaches combine neural ODEs with discrete-time RNNs. \citet{Rubanova-2019-LOI} use a neural ODE to evolve the hidden state of an RNN between the discrete time steps. 

Neural ODEs help model irregularly sampled data. The state can be evolved forward between sparsely sampled observations. In contrast, RNNs learn a  discrete update to a system. In practice, RNNs operate on a fixed temporal grid, and sparsely sampled observations necessitate data imputation or other data processing methods. 

Taking a pretrained neural ODE and rescaling its learned vector field for the purpose of transfer learning does not appear to be explicitly studied in the existing literature. In particular, given a trained model, the dynamics could be modified to
\[
\frac{dy}{dt} = \gamma f_\theta(t, y),
\]
where $\gamma$ is the time-warping constant. To our knowledge, time-warping a pretrained network as a transfer learning approach has not been systematically explored.

While time-warping with Neural ODEs is conceptually straightforward, there are several reasons why we focus on time-warping with standard RNNs in this thesis. LSTMs and related RNN architectures are widely adopted in both research and industry, with mature software support for major ML frameworks. In contrast, neural ODEs are more commonly used in specialized research settings and typically require dedicated solver infrastructure and additional numerical considerations during training. In addition, training Neural ODEs requires repeatedly solving a differential equation during both the forward and backward passes, which introduces additional computational overhead and sensitivity to solver tolerances. In practice, discrete RNN architectures can be scaled more efficiently using highly optimized ML frameworks.

\subsection{Dynamic Time-Warping}

Dynamic time-warping (DTW) is a method for comparing the similarity between discrete time series, originally developed for speech recognition~\citep{Sakoe-1978-DPA}. Though the terminology is similar, it is not time-warping as we have defined it in this thesis, as there is no global temporal rescaling. A loss function, typically the absolute error or the squared error, is calculated for all pairwise combinations of index positions to form a loss matrix. The minimum-cost path through this matrix is computed, which defines an alignment between the time series and yields a similarity score. Dynamic time-warping has been utilized in recent years in ML sequence learning tasks. \citet{Duan-2026-DTW} use DTW as a preprocessing step to select representative source datasets for training an LSTM, which is then fine-tuned to a target learning problem. \citet{Laperre-2020-DTW} construct an accuracy metric based on DTW as a postprocessing step to evaluate LSTM models used to forecast geomagnetic properties of the Earth. The DTW-based metric is used in cases where standard loss functions, such as RMSE, are inadequate.

\section{Fuel Moisture Content}\label{sec:fmc}

The applied research area that motivated this thesis is in fuel moisture content (FMC) modeling. FMC is a measure of the water content of burnable fuels, e.g.~\citep{NCEI-2024-DFM}. FMC is defined as the ratio of the weight of a fuel to the dry weight of the fuel, or the total mass of burnable woody material when all moisture has been removed. 
\begin{equation}
\label{eq:fmc}
    \text{FMC} (\%) = \frac{\text{Fuel Weight} - \text{Dry Weight}}{\text{Dry Weight}} \cdot 100\%
\end{equation}

Mathematical models of FMC are used to make predictions at unobserved locations in the future. Forecasts of FMC have been used within fire danger rating systems and wildfire spread models for decades, e.g.~\citep{Jolly-2024-MUN}.

FMC is strongly related to wildfire spread, and it is an important input to wildfire spread simulations, e.g.~\citep{Mandel-2019-IDH}. The United States Forest Service maintains operational wildfire simulation tools that model FMC over time and the affect of FMC on fire spread and fire intensity~\citep{Finney-1998-FAS, Andrews-2009-BFM}. FMC is also used to estimate wildfire risk and identify areas that are prone to rapid wildfire spread. Wildfire risk indices are scores calculated from FMC and other weather conditions that are published by various agencies. The National Fire Danger Rating System (NFDRS) is a set of wildfire models and management tools from the United States Forest Service~\citep{Jolly-2024-MUN}. The NFDRS uses forecasts of FMC to analyze wildfire danger across the country. FMC cannot go below zero, and values less than $1\%$ are considered not physically reasonable, e.g.~\citep{vanderKamp-2017-MFS}. The maximum theoretical FMC is the saturation level, or the moisture content of the woody material when it cannot hold any more water. 

The equilibrium moisture content is the hypothetical FMC level that a fuel will reach if left in the same environmental conditions for a long time. It is usually defined as a combination of the air temperature and the relative humidity (RH). Empirical tests have shown that fuels approach a different equilibrium when the FMC is increasing versus when it is decreasing, so researchers have defined the ``wetting'' equilibrium and ``drying'' equilibrium moisture levels to understand how FMC changes over time in different weather conditions as~\citep{Viney-1991-RFF} 
\begin{equation}
\begin{aligned}
E_d &= 0.924\,\mathrm{RH}^{0.679}
      + 0.000499\,e^{0.1\,\mathrm{RH}}
      + 0.18\,(21.1 - T)\left(1 - e^{-0.115\,\mathrm{RH}}\right) \\
E_w &= 0.618\,\mathrm{RH}^{0.753}
      + 0.000454\,e^{0.1\,\mathrm{RH}}
      + 0.18\,(21.1 - T)\left(1 - e^{-0.115\,\mathrm{RH}}\right),
\end{aligned}
\label{eq:eqs}
\end{equation}
where $E_d$ and $E_w$ denote the drying and wetting equilibrium moisture contents, $\mathrm{RH}$ is relative humidity (percent), and $T$ is air temperature in $^\circ\mathrm{C}$.

While fuels can be of any size, researchers have divided fuels into several idealized categories to make modeling tractable. These are referred to as fuel classes. The fuel classes are defined by a threshold lag-time, or the expected number of hours needed for a fuel to change its moisture content by $1-e^{-1}$ ($\approx63\%$) in response to an atmospheric change~\citep[p.~216]{Viney-1991-RFF}.

The key assumption of the time lag concept in FMC modeling is that the same physical process describes the change in moisture for sticks of different sizes. In this framework, the only difference between a small twig and a large log is the rate at which they change moisture. The same dynamic equations are used to describe the change in moisture, but for the small twig the dynamics are sped up, and for the larger log the dynamics are slowed down. A 1h fuel is said to have a short memory. The FMC of a 1h fuel responds quickly to changes in the environment, and the FMC of a 100h fuel responds much slower. Researchers informally describe fuels as having a ``memory'' of the past weather conditions. The memory for a fine twig is very short, and only the most recent weather conditions are relevant for forecasting the FMC. For the standard 10h fuel, researchers consider it to have up to one week of memory of the weather. That means that weather conditions like a heavy rain event several days in the past will be useful to forecasting the FMC, but the weather conditions earlier than one week ago are considered to be irrelevant. Larger sticks like 100h and 1000h fuels can have much longer memories. The FMC of a large log can be elevated relative to a baseline due to a rain event many days or even weeks earlier. These different time scales for modeling FMC is one of the primary motivations for this thesis. 

\subsection{Observations of FMC for Different Fuel Classes}

FMC can either be measured directly with weighing fuels or indirectly with sensors. Direct measurements of FMC are done through gravimetric measurement. Fuel samples are weighed to determine their wet weight, then oven-dried to determine their dry weight, and the FMC is calculated with~\eqref{eq:fmc}. Gravimetric measurement is accurate and direct, but it is slow and labor-intensive. The other main measurement technique is through fuel moisture sensors, which use various physical proxies to estimate FMC. The sensors allow for real-time observation of FMC, and they are used throughout the wildfire science literature. 

Direct gravimetric measurement of FMC is done in controlled field studies and also by field technicians from various wildfire management agencies. The field studies are carefully controlled; the exact time of sampling is known, and the dimensions and orientations of the sticks are controlled. Additionally, field studies typically measure the weather at the same locations and times, so the weather is co-located with the FMC samples. The downside of controlled field studies is that they are infrequently performed, and the available data is highly sparse. \citet{Carlson-2007-ANM} conducted a field study to calibrate the Nelson model for use within the NFDRS. The field study was conducted near the town of Slapout, Oklahoma, from the spring of 1996 through 1997. Arrays of wooden dowels representing 1h, 10h, 100h, and 1000h fuels were measured up to twice daily, usually around 13:00 and 22:00 central time. The wooden dowels were Ponderosa Pine and were manufactured to be the same dimensions. From now on, we will refer to the moisture content of those fuel classes as FM1, FM10, FM100, and FM1000. This field study will be referred to as the Oklahoma field study from now on. The Oklahoma field study was used to calibrate the Nelson model, as well as other physical-process based models~\citep{vanderKamp-2017-MFS}. \citet{Weise-2025-DFM} conducted a field study that observed the FM1 values of sticks of various species in Hawaii. There are very few other controlled field studies of FMC. 

Remote Automated Weather Stations (RAWS) provide real-time data from standardized sensors for FM10~\citep{NIFC-2024-RAW}. RAWS are ground-based weather stations with scientific instruments for acquiring weather data, operated by the National Interagency Fire Center~\citep{NIFC-2024-RAW} and other public and private organizations~\citep{Synoptic-2025-MSN}. Many RAWS have an FMC sensor, which consists of a \mbox{1/2-inch} pine dowel with a probe that measures moisture content in the woody interior~\citep{Campbell-2015-CFM, FTS-2016-FSS}. Weather stations with FMC sensors are maintained by many different private and public organizations, and the data is aggregated and provided by Synoptic Data~\citep{Synoptic-2025-MSN}. We will use the term RAWS to refer to any weather station in the Synoptic provider network with available FMC data. The network of RAWS provides a rich data set that is easily compatible with ML modeling. Many years of data are available at hundreds of physical locations. This allows for hyperparameter tuning and cross validation methods that account for the uncertainty in space and time. However, RAWS only provide data for FM10. This thesis is focused on leveraging the vast network of FM10 sensors to improve predictions for other fuel classes. 

Field researchers manually collect samples of FMC across CONUS as part of an effort by the United States Forest Service, called the Fire Environment Mapping System (FEMS)~\citep{USDA-FEMS-2024}. Field technicians are directed to gather samples according to a methodology that attempts to standardize the collection process~\citep{Zahn-2011-SFM}. Field sampling is slow and labor-intensive. Technicians collect samples in sealed containers and transport them back to research labs. The samples are weighed to determine their wet weight. Then, the samples are dried in an oven for multiple days to determine the dry weight. FEMS field samples of FMC are measured approximately every two weeks at certain sites. 

The FEMS field samples are highly sparse in time. The exact time of sampling is not always reported, and there is uncertainty about the exact time of sampling due to the labor intensive process of collecting samples across a landscape, then storing and transporting them for measurement. Diurnal cycles in FMC are important to consider. Fine fuels represented by FM1 change rapidly. A statistical model trained on weekly-resolution data will struggle to capture the high-frequency oscillations that are relevant to fire modeling. There is a considerable amount of time, up to several days, between the time of the observation and when that data becomes available. Further, the exact time of day of an FM1 sample is needed to effectively model the influence of changing environmental conditions. 

The FEMS field samples are also more irregular than the standardized wooden dowels used in field studies and fuel sensors. Fuel samples vary in size and species. Samples are taken directly from the ground, where they interact with soils that are heterogeneous in space and time. This type of variability is desirable in the sense that it captures the variability of natural landscapes where fires occur. However, these sources of variability make modeling difficult because there are numerous sources of error that are not precisely quantifiable.

For both the Oklahoma field study and the FEMS field samples, FMC samples are not collected during rain or immediately after when water is on the surface of the sticks~\citep{Carlson-2007-ANM, Zahn-2011-SFM}. The surface water makes gravimetric measurement impossible, since water on the surface of the stick that doesn't contribute to the FMC nonetheless affects the weight and distorts the FMC calculation. Extensive experimental and anecdotal evidence has shown that it is ineffective to shake off the surface water, and any attempts to dry off the surface water distort the FMC measurements~\citep{Zahn-2011-SFM}. Fuel sensors on RAWS will report values during rain, but these have not been specifically calibrated to gravimetric measurements due to the issues listed above. FMC can sharply increase during rain, particularly for fine fuels. The quick changes in FMC combined with the lack of direct observation of FMC during rain makes modeling very difficult. 

The research presented in this thesis uses the RAWS observations of FM10 as the primary source of data for training models because of the high quantities of data that are dense in both space and time. We will use the Oklahoma field study as the source of data to transfer the knowledge learned on the FM10 sensors to other fuel classes. The Oklahoma field study is used because the temporal resolution is much higher than the FEMS field observations: roughly twice daily versus every two weeks. This study is also used because of the physical similarity between the fuels and the wooden dowels used in the fuel sensors. Lastly, the Oklahoma field study was central to calibrating the state-of-the-art Nelson model~\citep{Carlson-2007-ANM}, and we can compare our trained models to the accuracy metrics from the hourly Nelson model for that data set. The Nelson model showed stronger accuracy when driven with 15-min spaced weather inputs, but we will compare to the hourly model as that is the temporal resolution used to train the RNN. We will discuss how we plan to use FEMS field samples and the Weiss field study in Section~\ref{sec:future}.

\subsection{Time-Lag Models of FMC}

Time-lag models for predicting FMC have been utilized for decades by researchers and in operational use by agencies. \citet{Fosberg-1971-DHT} developed time-lag equations for predicting the daily change in FM1 and FM10, where the FMC exponentially approaches the equilibrium moisture content calculated from mid-day temperature and RH. The Canadian Forest Fire Weather Index system developed the Fine Fuel Moisture Code, which models the daily change in the FMC for fine fuels using time-lag equations with the wetting and drying equilibria~\citep{VanWagner-1987-DSC}. The NFDRS uses a moisture model that extends on an hourly version of the Fine Fuel Moisture Code to account for the contributions from solar radiation and elevation~\citep{Rothermel-1986-MMC}.

WRF-SFIRE is a coupled atmosphere-wildfire model that uses a time-lag model of FMC as an input to determine wildfire rate of spread. \citet{Mandel-2014-RAA} model a fuel element as homogeneous and use an ordinary differential equation (ODE) implementing a version of the classical time-lag model, in which the rate of change in FMC is proportional to its difference from the equilibrium moisture content. \citet{Vejmelka-2014-DAF} added data assimilation of RAWS to the model. Real-time measurements from RAWS were extended by a spatial regression to grid nodes, then assimilated into a simple ODE time-lag model using the augmented Kalman Filter on each grid node. This model makes explicit use of the time lag concept in fuels. In operational use for wildfire simulations, this model generates forecasts for FMC of the different fuel classes in real-time. We refer to the time-lag ODE with Kalman filtering as the ``ODE+KF'' model. 

This model uses drying and wetting equilibrium moisture content, $E_d(t)$ and $E_w(t)$ respectively, and hourly accumulated rainfall $r(t)$ as weather inputs. The fuel class is denoted as $T_k$, where in practice $k=1, 10, 100,$ and $1000$. Data assimilation of FMC observations from FMC sensors at RAWS. During wildfire simulations, the WRF-SFIRE receives the initial fuel state from the ODE+KF after data assimilation, and then the ODE model inside WRF-SFIRE runs in forecast mode without data assimilation to update the FMC throughout the duration of the simulation. Equation \ref{eq:ode} describes the mathematical form for $m(t)$, the FMC at time $t$:
\begin{equation}
\label{eq:ode}
    \frac{d}{dt}m(t) = \begin{cases}
        \frac{S-m(t)}{T_r}\left(1-\exp\left(\frac{r_0 - r(t)}{r_s}\right)\right) & \text{if}\;\; r(t) > r_0 \text{ (soaking in rain)},\\
        \frac{E_d(t)-m(t)}{T_k} & \text{if}\;\; r(t)\leq r_0 \text{ and }m(t)>E_d(t),\\
        \frac{E_w(t)-m(t)}{T_k} & \text{if}\;\; r(t)\leq r_0 \text{ and } m(t)<E_w(t),\\
        0 & \text{if}\;\; r(t)\leq r_0 \text{ and } E_w(t)\leq m(t)\leq E_d(t).
    \end{cases}
\end{equation}
The ODE+KF model is essentially a collection of time-lag models that parameterize the rate of decay differently based on external conditions. If the amount of rainfall accumulated per hour $r(t)$ is greater than the threshold value $r_0$, it is assumed that the stick absorbs water directly, and it approaches a saturation moisture level $S$ with a characteristic delay time depending on rain intensity, asymptotically $T_r$ for a very intense rain. The saturation rainfall intensity $r_s$ is when rain wetting is $1-e^{-1}\approx 63\%$ of the maximal rate. The calibration of the model resulted in a threshold rainfall intensity of $r_0=0.05 \,\mathrm{mm}\,\mathrm{h}^{-1}$, where $\mathrm{mm}$ denotes millimeters, a saturation level $S=250\%$, a delay time $T_r = 14\,\mathrm{h}$, and a rain saturation intensity of $r_s=8 \,\mathrm{mm}\;\mathrm{h}^{-1}$~\citep{Mandel-2014-RAA}. These values will be referred to as hyperparameters for consistency with other models. If the rainfall $r(t)$ is less than the threshold value $r_0$, the stick changes moisture by equilibrating to the equilibrium moisture content of the surrounding environment. There are three distinct dynamics for the FMC in no-rain conditions. If the FMC $m(t)$ is greater than the current drying equilibrium moisture content $E_d(t)$, then drying dynamics occur and the FMC approaches $E_d(t)$ with a characteristic time-lag of $T_k$ determined by the fuel class ($T_k$ = 10 for 10-h fuels). If $m(t)$ is less than the current wetting equilibrium moisture content $E_w(t)$, then wetting dynamics occur and the FMC approaches $E_w(t)$ with the same characteristic time-lag of $T_k$. If $m(t)$ is between $E_d(t)$ and $E_w(t)$, there is no change in moisture content. We use the ODE+KF with the calibrated parameters described above to forecast the FMC at the locations not included for training the ML models and at times in the future of all ML training data. The ODE equations described above can be solved analytically, and a discretized version of the analytic solution is used to model FMC. In the model by \citet{Vejmelka-2016-DAD}, the state of the system was extended to include an equilibrium bias correction to allow the model to learn differences between the equilibrium moisture content values as determined by laboratory experiments versus what is suggested by the data. This FMC model with data assimilation is incorporated into the WRF-SFIRE modeling system~\citep{Mandel-2019-IDH}.

\citet{Mandel-2023-BFM} showed that a single simple RNN unit with linear activation could reproduce the analytical solution to a simplified version of the time-lag ODE. The weights of the simple RNN were manually set to reproduce the analytical solution to the ODE. This research established a deep mathematical connection between RNNs and time-lag models, which inspired the time-warp method in this thesis. Consider a simplified version of the system above with just one equilibrium value $E(t)$
\begin{equation}
\frac{d}{dt}m(t) = \frac{E(t)-m(t)}{T_k}.
\end{equation}
Under a zero-order hold assumption, this system has a discretized analytic solution that follows the derivation in~\eqref{eq:tl_discrete}
\begin{equation}\label{eq:tl_fmc}
    m_{t+1} = (1-e^{-1/T_k})\, E_{t+1} + e^{-1/T_k}\, m_t.
\end{equation}

This equation can be interpreted as a weighted average of the previous state and the current weather conditions. The weighting depends on the characteristic time scale determined by the fuel class. This equation can be reproduced with a simple RNN unit with linear activation by using $E_t$ as the input, setting the input weight $W_x = (1-e^{-1/T_k})$, setting the recurrent weight $W_h = e^{-1/T_k}$, and setting the bias to zero. 

Table~\ref{tab:classes} shows the resulting equation for the four standard fuel classes~\citep{NCEI-2024-DFM}. The FMC of smaller fuels is more strongly determined by the current weather, while the FMC of larger fuels is more strongly determined by the past conditions. In other words, 1h fuels have a relatively short memory, while 1000h fuels have a long memory. 

\begin{table}[htbp]
\centering
\caption[Characteristic Time Scales for FMC Fuel Classes.]{Characteristic Time Scales for FMC Fuel Classes.}
\label{tab:classes}
\begin{tabular}{|l|l|l|l|}
\hline
\textbf{Time-Lag (T)} & \textbf{Fuel Diameter} & $\pmb{m_{t+1} \approx}$ \\  \hline
1     & Less than 1/4 in. & $0.63\cdot E + 0.37\cdot m_t$ \\ \hline
10    & 1/4 - 1 in.       & $0.095\cdot E + 0.905\cdot m_t$ \\ \hline
100   & 1-3 in.           & $0.01\cdot E + 0.99\cdot m_t$  \\ \hline
1000   & 3-8 in.           & $0.001\cdot E + 0.999\cdot m_t$  \\ \hline
\end{tabular}
\end{table}

\subsection{The Nelson Model: A State-of-the-Art Physical Process-Based Model}

In reality, the change of FMC over time is more complicated than a simple time-lag system. Physical process-based methods attempt to directly represent the physical processes of energy and moisture transfer, and such models have been the dominant way to predict FMC in publicly-available wildfire modeling projects. \citet{Nelson-2000-PDC} developed a set of physical equations to model FMC as part of the National Fire Danger Rating System. The model represents a stick as a 1-dimensional system, where moisture and energy exchange act radially towards the center of the stick, with various boundary conditions accounting for exchange of moisture and energy between the stick surface and the environment. The model receives daily weather inputs of air temperature, RH, accumulated rainfall, and downward shortwave solar radiation.

The Nelson model is currently used by agencies and wildfire researchers to predict FMC for multiple dead fuel classes~\citep{Jolly-2024-MUN}. The equations of the model utilize the radius of the stick as a parameter, which is modified to generate predictions for sticks of different fuel classes based on the assumed dimensions of those idealized fuels. In practice, some model parameters are calibrated to field observations for the different fuel classes separately, while other parameters are held fixed. The mathematical form of the equations for FMC is unchanged when generating predictions for sticks of various sizes. This is analogous to a parameter-based transfer learning method. However, calibration of parameters in physics-based models is often done based on understanding of the physical processes, rather than statistical methods that choose a set of parameters using cross validation that minimizes a loss function on a test set. 

The Oklahoma field study provides the primary empirical evaluation of the Nelson model and calibrated parameter choices used in operational applications~\citep{Carlson-2007-ANM}. Additional work has explored ``local calibration'' of the model using field observations within systems such as FEMS~\citep{Jolly-2024-MUN}, but the Oklahoma study remains the main published evaluation of Nelson model accuracy. The original data described in \citet{Carlson-2007-ANM} will be used for an empirical analysis of the time-warping method. The analysis design and details on the data will be described in Section~\ref{sec:fmc_analysis} and results in Section~\ref{sec:fmc_results}. 

\subsection{The RNN Model of FMC}

RNNs can learn statistical patterns from data, and can more accurately predict FMC than the time-lag models~\citep{Hirschi-2026-RNN}. We use principles from the simple time-lag model to inform a new method of transfer learning with RNNs. We try to leverage the insights from the time-lag concept to improve the accuracy of a modern RNN.

In \citep{Hirschi-2026-RNN}, we develop an RNN to forecast FM10 and evaluated the accuracy compared to observations from RAWS. This is the source learning task, and the pretrained RNN is the starting point for the transfer learning methods. The model receives inputs from time series of weather observations. To make the model compatible with real-time forecasting, the model was trained with weather forecasts from a numerical weather prediction (NWP). The model showed substantial accuracy improvements over several baselines, including the ODE plus Kalman filter method described above. The paper is included in full in Appendix~\ref{app:paper}. The paper includes an extensive literature review on models of FMC. Other researchers have used ML techniques to predict FM10, but this publication represents the first recurrent model compatible with real-time forecasting that has been extensively validated. 

The pretrained RNN will be used in the transfer learning analysis in this thesis. The goal is to leverage the learned patterns from the FM10 model to improve predictions of the other fuel classes, including FM1, FM100, and FM1000. Time-warping the learned dynamics of the RNN will be evaluated and compared to other standard methods of transfer learning. Chapter~\ref{chapter:03} will include relevant details of the RNN architecture and training and describe how that model will be adapted to produce predictions for the other fuel classes.

    \chapter{\uppercase{Time-Warping with RNNs}}\label{chapter:03}
\section{Mathematically Deriving a Time-Warped LSTM}\label{sec:math}

We prove several propositions related to time-warping dynamic systems with RNNs, building to the conclusion that there exists a set of weights for an LSTM where the dynamics can be time-warped with any desired accuracy. We show that an RNN with a recurrent layer consisting of a single unit with linear activation can exactly reproduce a first-order linear discrete time system. Then, we prove that a simple RNN can approximate the system with any desired accuracy when the activation is the standard tanh. We show how the parameters of a simple RNN can be modified to account for a time-warping by a constant factor of the original system. We prove that for a class of first-order linear discrete-time systems, there exists a set of weights in an LSTM where the time-warping can be made arbitrarily precise over a finite sequence. We prove that an LSTM with linear activation functions instead of tanh for the two instances of tanh in Equations~\eqref{eq:lstm_eqs} can approximate the time-lag system and can be time-warped with arbitrary accuracy. 

Consider the first-order linear discrete time system defined by \eqref{eq:timelag},
\[
z_{t} = az_{t-1} + (1-a)X_t, \quad 0<a<1,
\]
where $a\in(0,1)$ is fixed and $(X_t)_{t\geq 1}$ is an infinite sequence of external inputs. The condition $0<a<1$ ensures exponential stability of the homogeneous recursion $z_t = az_{t-1}$~\citep[Sect.~3.2.3]{Box-2016-TSA}. Assume that $(X_t)_{t\geq 1}$ is bounded, that is $\sup_t|X_t|\leq M$ for $M < \infty$ and all $t\geq1$. Given an initial condition $z_0$, the recursion generates a sequence $(z_t)_{t\geq 0}$. The boundedness of $(X_t)_{t\geq 1}$ and $0<a<1$ ensures boundedness of $(z_t)_{t\geq 0}$. 

\begin{lemma}\label{lem:bounded}
    Let $(z_t)_{t\geq 0}$ be a sequence defined by the relationship in Equation~\eqref{eq:timelag} with initial condition $|z_0|<\infty$, where  $a\in (0,1)$, and $(X_t)_{t\geq 1}$ is bounded. Then $(z_t)_{t\geq 0}$ is bounded and has the closed-form
\[
z_t = a^{t}z_0 + (1-a)\sum_{k=0}^{t-1}a^k X_{t-k}, \quad t\geq 1
\]
\end{lemma}
\begin{proof}
We derive the closed-form for $(z_t)$ using induction. At $t=1$ the formula yields, 
\begin{align*}
    z_1 = az_0 + (1-a)X_1,
\end{align*}
which agrees with the recurrence relation. Suppose that at time $t-1$ the formula holds,
\begin{align*}
    z_{t-1} &= a^{t-1}z_0 + (1-a)\sum_{k=0}^{t-2}a^k X_{(t-1)-k}.
\end{align*}
Thus, at time $t$,
\begin{align*}
    z_{t} &= az_{t-1} + (1-a)X_{t}\\
    &= a\left(a^{t-1}z_0 + (1-a)\sum_{k=0}^{t-2}a^kX_{(t-1)-k} \right) + (1-a)X_{t}\\
    &= a^{t}z_0 + (1-a)\sum_{k=0}^{t-2}a^{k+1} X_{(t-1)-k} + (1-a)X_t\\
    &= a^{t}z_0 + (1-a)\sum_{k=1}^{t-1}a^k X_{t-k} + (1-a)X_t\\
    &= a^{t}z_0 + (1-a)\sum_{k=0}^{t-1}a^k X_{t-k}
\end{align*}
This completes the proof by induction. 

Since $(X_t)_{t\geq 1}$ is bounded, there exists $M$ such that for all $t\geq 0$,
$\sup|X_t|\leq M< \infty$. The boundedness of the sequence follows from properties of the geometric series,
\begin{align*}
    |a^tz_0 + (1-a)\sum_{k=0}^{t-1}a^k X_{t-k}| &\leq |a^tz_0| + (1-a)\sum_{k=0}^{t-1}a^k|X_{t-k}|\\
    &\leq |z_0| + (1-a)M\sum_{k=0}^{t-1}a^k\\
    &\leq |z_0| + (1-a)M\sum_{k=0}^\infty a^k\\
    &\leq |z_0| + (1-a)M\frac{1}{1-a}
\end{align*}
Where the last step is true by the summation of power series. Thus,
\begin{align}
    |z_t| &\leq |z_0| + M < \infty.
\end{align}
Thus $|z_t| \le |z_0| + M$ for all $t\ge0$, and therefore $(z_t)_{t\geq 0}$ is bounded.
\end{proof}

A time-warping by $\gamma>0$ results in the modified linear first-order system~\eqref{eq:timewarp}, where the parameter $a$ is raised to the power $\gamma$,
\begin{equation}
\begin{aligned}
    \widetilde z_{t} = a^\gamma \widetilde z_{t-1} + (1-a^\gamma)X_t, \quad 0<a<1, \; \gamma>0
\end{aligned}
\end{equation}
If $\gamma >1$, this decreases the dependence on the past state and increases the dependence on the current input, thus speeding up the rate at which the system equilibrates to $X_t$. Conversely, if $\gamma < 1$, this increases the dependence on the past state and decreases the dependence on the current input, thus slowing down the rate at which the system equilibrates to $X_t$.

In practice, an exact closed-form for the data generating process is not known, and linear activation is unstable for training simple RNNs. Thus, tanh activation, or another nonlinear activation function, is typically used together with one or more dense layers downstream of the recurrent layer that modify the output. The nonlinearity of the tanh function help the RNN approximate general dynamic systems. We show that for a simple RNN with tanh activation, there exists a range of inputs in the near-linear portion of the tanh function where the simple RNN can approximate the dynamic system arbitrarily well. 

Given an initial condition $z_0$ and a finite input sequence $(X_t)_{t=1}^N$ that is bounded, the resulting trajectory $(z_t)_{t=1}^N$ is generated by the time-lag system defined in \eqref{eq:timelag}. We say that a model approximates $(z_t)_{t=1}^N$ with any desired accuracy, or ``arbitrary accuracy'', if for any $\epsilon>0$, a set of model weights exists such that the sequence of model predictions $(\hat z_t)_{t=1}^N$ satisfies
\begin{align}
    \max_{t=1,\dots, N} |z_t - \hat z_t| < \epsilon.
\end{align}

\subsection{Simple RNN Time-Warp}

The first proposition shows that a simple RNN can exactly reproduce the time-lag system. Proposition~\ref{prop:simplernn_linear} does not require finite sequences nor does it require boundedness of the inputs. The simple RNN with linear activation is exactly a time-lag system. 

\begin{proposition}\label{prop:simplernn_linear}
For any infinite input sequence $(X_t)_{t\geq 1}$, any fixed parameter $a\in(0,1)$, and any initial condition $z_0$, there exists a choice of weights and an initial hidden state for a simple recurrent neural network with a single recurrent unit and linear activation such that the hidden state satisfies
\[
h_t = z_t \quad \text{for all } t\geq 0,
\]
where $(z_t)_{t\geq 0}$ is the unique solution given by~\eqref{eq:timelag} 
\[
z_{t} = az_{t-1} + (1-a)X_t, \quad t\geq 1
\]
\end{proposition}
\begin{proof}
Equations~\eqref{eq:simplernn_eqs} define the simple RNN, 
\[
\phi(W_x\,X_t + W_h\,h_{t-1} + b).
\]
With linear activation, the activation function is the identity function $\phi(x) = x$. Setting the recurrent weight $W_h = a$, the input weight $W_x=1-a$, and the bias term $b = 0$ gives
\[
h_t = (1-a)X_t + a h_{t-1}.
\]
which is the same as~\eqref{eq:timelag} with $h_t$ in place of $z_t$. It remains to take $h_0=z_0$. By induction, $h_t = z_t$ for all $t\geq 0$.
\end{proof}

Under a time-warp $\gamma>0$ by~\eqref{eq:timewarp}, the weights of a simple RNN constructed in Proposition~\ref{prop:simplernn_linear} can be modified so that its hidden state exactly produces $(\widetilde z_t)$ by shifting the forget and input gate biases.

\begin{proposition}\label{prop:simplernn_warp}
For any infinite input sequence $(X_t)_{t\geq 1}$, parameter $a\in (0,1)$, initial condition $z_0$, and $\gamma>0$, let $(z_t)_{t\geq 0}$ be the sequence given by~\eqref{eq:timelag}, and let $(\widetilde z_t)_{t\geq 0}$ be the time-warped sequence defined by~\eqref{eq:timewarp} with $\tilde z_0 = z_0$,
\[
\widetilde z_{t} = a^\gamma \widetilde z_{t-1} + (1-a^\gamma)X_t, \quad t \geq 1
\]
Let an RNN with a single recurrent unit and linear activation have weights chosen as in Proposition~\ref{prop:simplernn_linear} so that it's hidden state $(h_t)_{t\geq 0}$ satisfies $h_t=z_t$ for all $t\geq 0$. Then there exists a modification of these weights such that the resulting hidden state $(\widetilde h_t)_{t\geq 0}$ satisfies
\[
\tilde h_t = \tilde z_t, \quad \text{for all } t\geq 0.
\]
\end{proposition}
\begin{proof}
From Proposition~\ref{prop:simplernn_linear}, we have $W_h=a$ and $W_x=1-a$. Define the modified weights
\begin{equation}\label{eq:simple_timewarped_params}
\begin{aligned}
    \widetilde W_h &= (W_h)^\gamma \\
    \widetilde W_x &= 1-(1-W_x)^\gamma    
\end{aligned}
\end{equation}
The time-warped parameters satisfy
\begin{align*}
    \widetilde W_h &= a^\gamma \\
    \widetilde W_x &= 1-a^\gamma
\end{align*}
The resulting sequence of hidden states $(\widetilde h_t)_{t\geq 0}$ satisfies
\[
\widetilde h_t = a^\gamma\; \widetilde h_{t-1} + (1-a^\gamma)X_t, \quad t\geq 1
\]
with initial condition $\widetilde h_0=z_0$. This matches the recursion and initial condition of the time-warped system defined in \eqref{eq:timewarp}. Thus,
\[
\widetilde h_t = \widetilde z_t \quad \text{for all } t\geq 0.
\]
Therefore, the simple RNN can be time-warped by $\gamma>0$.
\end{proof}

The previous propositions relied on linear activation for a simple RNN unit, which in practice leads to instability. The tanh function is a more common choice for the activation function that addresses some of the issues related to exploding and vanishing gradients. With the nonlinear tanh activation function, the simple RNN can approximate the time-lag system with arbitrary accuracy. The proof relies on the fact that the tanh function is approximately linear near zero. If the sequence and the external inputs are scaled to the near-linear range of the function, a simple RNN with tanh activation can approximate the system. Note that this isn't intended to be a practical method for transforming empirical observations, rather it is to show that for an RNN with the standard tanh activation function, there still exists a deep relationship with the time-lag system. 

\begin{proposition}\label{prop:simple_tanh}
For any $\epsilon>0$, finite input sequence $(X_t)_{t=1}^N$, initial condition $z_0$, and parameter $a\in(0,1)$, there exists a choice of weights and an initial hidden state for a simple recurrent neural network with a single recurrent unit and tanh activation such that the sequence of hidden states $(h_t)_{t=1}^N$ approximates the sequence $(z_t)_{t=1}^N$ defined by~\eqref{eq:timelag}, 
\[ 
z_{t} = az_{t-1} + (1-a)X_t, \quad t=1,\ldots, N, \]
with accuracy $\epsilon$, i.e.
\[
\max_{t=1, \dots, N} |z_t - h_t| < \epsilon.
\]
\end{proposition}
\begin{proof}
Since $(X_t)_{t=1}^N$ is bounded, it can be extended to an infinite bounded sequence. By Lemma~\ref{lem:bounded}, the corresponding sequence $(z_t)_{t\geq0}$ is bounded, so $(z_t)_{t=1}^N$ is bounded. Let $B_X$ and $B_z$ be the uniform bounds for the sequences, 
\begin{align*}
    |X_t|&\leq B_X, \quad t=1,\ldots, N \\
    |z_t|&\leq B_z \quad t=1,\ldots, N
\end{align*}
Define $B=\max\{B_X, B_z\}$. Choose $\delta>0$ such that
\begin{align}\label{eq:delta1}
    \delta < \left[\epsilon\cdot\frac{3}{B}\cdot \frac{1-a}{1-a^N}\right]^{1/2}
\end{align}
Define a linear rescaling of the sequences $(X_t)_{t=1}^N$ and $(z_t)_{t=1}^N$ into the interval $[-\delta, \delta]$ by,
\begin{align*}
    z'_t &:= \left(\frac{\delta}{B}\right)z_t,
    \qquad z'_0  := \left(\frac{\delta}{B}\right)z_0\\
    X'_t &:= \left(\frac{\delta}{B}\right)X_t
\end{align*}
Construct the simple RNN unit as in Proposition~\ref{prop:simplernn_linear}. That is, set the recurrent weight $W_h = a$, the input weight $W_x=1-a$, and the bias term $b = 0$. The hidden state evolves by the equation
\[
h_t = \tanh(ah_{t-1} + (1-a)X'_t).
\]
Define $u_t$ as the pre-activation for the hidden state
\[
u_t = ah_{t-1} + (1-a)X'_t.
\]
Let the initial hidden state be $h_0=z'_0$, so $h_0\in[-\delta, \delta]$. Assume $|h_{t-1}|\leq \delta$. Then,
\[
|u_t|\leq a\delta + (1-a)\delta = \delta.
\]
Since $|\tanh(u_t)|\leq |u_t|$, it follows that $|h_t|\leq \delta$. By induction, $h_t\in[-\delta, \delta]$ for $t=1, \dots, N$. Define $e_t$ to be the error at time $t$,
\begin{align*}
    e_t &= z'_t - h_t \\
    &= (az'_{t-1} + (1-a)X'_t) - (u_t + (\tanh(u_t) - u_t))\\
    &= ae_{t-1} + u_t - \tanh(u_t)
\end{align*}
Applying the triangle inequality and the boundedness of $u_t$, we have
\begin{align*}
    |e_t|&\leq a|e_{t-1}| + |u_t - \tanh(u_t)|\\
    &\leq a|e_{t-1}| + \sup_{u\in [-\delta, \delta]} |u - \tanh(u)|
\end{align*}
Define $\eta = \sup_{u\in [-\delta, \delta]} |u - \tanh(u)|$. We show that the errors for $t=0,\dots, N$ are bounded by,
\[
|e_t|\leq \eta \sum_{k=0}^{t-1}a^k.
\]
With $e_0=0$ the inequality holds. Assume $|e_{t-1}|$ satisfies the inequality, so
\begin{align*}
    |e_t| &\leq a|e_{t-1}|+\eta
    \leq a|\eta \sum_{k=0}^{t-2}a^k| + \eta
    = \eta \sum_{k=1}^{t-1}a^k + \eta
    = \eta \sum_{k=0}^{t-1}a^k.
\end{align*}
By induction, $|e_t|$ is bounded by $\eta\sum_{k=0}^{t-1}a^k$. Using properties of geometric series, we have
\[
\max_{t=1, \dots, N} |e_t| \leq \eta \frac{1-a^N}{1-a}.
\]
To bound $\eta$, we estimate the difference of $\tanh$ and the linear function
\begin{align*}
    u-\tanh(u) &= \int_{0}^u (\tanh'(s)-1)ds\\
    &= \int_0^u (\text{sech}^2(s)-1)ds\\
    &= -\int_0^u \tanh^2(s)ds
\end{align*}
And for all $u\in[-\delta, \delta]$, corresponding to the bounded pre-activation, taking absolute values we have
\begin{align*}
    |u-\tanh(u)| &\leq \int_0^{|u|}\tanh^2(s)ds\\
    &\leq \int_0^{|u|} s^2ds\\
    &\leq \frac{|u|^3}{3}
\end{align*}
Thus,
\begin{align*}
    \eta = \sup_{u\in [-\delta, \delta]} |u - \tanh(u)| \leq \sup_{u\in [-\delta, \delta]} \frac{|u|^3}{3} = \frac{\delta^3}{3}
\end{align*}
And thus, the error between the simple RNN and the rescaled sequence defined earlier is bounded by,
\begin{align*}
    \max_{t=1, \dots, N} |e_t| \leq \frac{\delta^3}{3} \left(\frac{1-a^N}{1-a}\right)
\end{align*}
The hidden state can be reverse-transformed to approximate the original sequence $(z_t)$ by,
\[
\hat z_t := \frac{B}{\delta}h_t
\]
So the error between the reverse-transformed RNN output and the sequence $(z_t)$ is
\[
|z_t - \hat z_t| \leq \left|\frac{B}{\delta}z'_t - \frac{B}{\delta}h_t\right| = \frac{B}{\delta}|e_t|
\]
The universal bound of this error over the finite sequence is,
\begin{align*}
    \max_{t = 1, \dots, N}|z_t - \hat z_t| &\leq \frac{B}{\delta}\max_{t = 1, \dots, N}|e_t|\\
    &\leq \frac{B}{\delta}\frac{\delta^3}{3} \left(\frac{1-a^N}{1-a}\right)\\
    &\leq \frac{B\delta^2}{3}\left(\frac{1-a^N}{1-a}\right)
\end{align*}
With the choice of $\delta$ in \eqref{eq:delta1}, the error is bounded by $\epsilon$. Therefore, a simple RNN with a single recurrent unit and tanh activation can approximate, to any desired accuracy on a finite sequence, the trajectory generated by the linear recursion \eqref{eq:timelag}.
\end{proof}

\subsection{LSTM Time-Warp}\label{sec:lwarp}

We first prove that an LSTM can approximate a first-order linear system with arbitrary accuracy in the special case when linear activation is used instead of tanh in two locations. In the standard formulation of the LSTM equations in~\eqref{eq:lstm_eqs}, the candidate memory $g_t$ uses tanh activation, and the final output of the hidden state $h_t$ applies a tanh to the cell state. We refer to this as an LSTM with tanh activation. When both of these instances of tanh are changed to the identity function, we refer to this as an LSTM with linear activation. The other three gates, the forget gate, the input gate, and the output gate, will have the standard sigmoid activation in all cases in this thesis. Standard software packages implement the choice of activation functions this way, such as the Keras set of software tools \citep{Keras-2025-Activation}. The accuracy of the approximation is controlled by how closely the output gate $o_t$ approximates 1. If the output gate $o_t$ was set exactly to 1, then a set of weights for an LSTM would exist that would exactly reproduce the time-lag system, but then the LSTM would reproduce a simple RNN and the proof would be a trivial extension of Proposition~\ref{prop:simplernn_linear}, so it is not shown. 

The derivations will make use of the inverse function of the sigmoid~\eqref{eq:sigmoid}, known as the logit function,
\begin{align}\label{eq:logit}
    \text{logit} : (0,1) \to \mathbb{R}, \quad
    \text{logit}(x) = \sigma^{-1}(x) = \ln\left(\frac{x}{1-x}\right),
\end{align}
where $\sigma^{-1}$ denotes the inverse function satisfying $\sigma^{-1}(\sigma(z)) = z$.

\begin{proposition}\label{prop:lstm_linear}
For any $\epsilon>0$, finite input sequence $(X_t)_{t=1}^N$, parameter $a\in(0,1)$, and initial condition $z_0$, there exists a choice of weights and an initial recurrent state for an LSTM with one unit and linear activation such that the sequence of hidden states $(h_t)_{t=1}^N$ approximates the sequence $(z_t)_{t=1}^N$ defined by~\eqref{eq:timelag} 
\[ 
z_{t} = az_{t-1} + (1-a)X_t, \quad t=1,\ldots, N, \]
with accuracy $\epsilon$, i.e.
\[
\max_{t=1,\dots, N} |z_t - h_t| < \epsilon.
\]
\end{proposition}
\begin{proof}
We first choose gate weights so that the LSTM cell behaves like a simple RNN, and then use a proof similar to the proof of Proposition~\ref{prop:simple_tanh}. Since $(X_t)_{t=1}^N$ is finite, it is bounded by $M>0$. By Lemma~\ref{lem:bounded}, the sequence is bounded by $|z_t|\leq M$ for $t=1, \dots, N$. 

Choose the following weights and biases of the LSTM cell. In the forget gate set the weight $W_{xf}=0$, the recurrent weight $W_{hf}=0$, and choose $b_f = \text{logit}(a)$. In the input gate set the weight $W_{xi}=0$, the recurrent weight $W_{hi}=0$, and choose $b_i = \text{logit}(1-a)$. For the candidate memory set the weight $W_{gx}=1$, the recurrent weight $W_{hg}=0$, and the bias $b_g=0$. Finally, for the output gate set the weight $W_{xo}=0$, the recurrent weight $W_{ho}=0$, and choose
\begin{align}
    b_o>\text{logit}(1-\epsilon/M).
\end{align}
Therefore, the bias $b_o$ was chosen so that $o_t$ is approximately equal to 1. Since $\sigma(x)\in(0,1)$, we have $1-\epsilon/M< \sigma(b_0)<1$. These choices of weights result in the following equations for the four LSTM gates,
\begin{equation}
\label{eq:lstm_linear}
\begin{aligned}
    f_t &= a\\
    i_t &= 1-a\\
    g_t &= X_t\\
    1-\epsilon/M< o_t & < 1
\end{aligned}
\end{equation}
The corresponding equations for the cell state and the hidden state of the LSTM cell are,
\begin{equation}
\begin{aligned}
    c_t &= a \cdot c_{t-1} + (1-a) \cdot g_t,\\
    h_t &= o_t\cdot c_t.
\end{aligned}
\end{equation}
Choose the initial recurrent state to be $[z_0, 0]$, where the initial cell state $c_0$ is equal to $z_0$ and the initial hidden state $h_0$ is zero. By Lemma~\ref{lem:bounded}, the resulting trajectory $(c_t)_{t=1}^N$ is bounded by $|c_t|\leq M$ for all $t=1,\dots, N$, and has the closed-form
\begin{align*}
    c_t = a^tz_0 + (1-a)\sum_{k=0}^{t-1}a^k X_{t-k} = z_t.
\end{align*}
So, the cell state $c_t$ is equal to $z_t$ for all $t=1,\dots, N$.

Since the weights for the output gate are zero except for a constant bias, the output gate is constant in time. Using the bounds on $o_t$ in~\eqref{eq:lstm_linear},
\begin{align*}
    \max_{0\leq t\leq N}|z_t - h_t| &= \max_{0\leq t\leq N} |c_t - o_tc_t| \\
    &= \max_{0\leq t\leq N}|(1-o_t)c_t|\\
    &= (1-o_t)\max_{0\leq t\leq N}|c_t|\\
    &\leq (1-o_t)M\\
    &< (1-(1-\epsilon/M))M = \epsilon.\\
\end{align*}
Thus, 
\[\max_{t=1,\ldots, N}|z_t - h_t|<\epsilon.\] 
Therefore, the LSTM with the chosen weights and biases approximates the sequence $(z_t)_{t=1}^N$ with accuracy $\epsilon$.
\end{proof}

We now show that this construction extends naturally to the time-warped system defined in~\eqref{eq:timewarp} by modifying the forget and input gate bias terms in the LSTM. So, an LSTM with linear activation can be time-warped to approximate a time-warped time-lag system with arbitrary accuracy. 

\begin{proposition}\label{prop:lstm_warp}
For any infinite input sequence $(X_t)_{t\geq 1}$, parameter $a\in (0,1)$, initial condition $z_0$, and $\gamma>0$, let $(z_t)_{t\geq 0}$ be the sequence given by~\eqref{eq:timelag}, and let $(\widetilde z_t)_{t\geq 0}$ be the time-warped sequence defined by~\eqref{eq:timewarp} with $\tilde z_0 = z_0$,
\[
\widetilde z_{t} = a^\gamma \widetilde z_{t-1} + (1-a^\gamma)X_t, \quad t \geq 1
\]
Let an LSTM with a single recurrent unit and linear activation have weights chosen as in Proposition~\ref{prop:lstm_linear} such that its hidden state $(h_t)_{t\geq 0}$ satisfies
\[
\max_{t=1,\dots,N} |z_t - h_t| < \epsilon.
\]
Then there exists a modification of the forget and input gate bias parameters such that the resulting hidden state $(\widetilde h_t)_{t\geq 0}$ satisfies
\[
\max_{t=1, \dots, N} |\widetilde z_t - \widetilde h_t| < \epsilon.
\]
\end{proposition}
\begin{proof}
From Proposition~\ref{prop:lstm_linear}, $b_f = \text{logit}(a)$, $b_i = \text{logit}(1-a)$, and the bias for the output gate is chosen to be
\[
b_o>\text{logit}(1-\epsilon/M).
\]
Define the modifies weights,
\begin{equation}\label{eq:lstm_linear_timewarp}
\begin{aligned}
    \widetilde b_f &= \text{logit}(a^\gamma)
    = \text{logit}(\sigma(b_f)^\gamma)\\
    \widetilde b_i &=\text{logit}(1-a^\gamma)
    = \text{logit}(1-(1-\sigma(b_i))^\gamma)
\end{aligned}
\end{equation}
The resulting cell state satisfies
\[
\widetilde c_t = a^\gamma\; \widetilde c_{t-1} + (1-a^\gamma)X_t, \quad t\geq 1.
\]
with initial recurrent state $[c_0, h_0] = [\widetilde z_0, 0]$. The cell state is thus equal to $\widetilde z_t$ for all $t=1,\dots, N$. The hidden state $\widetilde h_t=o_t\cdot \widetilde c_t$. By Proposition~\ref{prop:lstm_linear}, 
\[
\max_{t=1,\dots N}|\widetilde z_t - \widetilde h_t|<\epsilon
\]
Therefore, the LSTM that approximates $(z_t)_{t=1}^N$ with accuracy $\epsilon$ can be time-warped by $\gamma>0$ by modifying the forget and input gate biases so that it approximates $(\widetilde z_t)_{t=1}^N$ with accuracy $\epsilon$.
\end{proof}

\subsection{Discussion}

The Propositions proved in this section provide a theoretical motivation for time-warping a pretrained RNN by modifying the forget and input gate biases. For the class of time-lag models, Proposition~\ref{prop:simplernn_linear} shows that a simple RNN can exactly reproduce the system, and Proposition~\ref{prop:simplernn_warp} shows that a time-warping of the dynamic system can be reproduced by modifying the weights of a simple RNN. Proposition~\ref{prop:lstm_linear} shows that an LSTM can approximate time-lag systems with arbitrary accuracy, and Proposition~\ref{prop:lstm_warp} shows that a time-warping of the dynamic system can be approximated with arbitrary accuracy by modifying only the forget and input gate bias terms. 

The proofs in this chapter are for a simple dynamic system and assume that the weights of the LSTM are specifically chosen to match the true dynamics. In practice, an LSTM trained on an arbitrary sequence learning task will have a complicated set of temporal dependencies, and it would be intractable to analytically derive how to time-warp the network. The purpose of these proofs is to motivate a general strategy for time-warping an LSTM that has been trained on a set of empirical observations. Different time scaling constants for the bias terms in the forget gate and input gates can be evaluated empirically. The proofs in this section give a theoretical justification for this strategy, since there exist dynamic systems and certain sets of weights where the time-warping is exact, or arbitrarily accurate.

The theoretical time-warping formulas for the LSTM biases shown in Equation~\eqref{eq:lstm_linear_timewarp} will be applied numerically using an example from Newton's Law of Cooling. Whether the approximate time-warp property is useful for arbitrary trained RNNs outside the near-linear regime will be studied with an empirical example on modeling FMC for different fuel classes.

\section{Proposed Time-Warping Method}

Based on the time-warping concept, we propose a novel approach to transfer learning for time series models. In this section, we describe a practical way to use the time-warping concept to post-hoc manipulate the learned dynamics of a pretrained RNN. Start with an RNN that is pretrained on a source learning task at one characteristic time scale. Then, if observations are available for the target learning task, empirically evaluate different time-warping parameters on the target learning task. 

For a time-warping constant $\gamma>0$, denote the time-warping functions
\begin{equation}
\label{eq:scaling_constants}
\alpha_f(\gamma), \qquad \alpha_i(\gamma),
\end{equation}
where $\alpha_f(\gamma)$ shifts the forget gate bias and $\alpha_i(\gamma)$ shifts the input gate bias. The modified gate biases are
\begin{equation}
\label{eq:timewarped_bias}
\begin{aligned}
    \widetilde b_f &= b_f + \alpha_f(\gamma) \\
    \widetilde b_i &= b_i + \alpha_i(\gamma)
\end{aligned}
\end{equation}

Shifting the bias terms results in increasing or decreasing the associated gate activation. This changes the dependence on the past state of the system versus the dependence on the current input data. The resulting time-warp changes how quickly the system approaches an equilibrium. 

For $\gamma > 1$, the dynamics are slowed down. This corresponds to increasing the forget gate activation with $\widetilde b_f > b_f$ and decreasing the input gate activation with $\widetilde b_i < b_i$. Slowing down the dynamics can be achieved by shifting the bias terms with $\alpha_f(\gamma) >0$ and $\alpha_i(\gamma)<0$. 

For $\gamma < 1$, the dynamics are sped up. This corresponds to decreasing the forget gate activation with $\widetilde b_f < b_f$ and increasing the input gate activation with $\widetilde b_i > b_i$. Speeding up the dynamics can be achieved by shifting the bias terms with $\alpha_f(\gamma) <0$ and $\alpha_i(\gamma)>0$. 

Note that the expected signs of the time-warping functions~\eqref{eq:scaling_constants} are based on theoretical considerations, under the assumptions that the data generating process is a time-lag system and the weights of the LSTM lead to well-approximating the system. For an arbitrary pretrained LSTM, there is no guarantee that the time-warping functions that lead to more accurate predictions will match the theoretical expectations. 

The propositions above considered the time-lag system~\eqref{eq:timelag}, where the time-warping could be made arbitrarily accurate with a scalar-valued, multiplicative scaling factor that is constant in time. For an LSTM layer with $N$ units, the time-warping functions map a scalar-valued $\gamma$ to a vector of length $N$, or $\alpha_f, \alpha_i: \mathbb R\to \mathbb R^N$. In practice, we will shift the bias parameters by adding a time-warping factor, since if a bias parameter ends up as zero, it will not be possible to shift it with a multiplicative time-warp. We define a time-warping function as ``static'' if it is constant in time and ``homogeneous'' if it applied the same bias shift for each unit in an LSTM layer. In contrast, a ``time-varying'' time-warping function could make shift the gate biases by different amounts over time, and a ``heterogeneous'' time-warping function could assign a different value to different units within the LSTM layer. 

This thesis focuses on static and homogeneous cases, but we introduce the time-varying and heterogeneous cases to generalize the method and allow for future research developments. The time-warping functions are a single number in the static and homogeneous case, so we denote them as the ``time-warping parameters''. The empirical analysis with FMC will explore how to optimize these time-warping parameters. 

For a given time-warping constant $\gamma$, the parameters $\alpha_f$ and $\alpha_i$ can be optimized using a grid search if observations are available for the target learning task. Table~\ref{tab:twarp_algorithm} describes an algorithm for optimizing the time-warping parameters with no fine-tuning, so only two parameters are fit to data that shift the forget and input gate bias terms.

\begin{table}[htbp]
\centering
\captionsetup{labelformat=empty}
\caption[Algorithm 1: Time-warping parameter grid search for LSTM]{\textbf{Algorithm 1:} Grid search over time-warping parameters (forget and input gate bias shifts) for a pretrained LSTM under fixed temporal scaling, with model selection based on training performance and final evaluation on a test set.}
\label{tab:twarp_algorithm}
\begin{tabular}{p{0.95\linewidth}}
\hline
\textbf{Setup:} Pretrained LSTM with bias vectors $(b_f, b_i)$; training set $\mathcal{D}_{\text{train}}$; test set $\mathcal{D}_{\text{test}}$; fixed temporal scaling $\gamma$; specified ranges of scalar candidate values for $\alpha_f$ and $\alpha_i$. \\
\\
\textbf{1.} Construct a grid of candidate values for $\alpha_f$ and a grid of candidate values for $\alpha_i$. \\
\textbf{2.} Form all possible pairs $(\alpha_f, \alpha_i)$ by combining each value of $\alpha_f$ with each value of $\alpha_i$. \\
\\
\textbf{3.} For each pair $(\alpha_f, \alpha_i)$: \\
\hspace*{1em} \textbf{3.1} Shift LSTM biases: $\widetilde{b}_f = b_f + \alpha_f$, $\widetilde{b}_i = b_i + \alpha_i$. \\
\hspace*{1em} \textbf{3.2} Generate predictions and evaluate performance on $\mathcal{D}_{\text{train}}$. \\
\textbf{4.} Select the pair $(\alpha_f, \alpha_i)$ that gives the highest training performance. \\
\\
\textbf{5.} Construct the final model using the selected $(\alpha_f, \alpha_i)$. \\
\textbf{6.} Generate predictions and evaluate performance on $\mathcal{D}_{\text{test}}$. \\
\\
\textbf{Output:} Test performance and selected bias shifts $(\alpha_f, \alpha_i)$. \\
\\
\textbf{Note:} This procedure can be repeated across multiple random initializations and data splits to assess robustness. \\
\hline
\end{tabular}
\end{table}

Fine-tuning can be performed by utilizing a validation set. Table~\ref{tab:twarp_finetune_algorithm} describes optimizing the time-warping parameters while utilizing fine-tuning. Note that the motivation for this method is in a transfer learning context where observations for the target learning task are sparse, so fine-tuning risks overfitting to a small data set. Time-warping with no fine-tuning, by contrast, is less prone to overfitting since it fits just two parameters and shifts a small subset of the LSTM layer parameters. 

\begin{table}[htbp]
\centering
\captionsetup{labelformat=empty}
\caption[Algorithm 2: Time-warping parameter grid search with fine-tuning for LSTM]{\textbf{Algorithm 2:} Grid search over time-warping parameters (forget and input gate bias shifts) for a pretrained LSTM under fixed temporal scaling, with fine-tuning on a training set, model selection based on validation performance, and final evaluation on a test set.}
\label{tab:twarp_finetune_algorithm}
\begin{tabular}{p{0.95\linewidth}}
\hline
\textbf{Setup:} Pretrained LSTM with bias vectors $(b_f, b_i)$; training set $\mathcal{D}_{\text{train}}$; validation set $\mathcal{D}_{\text{val}}$; test set $\mathcal{D}_{\text{test}}$; fixed temporal scaling $\gamma$; specified ranges of scalar candidate values for $\alpha_f$ and $\alpha_i$. \\
\\
\textbf{1.} Construct a grid of candidate values for $\alpha_f$ and a grid of candidate values for $\alpha_i$. \\
\textbf{2.} Form all possible pairs $(\alpha_f, \alpha_i)$ by combining each value of $\alpha_f$ with each value of $\alpha_i$. \\
\\
\textbf{3.} For each pair $(\alpha_f, \alpha_i)$: \\
\hspace*{1em} \textbf{3.1} Shift LSTM biases: $\widetilde{b}_f = b_f + \alpha_f$, $\widetilde{b}_i = b_i + \alpha_i$. \\
\hspace*{1em} \textbf{3.2} Fine-tune the modified LSTM on $\mathcal{D}_{\text{train}}$, using $\mathcal{D}{\text{val}}$ to monitor training. \\
\hspace*{1em} \textbf{3.3} Generate predictions and evaluate performance on $\mathcal{D}_{\text{val}}$. \\
\\
\textbf{4.} Select the pair $(\alpha_f, \alpha_i)$ that gives the highest validation performance. \\
\\
\textbf{5.} Use the corresponding fine-tuned model to generate predictions and evaluate performance on $\mathcal{D}_{\text{test}}$. \\
\\
\textbf{Output:} Test performance and selected time-warping parameters $(\alpha_f, \alpha_i)$. \\
\\
\textbf{Note:} Fine-tuning requires a validation set for model selection. This procedure can also be repeated across multiple random initializations and data splits to assess robustness. \\
\hline
\end{tabular}
\end{table}

\section{Simulation Analysis with Newton's Law of Cooling}

\subsection{Problem Overview}

Newton's law of cooling is an ODE that models the rate at which the temperature of a liquid equilibrates to its environment. It is a foundational ODE that is used throughout introductory mathematics and physics courses. The canonical form of the ODE describes the change in temperature $T$ of some object over time $t$, when the liquid is exposed to the ambient temperature from the environment $T_a$, with
\begin{align}
\label{eq:newton}
    \frac{dT(t)}{dt} = -k(T(t)-T_a).
\end{align}
Equation~\eqref{eq:newton} relies on a number of simplifying assumptions. The constant $k$ is often referred to as the ``cooling constant". This parameter is used to approximate many specific physical properties of the system, such as the thermal mass of the object, the heat conductivity of the object and the surrounding ambient medium, the surface area of the object, etc. Other assumptions include that the modeled object has a uniform temperature, the ambient temperature of the environment is constant, the object itself does not change the temperature of the environment, etc. 

Modern physical modeling of heat transfer still uses the basic principles from Newton's Law of Cooling, either in constructing boundary conditions or in linear approximations of more complex systems. In many physical systems, the simplifying assumptions of Newton's Law are reasonable. Within wildfire science, the temperature of a small stick is often modeled as a single value, and the surrounding air is modeled as a constant background~\citep{Rothermel-1972-MMP}. In other areas of physical modeling, it is reasonable to consider the background temperature as fixed if you are modeling the change of temperature of an object that is embedded within a medium like the air, the soil, or the water.

The solution to Equation~\eqref{eq:newton}, using initial conditions $T(0)=T_0$, is
\begin{align}
\label{eq:newton_sol}
    T(t) =  (1-e^{-kt})T_a + e^{-kt} T_0.
\end{align}
Equation~\eqref{eq:newton_sol} describes the temperature of an object over time as it equilibrates with the ambient temperature of the environment. The rate of change is described by a smooth exponential. The constant $k$ controls the exponential rate of change $e^{-kt}$, and $1/k$ can be denoted the characteristic lag time.

\subsection{Demonstrating Time-Warping an LSTM with Newton's Law}

As an empirical demonstration of the time-warp scaling terms, we simulate solutions to Newton's Law with varying values of the cooling constant $k$. In each case, the initial temperature is 50$^\circ$, the ambient environment temperature is a constant 20$^\circ$, and the system is evolved forward for 50 time steps from the initial state. These values were chosen arbitrarily. A Jupyter Notebook for this demonstration is included in the software repository associated with the thesis where these parameters can be varied. Note that Newton's Law of Cooling is invariant to the temperature scale, so the simulation temperatures could be degrees Celsius or Kelvin or Fahrenheit. We construct a grid of 10 values for $k$ from a minimum of 0.05 to a maximum of 0.5, then solve for the deterministic trajectory for each value. Figure~\ref{fig:kcurves} shows the temperature curves for varying values of $k$. For each value of $k$, the system approaches the same equilibrium value of 20$^\circ$. The parameter $k$ controls the rate at which the system approaches the equilibrium; as $k$ increases the temperature reaches the equilibrium quicker.

\begin{figure}[htbp]
\centering
\includegraphics[width=12 cm]{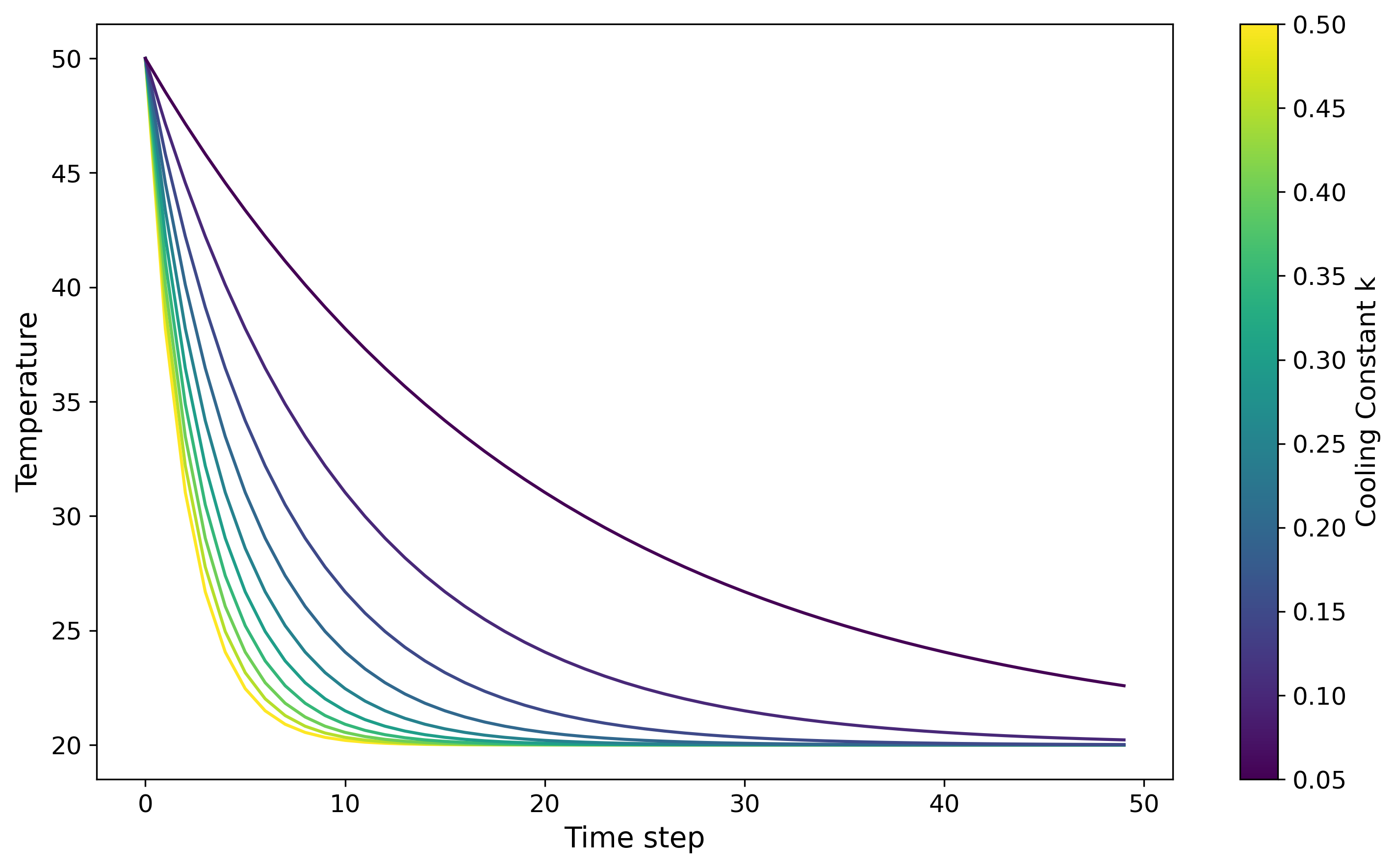}
\caption[Simulated temperature curves for varying cooling constants.]{Deterministic trajectories from Newton's Law of Cooling for varying cooling constants. Varying values of $k$ change the characteristic time scale. The initial temperature and the ambient temperature of the environment are kept constant in each case.}
\label{fig:kcurves} 
\end{figure} 

Using a simple RNN with a single recurrent unit, the system can be initialized for a starting value of $k$. Then, the time-warping can be made exact using the Equations~\eqref{eq:simple_timewarped_params}. We demonstrate that for the simple RNN, the approximation error is essentially computational rounding error. 

The LSTM case is nontrivial and involves approximation error. An LSTM with a single recurrent unit is constructed to approximate the system using a starting value of $k=0.5$ using~\eqref{eq:lstm_linear}. The output gate controls the error of approximation; the closer the output gate is to 1, the smaller the approximation error. 

We pick an error tolerance of $\epsilon=0.003$. By construction, $M=50$ bounds all the curves from above. In this case, we set the output gate bias to be 10. The sigmoid function evaluated at 10 is $\sigma(10)\approx 0.99995$.
\[
10 > \text{logit}(1-0.003/50)\approx 9.7
\]
We verify that using the output gate bias of $b_o=10$ constrains the approximation error for the LSTM compared to the true system to less than $\epsilon=0.003$.

Then, the LSTM is time-warped by modifying the bias terms with~\eqref{eq:lstm_linear_timewarp}. Section~\ref{sec:newton_results} will demonstrate that this results in a small approximation error. The approximation error can be made as small as desired by increasing the output gate activation as shown in Proposition~\ref{prop:lstm_linear}.

\section{Overview of Empirical Analysis with Fuel Moisture Classes}\label{sec:fmc_analysis}

The main empirical analysis in this thesis is based on testing transfer learning methods for generating FMC forecasts for different fuel classes. The data source used in this analysis is the original data from \citet{Carlson-2007-ANM}, which has been used to calibrate multiple FMC models, as described in Section~\ref{sec:fmc}. A copy of the original data for this study was provided via personal correspondence by Dr. Van der Kamp. In a personal correspondence, Dr. Carlson informed us that the original hardware copy of the data had been destroyed, and he knows of no other copy of the data. The paper reported 1,237 observations of FM1, 1,237 observations of FM10, 874 observations of FM100, and 877 observations of FM1000. In the data received from Dr. Van der Kamp, there were 1,233 observations of FM1, 1,232 observations of FM10, 871 observations of FM100, and 874 observations of FM1000. We are unable to reconcile the slight discrepancies in the observation counts, and no other backup of the original data exists to our knowledge. \citet{Carlson-2007-ANM} calibrated the Nelson model to the field observations and produced accuracy metrics, including the $R^2$ score and the bias, or mean error. They report the accuracy metrics for the overall sample and for the observations where the FMC is less than $30\%$. 

The pretrained RNN was trained on observations from FM10 sensors and with weather inputs from NWP forecasts. Figure~\ref{fig:rnn_fmc} shows the RNN model architecture~\citep{Hirschi-2026-RNN}. This model architecture will be fixed when comparing the transfer learning methods. From now on, we will refer to ``the RNN'' as this specific model architecture, and the parameters within it will be modified by the transfer learning methods. The RNN consists of a single LSTM layer with 64 units, two dense layers with 32 and 16 units, respectively, and an output layer with one dense unit. In total, this RNN architecture consists of 21,825 trainable parameters.

\begin{figure}[htbp]
\centering
\includegraphics[width=12 cm]{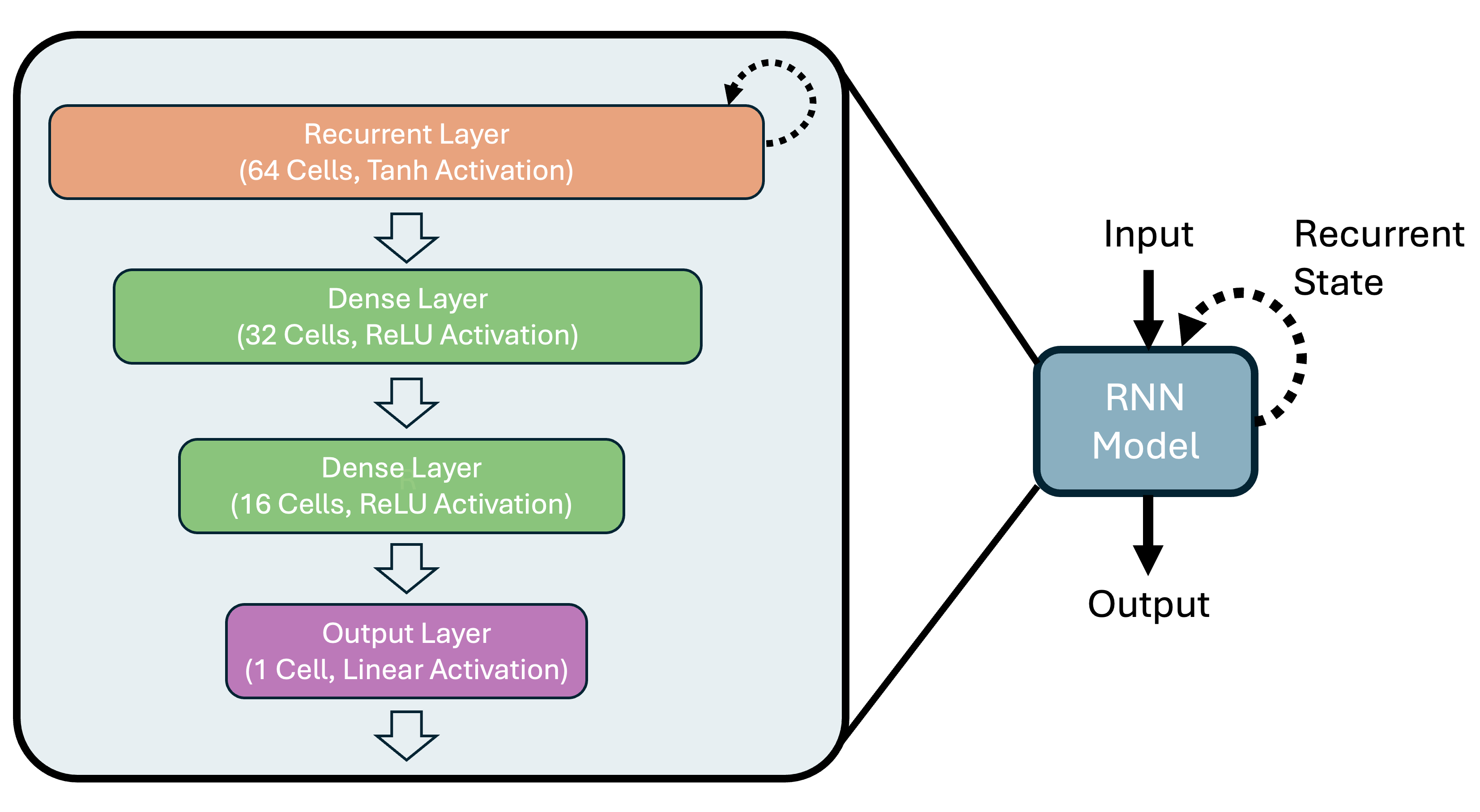}
\caption[RNN model architecture, pre-trained on FM10 from RAWS.]{RNN model architecture, pre-trained on FM10 from RAWS. An RNN with one LSTM layer and three subsequent dense layers.}
\label{fig:rnn_fmc} 
\end{figure} 

The RNN was developed and evaluated using a spatiotemporal cross-validation procedure. The training set, validation set, and test set were ordered in time so the RNN predictions are an extrapolation in time, with the test set in the future relative to the training set. The observation locations of observations were randomly assigned to the different data splits to account for spatial variability. A sample of locations were held out during training. Thus, the test set represents predictions that are at new locations and in the future. Training was controlled with an early-stopping procedure with patience of 5 to prevent overfitting. The hyperparameters for the RNN were tuned by training a set of candidate models in 2022 and choosing the model architecture that led to the most accurate forecasts in 2023. Then, the model was reinitialized and trained on data from 2023 and used to forecast all of 2024. The final accuracy metrics were presented on 2024. 

We use a ``final'' version of the RNN that is trained on all available FM10 observations from RAWS from 2023 through 2024. The training uses a validation set solely to control the early-stopping procedure, not to tune any other hyperparameters. The validation set consists of a random sample of $10\%$ of observation locations, with data from the last 48-hour period of 2024. Training samples were constructed from all other 48-hour sequences in the study area and time period. Batches of training data were constructed with random samples of the training sequences. In this case, we do not split off a test set of data, since the forecast accuracy was already estimated with an extensive spatiotemporal cross validation analysis. The sets of test data will come from the target learning task: data from the Oklahoma field study.

To account for the variability in the model training, we replicate the training process for the RNN 100 times. Each replication uses the same architecture and the same set of data: all available FM10 observations from RAWS within the Rocky Mountain region from 2023 through 2024. The replications utilize different random seeds that change the random split of stations into the training set versus the validation set and also randomize the order of samples used in training. The replications also change the random initial weights for the RNN. In summary, the replications are used to characterize the uncertainty of the training process due to initial weights, sample order, and the training set and validation set split. 

The Oklahoma field study is divided into a training set, a validation set, and a test set based on time. All observations are all from the same spatial location. Thus, there is no randomness in the cross validation split. Table~\ref{tab:all_vars} shows summary statistics for the hourly weather from the Slapout Mesonet weather station over the same time period as the observed FMC. A table with corresponding summary statistics from the source learning task, with data from all the Rocky Mountain region from 2023 through 2024, is found in Table 1 in the paper that is included as Appendix~\ref{app:paper}. Table~\ref{tab:fmc_summary} shows the summary statistics for the four fuel classes over the same time period. Notice that as the characteristic time-lag increases for the fuel classes, the typical FMC values decrease. Larger fuels tend to have lower average moisture content. 

\begin{table}[htbp]
\caption[Oklahoma Field Study Weather Data Summary.]{Oklahoma Field Study Weather Data Summary. Weather variables and geographic features are from the Slapout Mesonet station. The time period is from March 26, 1996 through the end of 1997.\label{tab:all_vars}}
\small
\begin{tabularx}{\textwidth}{l l p{4cm} p{1.5cm} p{2.5cm}}
\toprule
\textbf{Variable} & \textbf{Units} & \textbf{Description} & \textbf{Mean} & \textbf{(Low, High)} \\
\midrule
Drying Equilibrium & $\%$ & Derived from RH and temperature. & 18.02 & (1.29, 60.56) \\
\midrule
Wetting Equilibrium & $\%$ & Derived from RH and temperature. & 16.53 & (0.61, 56.10) \\
\midrule
Solar Radiation & $\text{W\,m}^{-2}$ & Downward shortwave radiative flux. & 208.41 & (0, 1177) \\
\midrule
Wind Speed & $\text{m\,s}^{-1}$ & Wind speed at 10m. & 2.52 & (0.40, 8.61) \\
\midrule
Rain & $\text{mm\,h}^{-1}$ & Rain accumulation over the hour. & 0.08 & (0.00, 42.17) \\
\midrule
Hour of Day & hours & Hour of day, UTC time. & 11.50 & (0, 23) \\
\midrule
Day of Year & days & Day of year, UTC time. & 201.70 & (1, 366) \\
\midrule
Elevation & meters & Height above sea level. & 774.00 & - \\
\midrule
Longitude & degree & X-coordinate & -100.26 & - \\
\midrule
Latitude & degree & Y-coordinate & 36.60 & - \\
\bottomrule
\end{tabularx}
\end{table}

\begin{table}[htbp]
\centering
\begin{tabular}{lccc}
\hline
Fuel Class & N. Observations & Mean & (Low, High) \\
\hline
FM1    & 1,233 & 15.33 & (0.0, 109.2) \\
FM10   & 1,232 & 15.02 & (1.6, 64.3) \\
FM100  & 871   & 13.44 & (5.4, 35.7) \\
FM1000 & 874 & 11.37 & (4.7, 27.9) \\
\hline
\end{tabular}
\caption[Oklahoma Field Study FMC Data Summary]{Oklahoma Field Study FMC Data Summary. Observations are made with gravimetric measurements on arrays of wooden dowels. The time period is from March 26, 1996, through the end of 1997.}
\label{tab:fmc_summary}
\end{table}

We split the Oklahoma field data into a training set, a validation set, and a test set. The earliest observation was on March 26, 1996, and the study conducted observations through 1997. The data consist of a single time series for each fuel class. All observations are from a single spatial location near Slapout, Oklahoma. The data splitting is therefore only based on time. The training set is defined to be one full year of data starting from the earliest observation. We use one full year so that, in principle, any models trained on this data could learn patterns of seasonal variation. Then, the validation set and test set are constructed by splitting the remaining times evenly. Final accuracy metrics will be calculated only on the test set so that the models are evaluated on their ability to forecast in time. The test set consists of about four months of data, from August of 1997 through December of 1997. Chapter~\ref{chapter:conclusions} will discuss the limitations of this data set and why it suggests that transfer learning is useful. 

\begin{table}[htbp]
\centering
\begin{tabular}{lccc}
\hline
 & Training Set & Validation Set & Test Set \\
\hline
Start Date (UTC) & 1996-03-26, 23:00 & 1997-03-27, 00:00 & 1997-08-13, 12:00 \\
End Date (UTC)   & 1997-03-26, 23:00 & 1997-08-13, 11:00 & 1997-12-30, 22:00 \\
\hline
\end{tabular}
\caption[Temporal data splits for Oklahoma field data.]{Temporal definition of the training, validation, and test splits for the Oklahoma field data. All dates are reported in UTC.}
\label{tab:data_split_summary}
\end{table}

\begin{table}[htbp]
\centering
\begin{tabular}{lccc}
\hline
N. Observations & Training Set & Validation Set & Test Set \\
\hline
Weather & 8,761 & 3,348 & 3,347 \\
FM1     & 704   & 258   & 271   \\
FM10    & 704   & 258   & 270   \\
FM100   & 481   & 184   & 206   \\
FM1000  & 482   & 183   & 209   \\
\hline
\end{tabular}
\caption[Observation counts by temporal split for Oklahoma field data.]{Number of observations in the training, validation, and test splits for each variable in the Oklahoma field data. Note that not all analyses utilize the validation set.}
\label{tab:data_split_counts}
\end{table}

Several analyses are used to asses the time-warping transfer learning method. First, we evaluate the accuracy of the predictions from the pretrained RNN compared to the observed FM10 from the Oklahoma field study. Then, we evaluate the empirical distributions of the forget gate and input gate bias parameters for the pretrained model across 100 replications. This analysis is to provide evidence that the RNN is learning a stable set of parameters, and that the time-warping method is not being done on an arbitrary set of weights. Lastly, we evaluate several transfer learning methods to with the target learning tasks of predicting FM1, FM100, and FM1000. These are the fuel classes where no sensor data exists, and there is only sparse experimental data available. 

\citet{Carlson-2007-ANM} reported accuracy metrics for the Nelson model for each fuel class in the study period. The accuracy metrics for the Nelson model do not include uncertainty estimates. We will compare the RNN accuracy metrics to the Nelson model. However, the entire study period from 1996 through 1997 was used to estimate the accuracy of the Nelson model, with no data splitting into training, validation, and testing periods. Direct comparison to these metrics is possible for zero-shot transfer learning for FM10, which will be described in the next section. The other analyses utilize fitting or fine-tuning of some sort, and will split the data as described in Table~\ref{tab:data_split_summary}. We will still compare the accuracy to the Nelson model with the qualification that this isn't an exact comparison. Two static ML baselines, a simple linear regression and a static XGBoost, are fit to the data and results are presented in Appendix~\ref{app:extra}. These models are used to show that the more complex RNNs built with transfer learning are improving upon a simple baseline. 

\section{Zero-Shot Transfer Learning for FM10}

Before evaluating any transfer learning methods that change the target learning task, we evaluate whether the RNN that was pretrained on fuel sensor observations and inputs from NWP forecasts is still accurate when used with data from the Oklahoma field study. We generate predictions of FM10 using the ground-based weather observations and compare them to the observed FM10 values. We wish to verify that the error is within the estimated forecast error bounds reported in the published study~\citep{Hirschi-2026-RNN}, and then we compare the resulting accuracy to the state-of-the-art Nelson model that was calibrated on the Oklahoma field study. 

There are a couple of things to consider when deploying the pretrained RNN in this case. The FM10 observations were collected with different measuring devices, and the weather inputs are from different sources. The measurements from the fuel sensors on RAWS are intended to produce the same results as the gravimetric measurement technique that was used in \citet{Carlson-2007-ANM}. Both measurement methods use ponderosa pine dowels elevated 30~cm off the ground~\citep{Carlson-2007-ANM, Campbell-2015-CFM}. However, the fuel sensors indirectly measure FMC through physical proxies, and there will be differences relative to the direct gravimetric measurements \citep{Campbell-2015-CFM, FTS-2016-FSS}. Additionally, the RNN was trained using weather inputs from numerical weather models, and in this case we use inputs from ground-based weather observations from Slapout, Oklahoma, along with the longitude, latitude, and elevation reported for that station by the Oklahoma Mesowest network~\citep{OKMeso-2026-MSI}. 

This is a type of transfer learning problem. The feature space is the same, but the different measurement devices for the weather correspond to different marginals over that space. The response variable space is also the same in this case, but the distributions of the response variable conditional on the weather could be different due to measurement noise or bias. This is sometimes referred to as transductive learning~\citep{Pan-2010-STL}. We will analyze the accuracy of the zero-shot transfer learning using the data from Slapout, Oklahoma. As will be shown in Section~\ref{sec:fmc_results}, the error of the RNN forecasts compared to the gravimetric measurements is within the bounds estimated in the paper when you account for the FMC sensors' response to heavy rain events. Zero-shot transfer learning for FM10 is used to establish that the forecasts from the RNN can be compared to reality in the field data, and we will ignore the transfer learning effects from the source of the weather and the measurement device when analyzing the time-warp method.

We will generate predictions of FM10 across the entire study period from March of 1996 through December of 1997. The data is not split into training and test sets, since no training is being done. The entire period acts as the test set for the transfer learning problem in this case. The predictions are generated for each of the 100 RNN replications to estimate the uncertainty in the accuracy metrics. 

\section{Stability of the Learned Gate Bias Structure for the Pretrained FM10 Model}

For an LSTM layer with 64 units, both the forget and input gate biases are random vectors of length 64. Further, they are permutation invariant. The particular index order of an LSTM unit is arbitrary. The optimization process of the RNN parameters is random. The uncertainty in the final bias values comes from randomness in the initial weights, randomness in the order of training samples, and in the case of the spatiotemporal training used with the FMC model, randomness in the set of physical locations used for the training versus validation sets. Before we test the method of time-warping by directly modifying the forget and input gate bias parameters, we provide evidence that the biases are stable across the replications. In principle, with an arbitrary LSTM and an arbitrary set of training samples, there is no guarantee that the forget and input gate biases will converge to a consistent distribution or a shared learned structure.

Using the 100 replications of the source RNN, we analyze whether the gate-bias quantities targeted by the time-warp ($b_f$ and $b_i$) have a stable distribution under replication. The purpose of this analysis is to increase confidence that the time-warp produces a consistent shift in a learned timescale  across replications.

Summary statistics and visual examination of the empirical distribution of bias parameters are used to conclude that there is adequate evidence that the forget and input gate bias parameters have a consistent structure across replications. This analysis is meant to provide evidence that the bias parameters are consistent across replications, and that reasonable conclusions can be drawn about what effect the time-warp method has on the resulting network dynamics. A more complete treatment of this topic would relate to convergence of 1-dimensional distributions, which will not be dealt with in this thesis. 

\section{Comparison of Transfer Learning Methods for Predicting FMC}

The main application of this research is to produce accurate predictions of FMC for fuel classes with sparsely observed data. This thesis proposes a new method of transfer learning through time-warping the forget and input gate bias parameters in a pretrained LSTM. This new method of transfer learning will be compared to several baseline methods of transfer learning. 

The domain is the same for the FMC learning tasks. The space of features $\mathcal X$ is the 10-dimensional vector space defined by the set of predictors. The set of 10 predictors is listed in Table~\ref{tab:all_vars}, with summary statistics that apply to the Oklahoma field study. A table with summary statistics from the source learning task, with data from all the Rocky Mountain region from 2023 through 2024, is found in Table 1 in the paper that is included as Appendix~\ref{app:paper}. The marginal distribution over $\mathcal X$ reflects the combinations of the predictors that tend to occur in the study domain. The source task is predicting FM10. The target tasks are predicting FM1, FM100, and FM1000. In each case, the response variable space $\mathcal Y$ is taken to be $\mathbb R$. FMC is physically nonnegative and bounded above by properties of the woody material. But the RNN maps onto the real line, and physical constraints are enforced implicitly through the training data.

An important difference between the source learning task and the target tasks is the sparsity of the observed FMC data. The source task utilized hourly observations of FM10 from a large network of physical locations. The RNN was built as a sequence-to-sequence model, where sequences of weather inputs are mapped to sequences of FM10 predictions. A sequence length of 48-hours was used during training, as that is the maximum possible forecast time for the HRRR weather model and limits the forecast window for a real-time operation. For the target learning task, the FMC was measured at most twice daily. The weather was still recorded hourly, so the input structure is the same. We keep the model architecture the same, where the RNN produces sequences of predictions with hourly resolution. The accuracy when comparing to the Oklahoma field observations is computed using the RNN predictions at the corresponding times.

Two transfer learning methods will be evaluated that utilize the time-warping method, with and without fine-tuning. These will be compared to established transfer learning techniques. Table~\ref{tab:tl_list} shows a summary of the transfer learning methods that will be used for this analysis. Section~\ref{sec:tl} provides a literature review on these baseline methods. The baseline methods that utilize fine-tuning use the same training procedure, with different initialization of parameters and different sets of parameters that are allowed to be updated during training. The accuracy metrics will be the $R^2$, the bias, and the RMSE. The $R^2$ coefficient of determination is calculated as one minus the ratio of the mean squared error to the variance of the observed data. A value of 1 indicates perfect agreement between the observed data and the model. The bias and the RMSE are interpretable in units of percent FMC. The bias is defined as the mean of the observed values minus the modeled values, and is unrelated to the notion of ``bias'' used as an RNN parameter. Positive bias indicates underprediction, and negative bias indicates overprediction. RMSE measures the typical magnitude of prediction errors by taking the square root of the mean squared difference between observed and modeled values. The RMSE is the preferred metric in much of the ML literature. The RMSE is calculated by averaging the squared errors over all times in the testing period, then applying the square root after. In~\cite{Hirschi-2026-RNN}, the estimated forecast accuracy metrics were an RMSE of 3.02~$\pm$~1.12 and a bias of -0.19~$\pm$~0.33. These were calculated on a large dataset across the Rocky Mountain region for all of 2024. 

\begin{table}[htbp]
\centering
\caption[Transfer Learning Methods Comparison for FMC Prediction]{Transfer learning methods comparison for FMC prediction.}
\label{tab:tl_list}
\begin{tabularx}{\textwidth}{|>{\RaggedRight\arraybackslash}p{3.5cm}|>{\RaggedRight\arraybackslash}X|>{\RaggedRight\arraybackslash}X|}
\hline
\textbf{Method Name} & \textbf{Details} & \textbf{Literature Status} \\
\hline
No Transfer
& Random initialization and direct training on the target learning task.
& Standard supervised training baseline. \\
\hline
Full Fine-Tuning
& Model pretrained on a source task and all parameters updated on the target task.
& Standard transfer learning baseline. \\
\hline
Freeze Recurrent Layer
& Recurrent weights fixed while dense/output layers are retrained on the target task.
& Common transfer learning strategy, widely used for CNNs. \\
\hline
Freeze Dense/Output Layers
& Dense/output layers fixed while input and recurrent layers are retrained.
& Reverse-freezing baseline; rarely explored in literature. \\
\hline
Time-Warping
& Parameters for shifting the forget and input gate biases are fit, without additional training.
& Novel to this thesis. \\
\hline
Time-Warping + Fine-Tuning
& Time-warp initialization followed by full fine-tuning on the target task.
& Novel to this thesis. \\
\hline
\end{tabularx}
\end{table}

The data structure requires some modifications to the training process for the RNN when fine-tuning. The loss can only be calculated at a sparse set of times. During training (fine-tuning), we keep the sequence length of 48-hours to match the source training methodology. Sequences of weather inputs are mapped to 48-hour sequences of FMC outputs. Then, the loss is calculated by comparing to the FMC observations from the Oklahoma field study which happen to lie in the 48-hour prediction window. Since the FMC of the different fuel classes were observed at most twice daily, this corresponds to about 4 response values for the prediction window. Those 4 response values are used to calculate the loss, and the rest of the RNN predictions that don't line up in time with an FMC observation are not used to calculate loss.

In each transfer learning method that utilizes fine-tuning, there is the risk of overfitting. Observations for the target learning tasks all come from a single location in Oklahoma and from a relatively short time period. As mentioned before, this is the only publicly-available data set we are aware of that observes the different standard fuel classes under controlled experimental conditions and with temporal resolution less than one day. When we conduct fine-tuning in the scenarios described above, a validation set is used with an early-stopping procedure to control for overfitting. This will control for the RNN overfitting to a specific sequence in time. However, there is no way to verify whether the fine-tuned models generalize to other spatial locations. In Section~\ref{sec:future}, we will discuss planned future research with the FEMS field observations that will address this issue. 

We also evaluated two static, or non-recurrent models, using the same set of predictors as the RNN. The static models include a simple linear regression and an implementation of an XGBoost model with the same architecture used for the large forecast analysis in the published paper~\citep{Hirschi-2026-RNN}. These models are fit to observations from the training period and validation period combined for each fuel class, since no hyperparameter tuning is being done. The final accuracy metrics are calculated against observed FMC values in the test set. These models are less accurate than the other methods, and the results are included in Appendix~\ref{app:extra}. The static ML models are used as reference baselines to verify that the methods with RNNs and transfer learning provide meaningful improvements.

For the analyses presented in this thesis, the primary focus will be on the Time-Warping method. In that scenario, only two parameters are fit to data: the two parameters for time-warping the gate biases. These two parameters shift 128 out of over 21,000 trainable parameters in the RNN. The risk of overfitting is less than for the other methods that utilize fine-tuning. Further, the Time-Warping method provides a possible way to generate FMC predictions for sticks of arbitrary sizes that are not one of the standard fuel classes that are routinely observed. 

\subsection{No Transfer Learning Baseline}

As a baseline of comparison, the ``No Transfer'' method trains the RNN directly on the target tasks with no pretraining. The parameters of the RNN are randomly initialized, then the model is trained using only response value observations from the target task and predictors that correspond to those FMC observations in space and time. The RNN is fit to the training set with a validation set controlling the early-stopping procedure. The trained model is used to predict FMC in the test period and the predictions are evaluated against the observed values.

We hypothesize that the No Transfer method will perform relatively poorly due to the sparsity of data for the other fuel classes. However, even if that were not the case, we would have less confidence in a models ability to produce reliable FMC forecasts across the country, since the model cannot learn any spatial relationships from the data. To make this method feasible for comparison, we reduce the set of predictors from 10 to 7, excluding the three static predictors of latitude, longitude, and elevation. In practice, including features that have zero variability, like geographic features from a single location, leads to slow training and poor convergence for models trained with gradient descent. 

\subsection{Full Fine-Tuning}

The ``Full Fine-tuning'' method will fine-tune the RNN and allow all model parameters to be updated. The method will start with the RNN pretrained to predict FM10. Then, all weights in the RNN will be updated using samples from the training set, with the validation set controlling early-stopping. Then, the trained model will be used to predict FMC in the test period and evaluated against the observed values.

The Full Fine-Tuning method is the most simple type of transfer learning. We hypothesize it will perform worse than the more complex methods. However, this needs to be tested empirically. Freezing layers can lead to sub-optimal fine-tuning results due to the issue of fragile co-adaptation, since parameters in a neural network often have complex dependencies across the network architecture. 

\subsection{Fine-Tuning with Frozen Recurrent Layer}

Two scenarios will proceed just like the Full Fine-tuning method, with the difference that certain neural network layers will be frozen and not allowed to update during fine-tuning. The ``Freeze Recurrent Layer'' method will fine-tune the RNN while fixing all parameters in the recurrent layer, allowing the parameters in the dense and output layers to be updated. 

The Freeze Recurrent Layer method is used since it is standard in the transfer learning literature with CNNs and image classification, e.g.~\citep{Yosinski-HTF-2014}. The assumption is that the early layers of a network extract a set of latent features, and the downstream layers match the latent features with a particular class label. This has shown to be useful in the context of CNNs, but it is less explored with recurrent layers. In the context of RNNs, the recurrent layer would act as a generic temporal feature extractor. Then the subsequent dense layers learn how to combine different temporal features into a prediction. We hypothesize this method will perform poorly relative to the other methods. The recurrent layer learns a set of parameters that is used to evolve a dynamical system forward in time. If the characteristic time scale of the system changes, we expect that the parameters of the recurrent layer would need to be updated to maintain accuracy. 

\subsection{Fine-Tuning with Frozen Dense Layers}

The ``Freeze Dense/Output Layers'' method will fine tune the RNN and fix the parameters of the dense and output layers, allowing the parameters of the recurrent layer to be updated. 

The Freeze Dense/Output Layers method has not been explored extensively to our knowledge. \citet{Ma-2021-THD} evaluated freezing the input layers and output layers together. We hypothesize that freezing the output layers alone could be preferable. This method tests the hypothesis that the recurrent layer extracts a set of standardized signals with a characteristic time scale, and changing the time scale of the task requires modifying the dynamics within the recurrent layer. Fixing the input layer weights seems sub-optimal in the context of sticks of different sizes. Smaller sticks have less surface area than larger sticks, and would be expected to interact differently with environmental conditions like wind and rain that act on the surface of the fuel. Further, this method assumes that the mapping learned by the dense layers is invariant between tasks. While this method is used as a baseline of comparison for the time-warp technique, it shares a motivation with the time-warping technique. Fine-tuning the parameters of the recurrent layers is meant to accomplish the same thing as the time-warping, but in a purely data-driven way rather than through direct modification of the input and forget gates that is motivated by theoretical considerations.

\subsection{The Time-Warping Method}

The ``Time-Warping'' method modifies the bias parameters in the forget and input gates. We fit time-warping constants for each fuel class to data from the training set using Algorithm 1 described in detail in Table~\ref{tab:twarp_algorithm}. For each target learning task, we construct a grid of candidate parameters for shifting the forget and input gate biases. The candidate time-warping parameters shift the bias parameters for each fuel class separately. We select the time-warping parameters that lead to the best resulting model accuracy when evaluated against observations from the target learning task in the training period. The Time-Warping method splits the Oklahoma field data into a training period and a testing period. We do not use a validation set for this method, so the training and validation periods shown in Table~\ref{tab:data_split_summary} are combined into a single training period.

For each fuel class, we construct a grid of 25 possible values for the forget and input gate time-warp parameters. We consider a grid of 25 evenly spaced values ranging from -5 to 5 for both the forget and input gate time-warping parameters.  All possible combinations of the two time-warping parameters are considered, resulting in 625 possible parameter configurations (equal to 25 squared). The same grid of parameters is searched for FM1, FM100, and FM1000 separately. 

This will be applied to the entire recurrent layer to uniformly shift the bias terms for the forget gates and input gates. The optimal values of $\alpha_f$ and $\alpha_i$ are chosen via grid search. The set of time-warping parameters that leads to the lowest root mean squared error (RMSE) for the training set will be selected. The final estimate for model accuracy for the method will be calculated by generating forecasts for the target task in the test period and comparing them to the observed values. The time-warp method does not utilize fine-tuning. So, two parameters are fit to data that are used to shift 128 out of over 21,000 trainable parameters in the pretrained RNN. 

As will be discussed in Section~\ref{sec:future}, this method could be used to provide zero-shot transfer learning predictions of fuel classes that are never observed. We wish to be able to generate FMC forecasts for a hypothetical 50h fuel, a stick with a size between a standard 10h fuel and a 100h fuel, or sticks of any size. That will necessitate using zero-shot methods, since to our knowledge, no data exists for these fuel sizes. In this thesis, we optimize the temporal scaling parameters with training and validation sets of data. But in future research, we hope to generate zero-shot predictions for sticks of arbitrary sizes.

\subsection{Time-Warping plus Fine-Tuning}

The Time-Warping plus Fine-Tuning Method will modify the bias terms for the forget and input gates, then fine-tune and allow all model variables to be updated. The optimization process uses Algorithm 2 as described in Table~\ref{tab:twarp_finetune_algorithm}.

This method uses a validation set to control early-stopping. Training samples are constructed for the fine tuning step. The hyperparameters related to training will be reused from the source RNN, which are described in detail in the paper which is attached in Appendix~\ref{app:paper}. Training samples will be 48-hour sequences of the set of 10 predictors. The main difference from the source learning task is that for each fuel class, there is only a univariate time series of weather and FMC observations. We build samples with a rolling window approach, using a 48-hour window and a ``stride-length'' of 12. This means the first sample will start at the first available observation of FMC in the training period. The weather inputs will be the next 48 hours from that date. Since the FMC is measured at most twice daily, the associated response vector for the first sample will be at most 4 observations, with mostly missing values within the 48-hour sequence. The loss will be calculated across the 4 available observations. The next sample will start 12 hours from the start of the first sample. The weather data for the next 48 consecutive hours forms the input. The associated response vector will be any available FMC observations within the shifted 48-hour window. This rolling window process helps create more samples, and a nonzero stride length helps reduce the similarity between samples. Additionally, the rolling window results in the sparse FMC values appearing at different index positions along the 48-hour sequence, to avoid the situation where it overfits to FMC being at a particular set of index positions. The rolling window approach resulted in 727 samples out of 8,761 possible hours from the training period. If non-overlapping samples had been used without a rolling window, there would have been at most 183 samples (8,761 divided by 48). The resulting training set of weather variables is a tensor with shape: \texttt{(727, 48, 10)}. This same methodology will be used for all baseline transfer learning methods that use fine-tuning. 

Intuitively, the Time-Warping plus Fine-Tuning method should perform better since it combines pretraining, time-warping, and fine-tuning to leverage all of the available knowledge. There is the risk that fine-tuning will overfit to the specific location and time of the Oklahoma field study, but there is also the risk that fine-tuning could degrade the performance relative to the Time-Warping method without fine-tuning. Due to the issue of fragile co-adaptation of learned weights~\citep{Yosinski-HTF-2014}, this method cannot be assumed to perform the best and must be tested empirically. It is possible that time-warping could improve prediction accuracy, but also shift the parameters into a space that leads to poorer convergence in the gradient descent during the fine-tuning steps.

\section{The Effect of Time-Warping on the Learned Dynamics of FMC Models}

The goal of the time-warping method is to change the learned dynamics of the RNN.To investigate whether modifying the forget and input gate biases actually changed the characteristic time scales of the RNN, we analyze the bias shift parameters across 100 replications for the Time-Warping method, with no fine-tuning. Then, we summarize the numerical approximations of the derivatives of the predictions. Lastly, we analyze the ACF and PACF of the time-warped RNN predictions to empirically analyze the temporal dependencies in the resulting time series. 

Recall the simple linear ODE for FMC shown in Equation~\eqref{eq:tl_fmc}. Table~\ref{tab:classes} shows how time-warping would modify the parameters of the linear ODE. If the equation started as a model of FM10, then was time-warped to be an model of FM1, the time-warping would decrease the dependence on the previous state of the system and increase the dependence on the current weather input. If an LSTM had a set of weights that approximated this system as described in Section~\ref{sec:lwarp}, the time-warping would decrease the forget gate bias and increase the input gate bias. Recall the slightly counterintuitive interpretation of the forget gate. The forget gate activation is multiplied with the previous cell state. Activations close to 1 correspond to retaining most of the previous cell state $c_{t-1}$, and activations close to zero correspond to the next cell state depending mostly on the current candidate memory. So, negatively shifting the forget gate bias makes the system ``forget more''. Correspondingly, for time-warping to an FM1 model, the input gate bias would be increased. Time-warping to FM100 and FM1000 systems would have the opposite effect on the gate biases. The forget gate bias would be increased and the input gate bias would be decreased. 

The theoretical time-warping of an LSTM described in Section~\ref{sec:lwarp} describes a situation where the exact dynamics are a linear time-lag system, and the LSTM weights are specifically chosen to approximate the system. For an arbitrary LSTM trained on real data, there is no guarantee that the empirical time-warping parameters will correspond neatly to the linear approximation. Nevertheless, we will analyze the time-warping parameters that were fit to data for the FMC models to see if they directionally match our expectations. Then, the ACF and PACF of the resulting predictions for the different fuel classes will be compared to analyze whether there is evidence that the temporal dependencies have shifted.

    \chapter{\uppercase{Results and Discussions}}\label{chapter:04}

\section{Newton's Law of Cooling}\label{sec:newton_results}

We demonstrate the theoretical results for time-warping an LSTM using a case with idealized dynamics. We simulated solution curves from Newton's Law of Cooling with different cooling constants $k$, which control the characteristic time scale of the system. We initialize a simple RNN and then an LSTM to approximate the solution for one value of $k$, then time-warp the weights using Equations~\eqref{eq:lstm_linear_timewarp}. The simulated temperatures ranged from $20$ to $50$ units. 

The maximum of the absolute value of the difference between the simulated temperature curves and the simple RNN after time-warping was approximately $3 \times 10^{-5}$. The error between the true temperature curves and the simple RNN predictions is driven by computational rounding effects. These computational errors include the evaluation of the exponential $e^{-k}$. Then, the iterative process of generating the RNN predictions compounds these computational errors. Thus, the dynamic system can be exactly represented with the simple RNN, up to computational rounding error. 

In addition to the computational rounding errors that exist for the simple RNN, the errors for the LSTM in this case are also driven by the sigmoid function in the output gate, which approximates one. The maximum absolute error for the LSTM with linear activation was approximately $0.00221$. By construction, all curves were bounded above by $M=50$. Using the bounds derived in Proposition 4, an output gate bias of $b_o=10$ should have constrained all modeling errors to be less than $\epsilon=0.003$. The simulation study demonstrates that the choice of output gate bias successfully constrained the numerical errors to within the expected error tolerance. Therefore, the LSTM can approximate the dynamic system to any desired accuracy. 

Figure~\ref{fig:kcurves} shows the true temperature curves for varying values of the cooling constant $k$. The temperature curves that result from the simple RNN predictions and the LSTM predictions are all visually indistinguishable from this plot. 

\section{Transfer Learning for Fuel Classes}\label{sec:fmc_results}

The results from zero-shot transfer learning for FM10 and the stability of the learned gate bias structure are used to establish that transfer learning to other fuel classes is justified. The transfer learning results will be organized by fuel class for direct comparison. The reported accuracy metrics for the Nelson model will be shown for the different fuel classes, but they are only directly comparable to the FM10 zero-shot predictions, since in this case predictions are generated for the entire study period of the Oklahoma field study~\citep{Carlson-2007-ANM}.

Then we analyze the stability of the learned gate bias structure. The Time-Warping method will be evaluated with 100 replications of the source RNN. The learned gate bias parameters will vary across replications. But if they vary wildly without signs of stable convergence, that weakens the theoretical motivation for the time-warping method. For the pretrained FM10 RNN, we show that the bias structure is relatively stable across replications. 

The accuracy metrics for predicting FM1, FM100, and FM1000 will be presented for the different transfer learning techniques. The accuracy metrics are presented only on the test set for these fuel classes, which is from August of 1997 through December of 1997. We will compare the resulting accuracy metrics to those for the Nelson model for the other fuel classes, but caution needs to be taken when interpreting the results, as the RNN predictions are only on a subset of available times. The main focus for the transfer learning analysis is whether the Time-Warping method can provide accuracy results consistent with established techniques while retraining only a small fraction of the source model weights. 

Lastly, we analyze the temporal structure of the predictions from the Time-Warping method. While the accuracy metrics at predicting FMC are of primary interest, we wish to provide statistical evidence that the time-warping method changes the temporal structure of the resulting model predictions. The best-fit time-warping parameters are analyzed and compared to theoretical expectations. We present results from the empirical ACF and PACF of the predictions from the Time-Warping method.

\subsection{Zero-Shot Transfer Transfer Learning for FM10 Results}

Using ground-based weather observations of the variables listed in Table~\ref{tab:all_vars}, along with the geographic variables for the Slapout station~\citep{OKMeso-2026-MSI}, the RNN was used to predict hourly FM10 from March 26, 1996, at 23:00 through December 30, 1997, 22:00, all in Central Time. These times were converted to UTC to align with the data used to train the RNN from RAWS and the HRRR. This resulted in 15,466 individual FMC predictions. The Oklahoma field study collected 1,232 observations of FM10 over this time period. This process was repeated for each of the 100 RNN replications to account for the uncertainty in the final trained model.

Table~\ref{tab:model_comparison} shows the accuracy metrics for the zero-shot predictions and the Nelson model for all observations and for FM10 values filtered to less than $30\%$. \citet{Carlson-2007-ANM} do not report RMSE values, and it is not possible to calculate those values without access to the original predictions. Additionally, the accuracy metrics for the Nelson model do not include uncertainty estimates. We include the RMSE value for the RNN predictions, as it is the preferred metric of many ML researchers. The metrics for the RNN represent the mean over 100 replications; the plus-or-minus bounds are one standard deviation of the metric.

\begin{table}[htbp]
\caption[Accuracy metrics for FM10 model compared to field observations in Oklahoma.]{Accuracy metrics for FM10 Zero-Shot RNN Predictions and Nelson model~\citep[Table 4]{Carlson-2007-ANM}. Comparisons of the models are presented for all FM10 values ($n=1,232$) and for FM10 $\leq 30\%$ ($n=1,134$). RNN metrics are reported as mean $\pm$ standard deviation across 100 replications.}
\label{tab:model_comparison}
\centering
\begin{tabular}{lcc|cc}
\multicolumn{1}{c}{} 
& \multicolumn{2}{c}{All observed FM10} 
& \multicolumn{2}{c}{Observed FM10 $\leq$ 30\%} \\
\hline
Metric 
& Nelson (hourly) & RNN 
& Nelson (hourly) & RNN \\
\hline
$R^2$ & $0.79$ & $0.56\pm0.02$ & $0.61$ & $\;\;\,0.65\pm0.06$ \\
Bias (\%)  & $-0.9$ & $0.91\pm0.52$ & $-1.1$ & $-0.74\pm0.53$ \\
RMSE (\%) & - & $6.79\pm0.16$ & - & $\;\;\,3.21\pm0.25$ \\
\hline
\end{tabular}
\end{table}

Figure~\ref{fig:fm10_scatters} shows a scatter plot of the observed FM10 versus the predicted FM10 from the zero-shot transfer RNN. The RNN predictions are from the model with the median RMSE in the test period across 100 replications. A specific model is used instead of an aggregation over the 100 replications to prevent distorting or smoothing out any patterns that could be expected from a single model training. The median RMSE was $3.16\%$ for FM10 values less than or equal to $30\%$.

\begin{figure}[htbp]
\centering
\begin{minipage}[c]{0.8\textwidth}
  \centering
  \includegraphics[width=\textwidth]{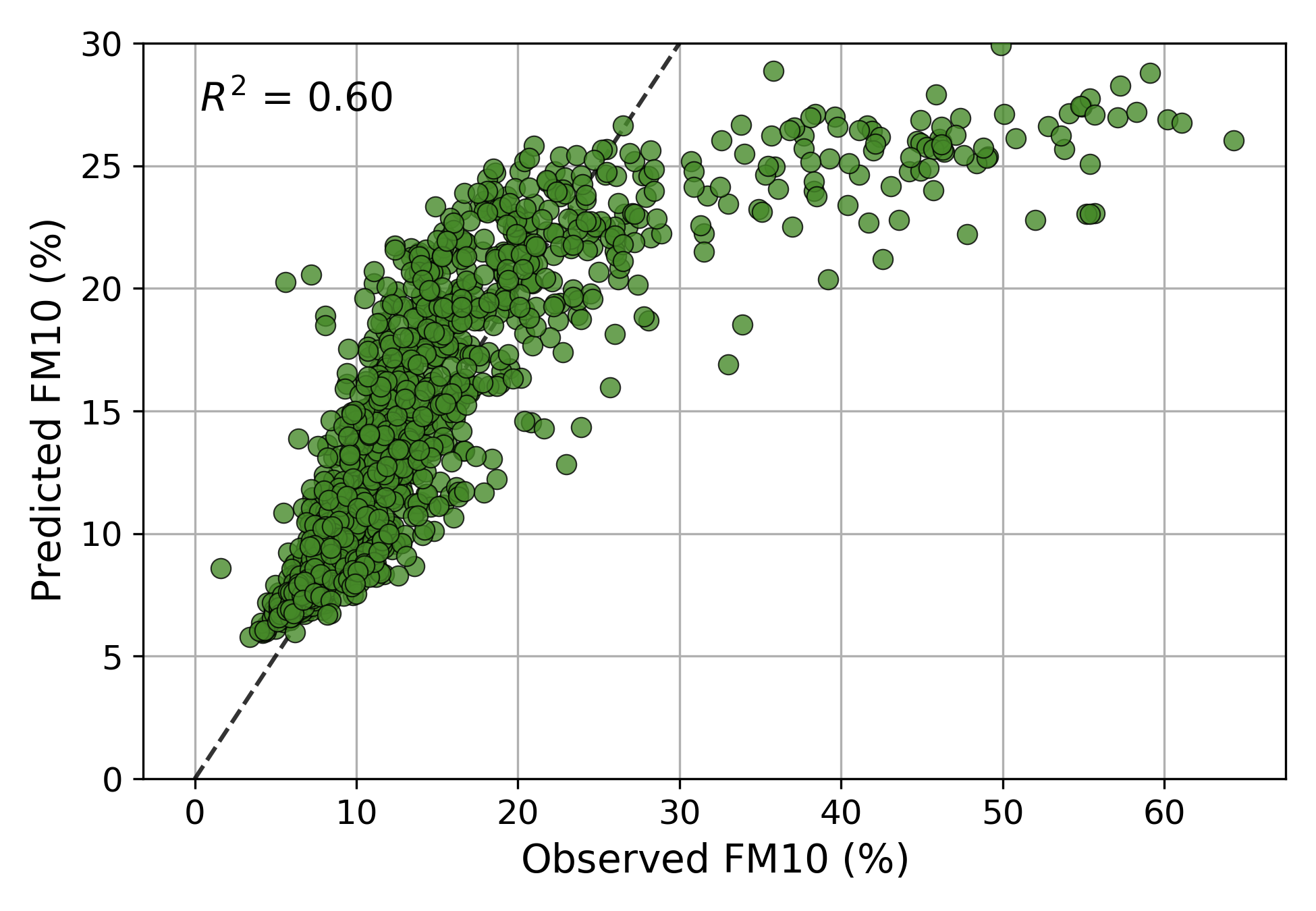}
\end{minipage}

\vspace{0.5cm}

\begin{minipage}[c]{0.8\textwidth}
  \centering
  \includegraphics[width=\textwidth]{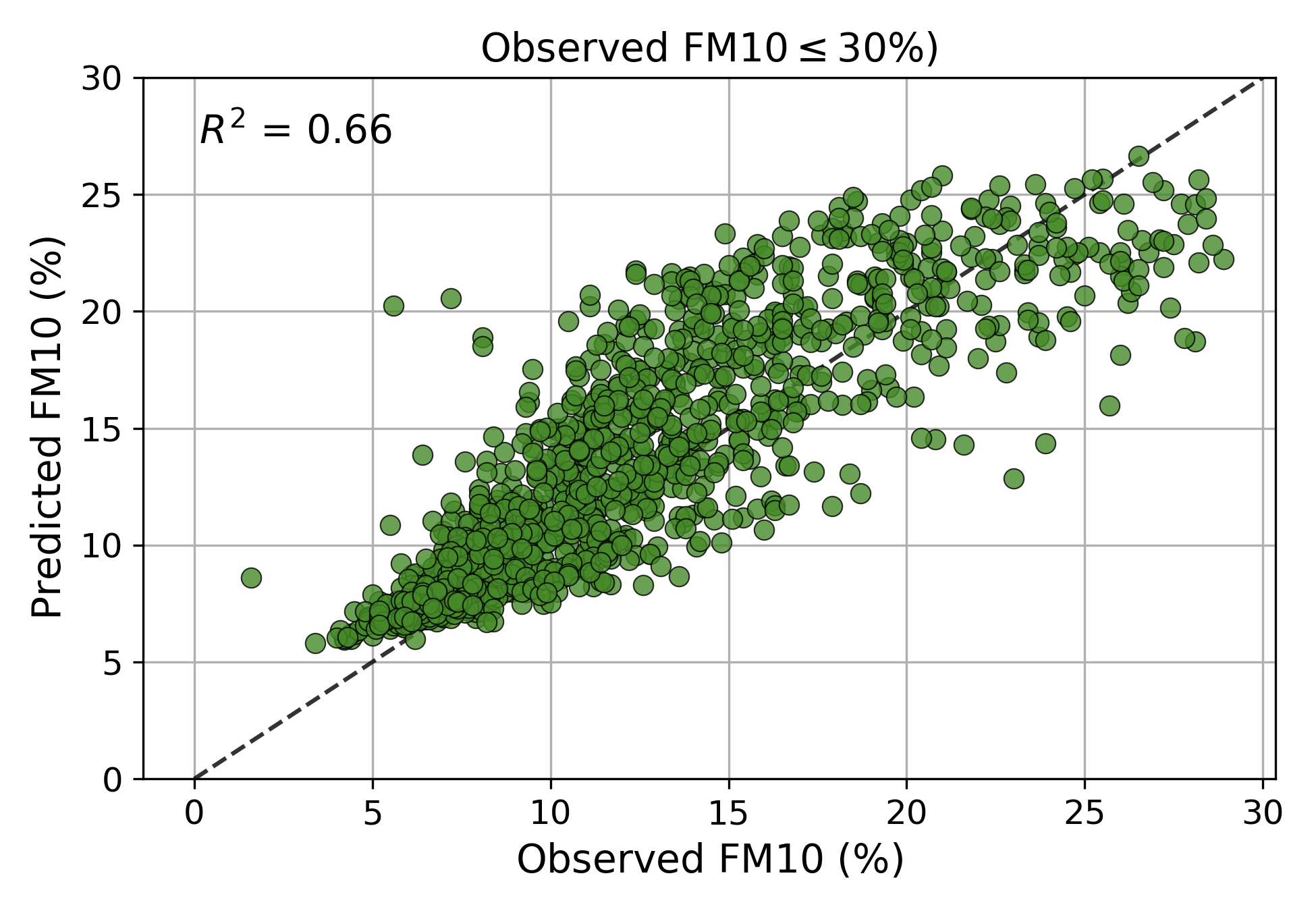}
\end{minipage}

\caption[
Observed vs Predicted FM10 for Zero-Shot Transfer RNN
]{
Observed vs Predicted FM10 for Zero-Shot Transfer RNN. The plotted RNN model is the set of weights that corresponds to the median RMSE on the test set out of 100 replications. Note the different x-axis ranges in the plots. 
(a) Top: All FM10 observations~(n=1,232). The zero-shot RNN predictions substantially underestimate the FM10 for wet fuels. 
(b) Bottom: FM10 observations filtered to less than or equal to $30\%$~(n=1,134). The zero-shot RNN predictions are much more accurate, though the model is overestimating the FM10 for very dry fuels and underestimating the FM10 for wetter fuels. 
}
\label{fig:fm10_scatters}
\end{figure}

For all FM10 values, the RNN performs substantially worse than the Nelson model in terms of $R^2$. We believe this is driven almost entirely by the artifacts of the FM10 fuel sensors. Some stations in the Synoptic data network report the manufacturer of the fuel sensor, and the majority are manufactured by FTS. These sensors appear to have a maximal response to rain, where the FM10 values are under 28$\%$. Figure~\ref{fig:ts_BAWC2} shows one week of sensor measurements from a station near Green Mountain, Colorado, southwest of Denver, Colorado. The RAWS identification number is BAWC2, and it is located at latitude and longitude coordinates (39.379440, -105.338330). The wetting and drying equilibria and the hourly rain are derived from the HRRR weather forecasts at that location and time. The equilibria are calculated from the air temperature and the relative humidity according to formula~\eqref{eq:eqs}. During the week from May 7, 2024, through May 14, there were a series of rain events that were forecasted by HRRR. The corresponding FM10 values from the sensor are under 28$\%$. We believe this is an artifact from the FTS sensor. Figure~\ref{fig:ts_zeroshot} shows the weather, FM10 field observations, and RNN zero-shot predictions from the week of March 27, 1996, through April 3 in Oklahoma. The RNN predictions are from the model replication with the median RMSE. There is a large rain event in the middle of that week, and the FM10 just after the rain is elevated up to nearly 50$\%$. The RNN predictions, however, reach a peak of about 27$\%$. We conclude that the RNN has learned this pattern from the sensor training data. This leads to large modeling errors for those times. However, an FM10 of 27$\%$ and an FM10 of $50\%$ are both very wet relative to the theoretical moisture of extinction for most fuels. The 13 Anderson fuel categories are widely used by wildfire researchers, and the values of the moisture of extinction for different fuels are almost all below 25$\%$~\citep{Anderson-1982-ADF}. While the modeling error for the observed FM10 of $50\%$ is large, the RNN is correctly predicting that the moisture is elevated above levels that are dangerous for rapid wildfire spread. That is why many FMC researchers choose to filter values below a high threshold of about $30\%$, as in \citet{Carlson-2007-ANM}. Further, the RNN can still in principle learn the relevant dynamics over time. That is, the RNN could learn to predict a near-constant maximal value of about 27$\%$ during times that the FM10 in reality is elevated due to rain.

\begin{figure}[htbp]
\centering
\includegraphics[width=14 cm]{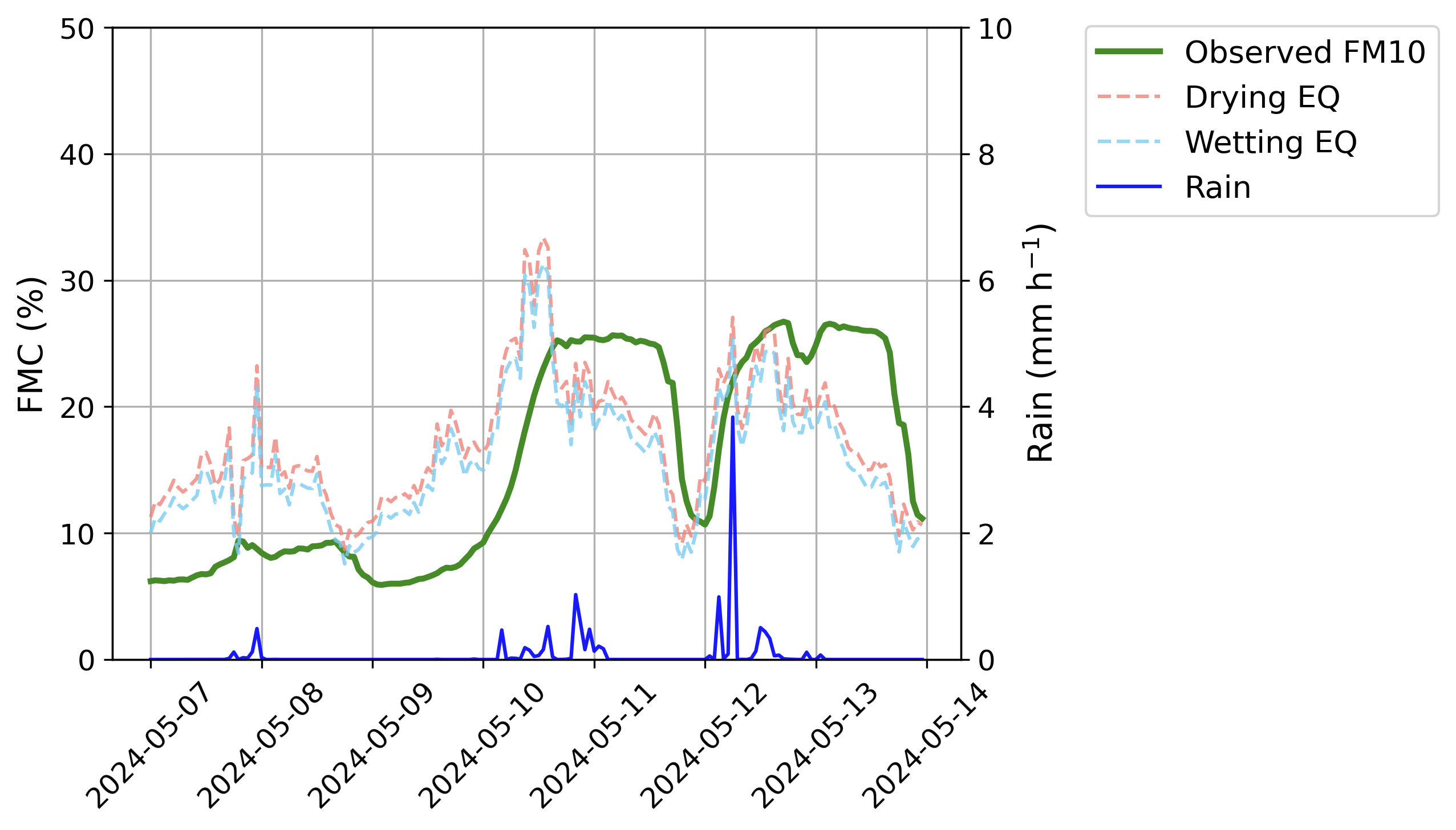}
\caption[Time series of weather and FM10 observations from a RAWS in Colorado.]{One week of FM10 sensor measurements and corresponding HRRR weather at BAWC2, southwest of Denver. The wetting and drying equilibria are derived from the forecasted air temperature and RH. The sensor appears to have a maximal response to rain, with the FM10 reaching a maximum value of around 27$\%$.}
\label{fig:ts_BAWC2} 
\end{figure} 

\begin{figure}[htbp]
\centering
\includegraphics[width=14 cm]{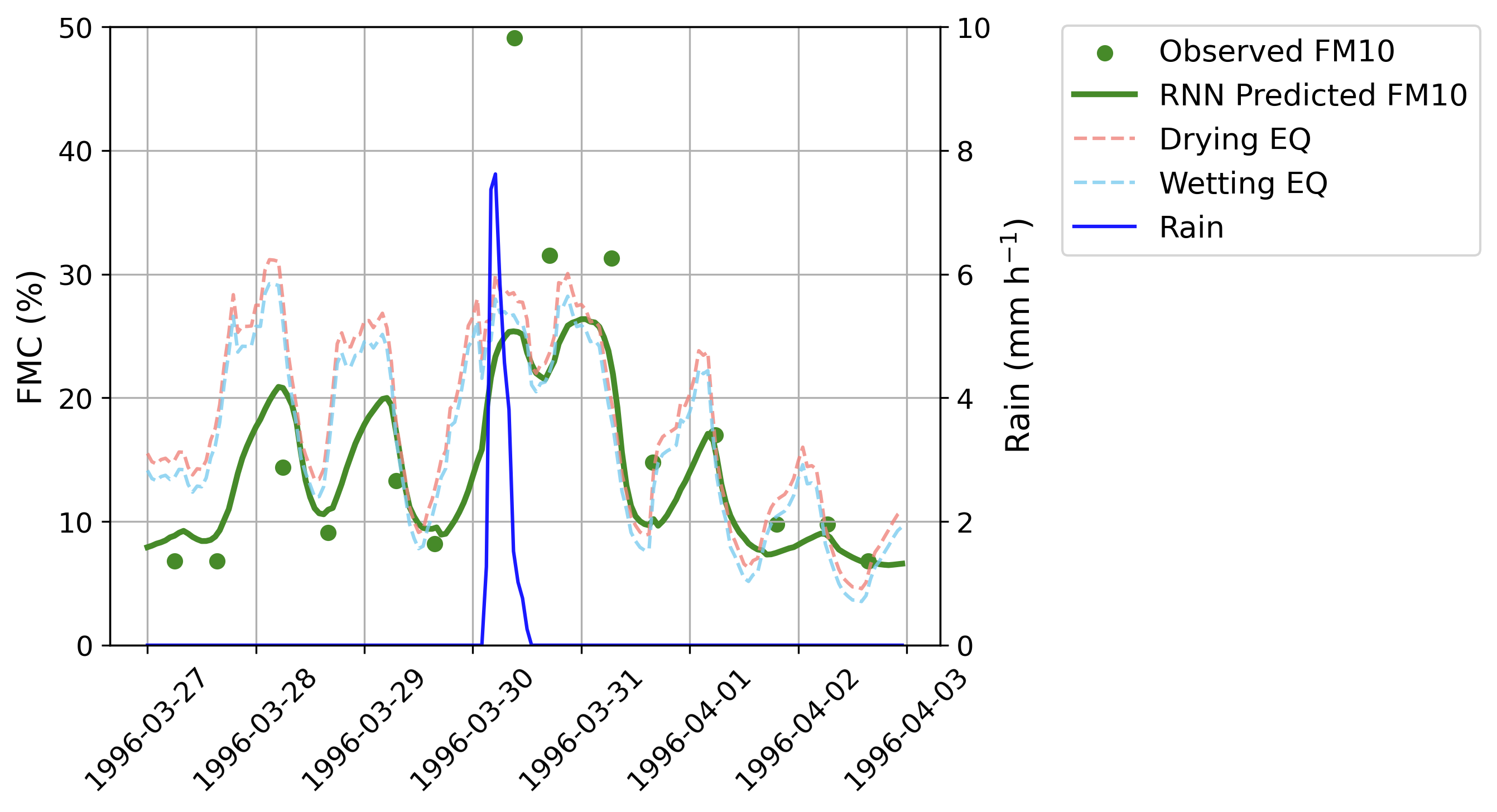}
\caption[RNN Zero-shot predictions with time series of weather and FM10 observations from Oklahoma field data.]{One week of FM10 field measurements in Oklahoma, the corresponding ground-based weather observations, and the RNN zero-shot predictions. The RNN predictions are from the model replication with the median RMSE in the test set. The RNN predictions reach a maximum of about 27$\%$, similar to the RAWS training data.}
\label{fig:ts_zeroshot} 
\end{figure} 

A large-scale analysis was conducted in a previous publication to estimate the forecast accuracy for the RNN. The analysis evaluated the forecast accuracy over all of 2024 in the Rocky Mountain region. The overall RMSE was 3.02~$\pm$~1.12 and the bias was -0.19~$\pm$~0.33~\citep{Hirschi-2026-RNN}. The confidence bounds were defined as one standard deviation over 500 replications. The replications varied the weather stations that were included in the training versus test set to account for the spatial variability of the weather stations. Also, the random seeds were changed for the replications. The RNN was initialized with different random parameters to account for the uncertainty introduced by initialization and gradient descent. The RMSE of $3.21\pm0.25\%$ across 100 replications for the zero-shot RNN forecast is within the estimated uncertainty bounds from the paper, so we consider this to confirm the utility of the RNN in this context. The bias of $-0.74\pm0.53\%$ falls partially outside the estimated uncertainty bounds from the paper. Further, in this case the bias is systematically less than zero, while in the paper the uncertainty bounds had zero within it. Therefore, we conclude the RNN is under-predicting the 10h dead FMC in the Carlson field data to a greater degree than for the fuel sensors. However, the bias value of $-0.74\%$ is still relatively small, and it is smaller in absolute terms than the bias for the Nelson model. 

As a baseline of comparison, Table~\ref{tab:fm10_static} shows the same accuracy metrics for a static XGBoost and a linear regression model. The static models perform better than the RNN in terms of RMSE when including very wet fuels, since they were trained directly on the Oklahoma field observations and do not suffer from the issues related to the FM10 sensors. The $R^2$ metrics are comparable for all FM10 observations. However, the RNN was substantially more accurate than the static models when FM10 was filtered to less than or equal to $30\%$. The accuracy metrics for the static models confirm that the zero-shot FM10 RNN is more accurate than simple ML baselines. 

Zero-shot transfer learning for FM10 confirms that the RNN is reliably transferring knowledge in this case. The input weather data came from numerical weather models in the source domain and from ground-based instruments in the target domain. The observations of FM10 were from fuel sensors in the source task and from gravimetric measurements in the target task. While the accuracy of the RNN was poor when considering wet conditions over $30\%$, the RNN was highly accurate when filtering to FM10 values less than $30\%$. The accuracy metrics for the RNN were slightly better than the Nelson model for those values of FM10. The Nelson model was also specifically calibrated to the gravimetric observations from the Oklahoma field study. The RNN was not trained with FM10 values from that exact location, nor with any gravimetric measurements at all. Thus, the RNN was disadvantaged relative to the Nelson model in multiple ways. We conclude that the zero-shot RNN was at least as accurate as the Nelson model for operationally relevant values of FM10 for the given study location and times. The Nelson model remains the state-of-the-art FMC model used by researchers and agencies, so this is an encouraging analysis that points to the potential operational use of the RNN.

\subsection{Stability of the Learned Gate Bias Structure for the Pretrained FM10 Model - Results}\label{sec:stable}

This section analyzes the 100 replications of the pretrained FM10 RNN. This is the model from the source learning task, and no transfer learning has been applied yet. We analyze the resulting model accuracy and the distributions of the forget and input gates to provide evidence that the RNN replications show a stable learned gate bias structure. 

The fitting accuracy of the RNN was calculated by comparing the predicted FM10 values at the end of training to the observed FM10 from the training samples and from the validation set for each of the 100 replications. We denote this as the fitting accuracy to contrast it with the forecast accuracy, which is a more rigorous model evaluation that requires spatiotemporal extrapolation. Table~\ref{tab:rnn_rep_errs} shows the fitting accuracy metrics for the training set and validation set averaged over the 100 replications. The range of the RMSE is the minimum and maximum observed RMSE across the replications. The number of observations is denoted $N$, and we report the mean and standard deviation of the number of training observations used. The value $N$ is the individual number of hours of sensor observations that were used, and it varies across replications due to splitting the sampling locations into training and validation sets. Specific locations have different numbers of available observations over time. Sensors occasionally break or stop reporting for whatever reason and need to be serviced, and sometimes new RAWS are added to the network so they will not have a full year of data for their first year of operations. 

\begin{table}[htbp]
\centering
\caption[Fitting accuracy of the FM10 RNN across replications.]{Fitting accuracy of the FM10 RNN across 100 replications.}
\label{tab:rnn_rep_errs}
\begin{tabular}{ll}
\toprule
Metric & Values \\
\midrule
RMSE (Mean $\pm$ SD) 
& 2.81 $\pm$ 0.49 \\

RMSE Range (Low, High) 
& (2.63, 2.91) \\

N (Mean $\pm$ SD) 
& 65,691,366 $\pm$ 589,543 \\
\bottomrule
\end{tabular}
\end{table}

The empirical RMSE across 100 replications was 2.81~$\pm$~0.49. The estimated forecast RMSE for the RNN was 3.02~$\pm$~1.12. The fitting RMSE is within the estimated bounds of the forecast RMSE. Similarly, the validation set RMSE across 100 replications was 2.47~$\pm$~1.67, which is also consistent with the estimated forecast accuracy. On average, the training set utilized over 65 million observations of FM10 for training. Contrast that with the Oklahoma field study, which had 1,232 observations of FM10 and even less for FM100 and FM1000 fuels.

The consistency of the fitting accuracy across replications is one piece of evidence that the learned structure of the RNN is stable. If we observed drastically different fitting accuracies for the RNN across replications, then we could not conclude that the RNN is converging to a stable set of learned dynamics. It is still possible that the similar fitting accuracy metrics could be produced by RNNs with drastically different weights, but the consistency of the fitting accuracy is a necessary condition for the stability of the learned dynamics. The fitting accuracy across replications increases our confidence that the time-warp method can be deployed on the bias terms of the LSTM without worrying that the bias values are completely arbitrary. Further, we conclude that there is no strong evidence that the RNN is being overfit to the training set. This increases our confidence that the RNN is learning deeper aspects of the FM10 dynamics and can be reliably deployed for the purposes of transfer learning. 

\begin{table}[htbp]
\centering
\caption[Summary Statistics for Empirical Distributions of Gate Biases]{Summary Statistics for Empirical Distributions of Gate Biases.}
\label{tab:bias_stability}
\begin{tabular}{lccc}
\toprule
Metric 
& Forget Gate Bias 
& Input Gate Bias 
& Difference \\
\midrule
\makecell{Mean of per-model means \\ ($\pm$ SD across reps)}
& 1.14 $\pm$ 0.04 
& 0.11 $\pm$ 0.08 
& 1.03 $\pm$ 0.10 \\
\midrule
\makecell{Range of per-model means \\ (low, high) }
& (1.08, 1.32) 
& (-0.45, 0.24) 
& (0.90, 1.56) \\
\midrule
\makecell{Mean of per-model SDs \\ ($\pm$ SD across reps) }
& 0.23 $\pm$ 0.15 
& 0.39 $\pm$ 0.17 
& 0.47 $\pm$ 0.22 \\
\midrule
\makecell{Range of per-model SDs \\ (low, high) }
& (0.14, 1.37) 
& (0.21, 1.42) 
& (0.25, 1.91) \\
\bottomrule
\end{tabular}
\end{table}

Table~\ref{tab:bias_stability} shows summary statistics for the gate biases and their per-unit difference. Variability in the first two moments across replications is modest relative to their magnitude for the forget gate bias, input gate bias, and their difference. In each case, the across-replication standard deviation of the per-model mean is smaller than the mean itself, and the same holds for the per-model standard deviation. This indicates that both the central tendency and overall spread of the forget gate bias distribution are consistent across replications. The input gate bias distributions are more variable, but still reasonably consistent. Therefore, the learned bias structure exhibits stability in its low-order summaries, even though fine-scale distributional differences exist.

Next, we visually examine the distributions of the forget and input gate biases across replications. Figure~\ref{fig:lstm_bias_hist} shows the distribution of the mean bias terms across 100 replications, which are relatively consistent. The full distributions are shown in Appendix~\ref{app:extra} and are visualized using Gaussian kernel density estimation to provide a smooth representation for comparison. The replications are sorted by the sample variance of the distribution. Figures~\ref{fig:bf_hist},~\ref{fig:bi_hist}, and~\ref{fig:diff_hist} show the full empirical distribution across 100 replications of the forget gate bias, the input gate bias, and the difference between the two. The difference is calculated by subtracting the input gate bias from the forget gate bias for each individual unit. The global mean is plotted as a dashed vertical line. The distributions show variability, but generally they have the same central tendency and similar spread. The distribution of the difference $(b_f - b_i)$ shows a consistent pattern. In principle, there is no requirement that these quantities align so cleanly. Due to permutation invariance, the shapes of the distributions in Figures~\ref{fig:bf_hist} and~\ref{fig:bi_hist} do not imply that Figure~\ref{fig:diff_hist} will show a stable pattern across replications. 

\begin{figure}[htbp]
\centering
\begin{minipage}{0.48\textwidth}
    \centering
    \includegraphics[width=\linewidth]{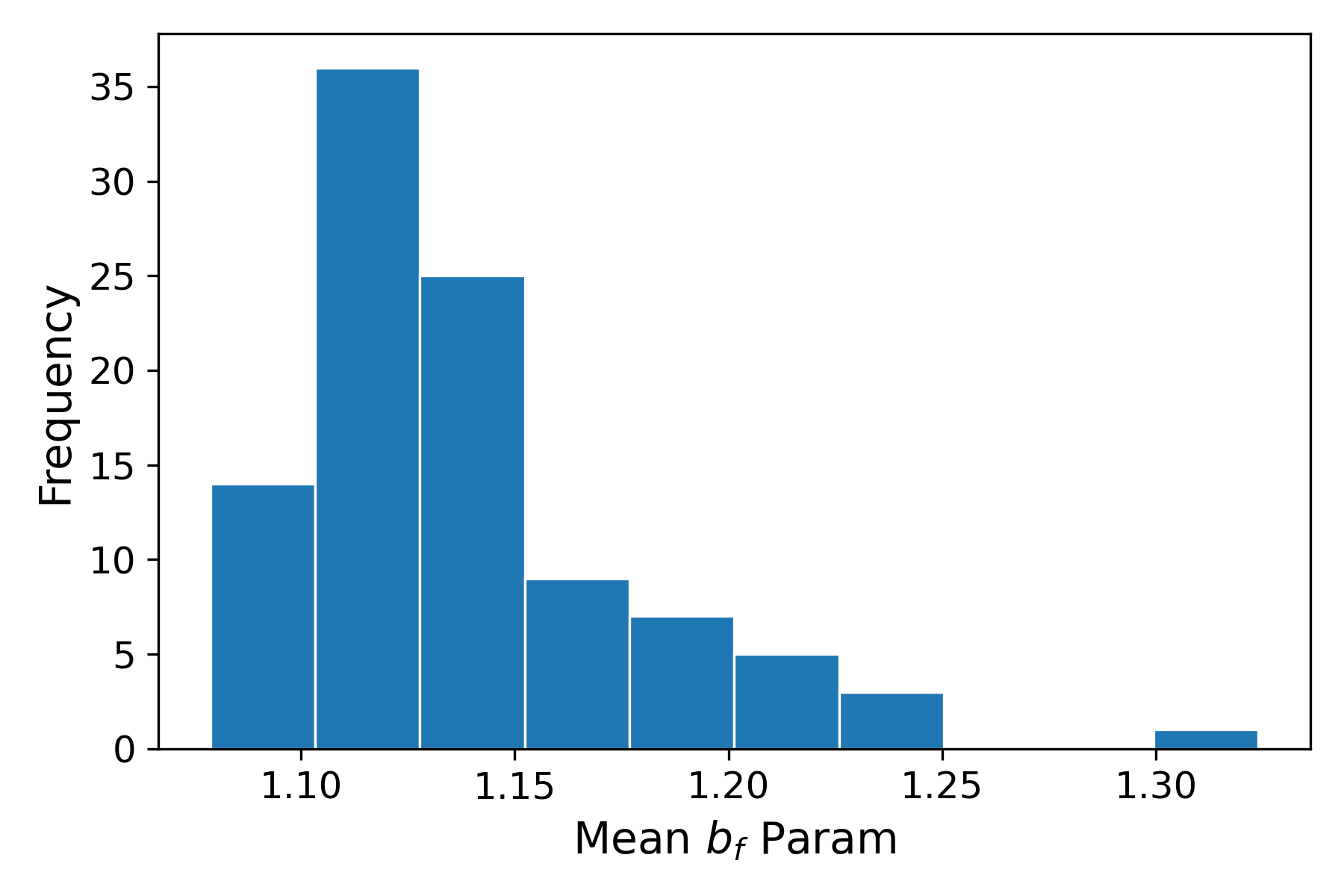}
\end{minipage}
\hfill
\begin{minipage}{0.48\textwidth}
    \centering
    \includegraphics[width=\linewidth]{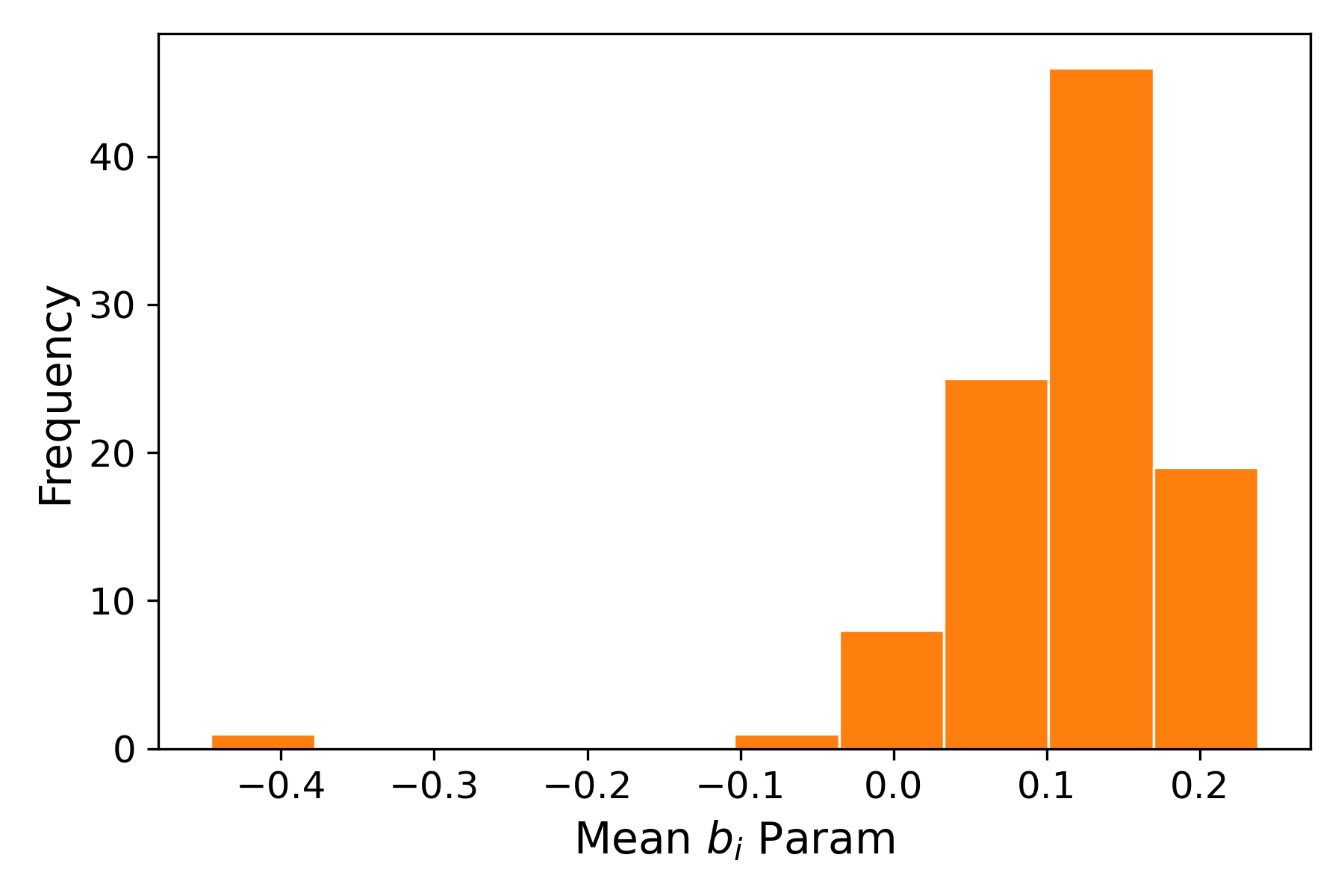}
\end{minipage}
\caption[Distribution of mean LSTM gate biases across replications.]{Distribution of the mean LSTM forget-gate ($b_f$) and input-gate ($b_i$) bias parameters across 100 training replications. In each replication, the mean bias across the 64 units in the LSTM layer is computed, producing one value per run. The histograms therefore summarize the distribution of these mean biases over the 100 replications.}
\label{fig:lstm_bias_hist}
\end{figure}

However, a formal statistical test rejects the null hypothesis that the distributions are identical. The k-sample Anderson–Darling test is a nonparametric procedure used to determine whether multiple samples arise from a common underlying distribution~\citep{Scholz-1987-KSA}. The test assumes that the samples are independent and drawn from continuous populations under the null hypothesis of a common distribution; the test statistic is based on a weighted measure of the squared differences between the empirical distribution functions of the samples and the pooled distribution, placing greater emphasis on discrepancies in the distribution tails. The k-sample Anderson–Darling test rejects the hypothesis that all replications share an identical distribution for $b_f$, $b_i$, and their difference, with p-values less than 0.001 in each case. This indicates that at least one replication’s empirical distribution differs.

The empirical distributions for the forget and input gates and their per-unit difference provide statistical evidence that the RNN is learning a stable set of parameters. This lends credibility to the method of time-warping the bias parameters for this particular RNN. Formal statistical testing nevertheless indicates that at least one replication exhibits a detectably different empirical distribution, suggesting that while the bias structure is broadly consistent, it is not identical across all runs.

\subsection{FM1 Transfer Learning Results}

The testing period for FM1 was from August 1997 through December 1997. A total of 271 observations were collected over this period, of which 247 had FM1 values less than or equal to $30\%$. Table~\ref{tab:fm1_all} shows the accuracy metrics on the test set for the different transfer learning approaches, and Figure~\ref{fig:fm1_rmse} shows a visual comparison of the RMSE statistics. Table~\ref{tab:fm1_le30} shows the same metrics with FM1 filtered to values less than or equal to $30\%$, and Figure~\ref{fig:fm1_30_rmse} shows a visual comparison of the corresponding RMSE statistics. Both tables also include the corresponding metrics for the Nelson model for comparison. The RMSE and bias are interpretable in units of percent FMC. A positive bias indicates underprediction, while a negative bias indicates overprediction.

\begin{table}[htbp]
\centering
\caption[FM1 Accuracy Metrics - Transfer Learning Comparison.]{Accuracy metrics for FM1 RNN predictions compared to field observations in Oklahoma for different transfer learning methods. Metrics are reported as mean $\pm$ standard deviation across 100 replications~(n=271).}
\label{tab:fm1_all}
\begin{tabular}{lccc}
\hline
Method & $R^2$ & Bias (\%) & RMSE (\%) \\
\hline
No Transfer & $0.78 \pm 0.03$ & $0.84 \pm 0.71$ & $6.31 \pm 0.40$ \\
Full Fine-Tuning & $0.84 \pm 0.02$ & $2.24 \pm 0.53$ & $5.31 \pm 0.38$ \\
Freeze Recurrent Layer & $0.79 \pm 0.02$ & $1.23 \pm 0.76$ & $6.06 \pm 0.36$ \\
Freeze Dense/Output Layers & $0.84 \pm 0.03$ & $2.19 \pm 0.46$ & $5.39 \pm 0.43$ \\
Time-Warping & $0.41 \pm 0.03$ & $2.46 \pm 0.42$ & $10.29 \pm 0.23$ \\
Time-Warping + Fine-Tuning & $0.84 \pm 0.05$ & $1.11 \pm 1.04$ & $5.26 \pm 0.71$ \\
\hline\hline
Nelson Model* & $0.64$ & $-2.7$ & - \\
\hline
\end{tabular}
\end{table}

\begin{figure}[htbp]
\centering
\includegraphics[width=14 cm]{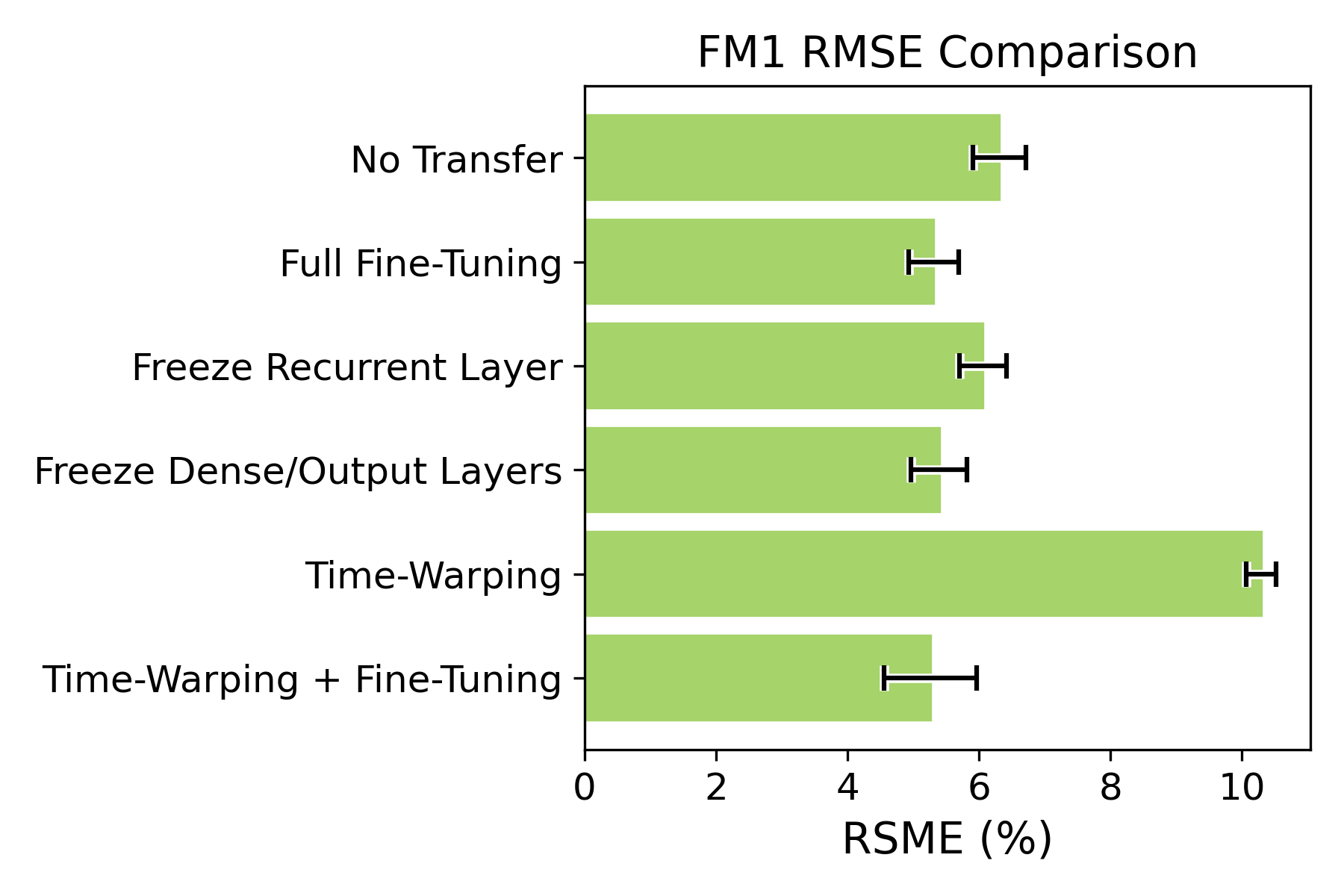}
\caption[FM1 RMSE by Transfer Learning Method.
]{FM1 RMSE comparison for different transfer learning methods. Bars represent the mean RMSE when predicting FM1 on the test set over 100 replications, and brackets show $\pm$ 1 standard deviation of the RMSE. }
\label{fig:fm1_rmse} 
\end{figure} 

The RNN-based methods perform relatively poorly when evaluated against all FM1 observations, as indicated by RMSE values over $5\%$. This is primarily due to issues with modeling very wet fuels. However, all the methods except for the Time-Warping method achieved higher $R^2$ metrics and lower bias in absolute terms than the Nelson model, though we cannot distinguish whether this difference arises from the RNNs being evaluated on a testing subset of the data versus the Nelson model.

The Time-Warping method with no fine-tuning performs the worst, with an RMSE over $10\%$ and a mean bias of $2.46\%$, indicating systematic underprediction. This method starts with the FM10 RNN from the source learning task that was trained on fuel sensors, which have a maximal response to rain. Only the two time-warping parameters are fit to actual FM1 observations, and the other remaining parameters remain fixed. The issue with predicting wet fuels persists. Fine-tuning alleviates much of this issue, resulting in the lowest mean RMSE of $5.26\%$ and tying for the best mean $R^2$ of $0.84$. The uncertainty bounds for RMSE and $R^2$ for the Time-Warping plus Fine-Tuning method indicate performance similar to the Full Fine-Tuning method and the Freeze Dense/Output Layers method.

\begin{table}[htbp]
\centering
\caption[FM1 Accuracy Metrics ($\leq 30\%$) - Transfer Learning Comparison.]{Accuracy metrics for FM1 RNN predictions compared to field observations in Oklahoma for different transfer learning methods. Metrics are reported as mean $\pm$ standard deviation across 100 replications. FM1 values are filtered $\leq 30\%$~(n=247).}
\label{tab:fm1_le30}
\begin{tabular}{lccc}
\hline
Method & $R^2$ & Bias (\%) & RMSE (\%) \\
\hline
No Transfer & $0.66 \pm 0.06$ & $-0.30 \pm 0.62$ & $3.60 \pm 0.29$ \\
Full Fine-Tuning & $0.67 \pm 0.05$ & $1.43 \pm 0.49$ & $3.59 \pm 0.29$ \\
Freeze Recurrent Layer & $0.43 \pm 0.11$ & $0.30 \pm 0.74$ & $4.69 \pm 0.45$ \\
Freeze Dense/Output Layers & $0.72 \pm 0.05$ & $1.40 \pm 0.40$ & $3.30 \pm 0.27$ \\
Time-Warping & $0.71 \pm 0.04$ & $-0.28 \pm 0.43$ & $3.35 \pm 0.25$ \\
Time-Warping + Fine-Tuning & $0.64 \pm 0.11$ & $0.62 \pm 0.92$ & $3.69 \pm 0.53$ \\
\hline\hline
Nelson Model* & $0.46$ & $-3.6$ & - \\
\hline
\end{tabular}
\end{table}

\begin{figure}[htbp]
\centering
\includegraphics[width=14 cm]{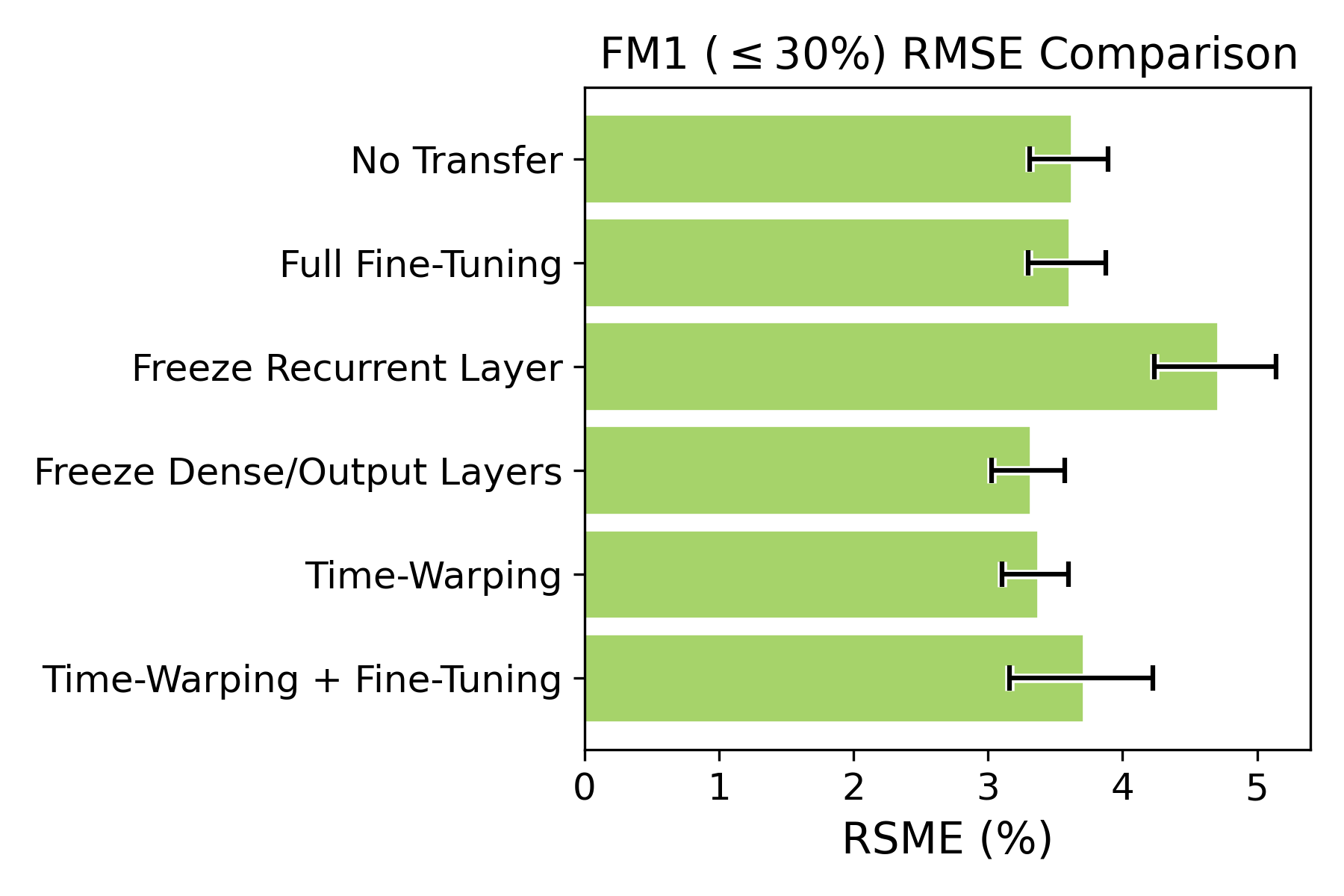}
\caption[FM1 ($\leq 30\%$) RMSE by Transfer Learning Method.
]{FM1 RMSE comparison for different transfer learning methods with FM1 filtered to ($\leq 30\%$). Bars represent the mean RMSE when predicting FM1 on the test set over 100 replications, and brackets show $\pm$ 1 standard deviation of the RMSE. }
\label{fig:fm1_30_rmse} 
\end{figure} 

For the FM1 values less than or equal to $30\%$, the RNN-based methods have substantially lower RMSE values. However, all of the $R^2$ values decrease except for the Time-Warping method, which highlights a limitation of the $R^2$ metric and motivates the use of RMSE as our primary evaluation metric. Even when the average prediction error is smaller, the $R^2$ can decrease because the variability of the FM1 values is reduced when the very wet observations are removed. The high variance introduced by the wet observations can therefore lead to a superficially higher $R^2$ value. The Nelson model had the second worst $R^2$ of $0.43$ and the largest bias at $-3.6\%$.

The Time-Warping method performed better when the FM1 values over $30\%$ are filtered out. The $R^2$ of $0.71$ was among the best and was within the uncertainty bounds of the best-performing method. The bias of $-0.28\pm0.43$ includes zero within the uncertainty bounds, so these predictions are consistent with being systematically unbiased. The mean RMSE of $3.35\%$ was the second best and again was within the uncertainty bounds of the RMSE from the most accurate method. The $R^2$ was substantially better than the Nelson model baseline, again noting the different data sets and limitations of $R^2$.

Somewhat surprisingly, fine-tuning did not improve the accuracy of the Time-Warping method in this case. The mean RMSE for the Time-Warping plus Fine-Tuning method was $3.69\%$, which is larger and outside the uncertainty bounds of the RMSE for the Time-Warping method of $3.35\pm0.25$. We hypothesize this is due to the fine-tuning process selecting a set of parameters that minimized the large errors due to wet fuels, and this optimization came at the expense of accuracy for drier fuels. We discuss in Section~\ref{sec:future} ways to augment the training with custom loss functions that may help address this issue.

The results for the other transfer learning methods provide insight as to which aspects of the problem generalized between the source and target learning tasks. The Full Fine-Tuning method, which started with the pretrained FM10 RNN and allowed all weights to be updated, did not provide substantial improvements over the No Transfer baseline method. The $R^2$ and RMSE values are essentially the same, and the bias is worse for the Full Fine-Tuning method. This may be the result of the same issues with fine-tuning with very wet observations that we discuss above. Freezing the recurrent layer performed the worst, with a mean RMSE of $4.69\%$ that is outside of the confidence bounds of the RMSE for all the other methods. This supports the idea that the temporal dynamics need to be adjusted when transferring knowledge from the FM10 source task. Freezing the dense and output layers was the most accurate method, but the Time-Warping method showed consistent accuracy in terms of $R^2$ and RMSE, and it had a better bias. Freezing the dense and output layers performed the best. The estimated RMSE uncertainty bounds were $3.30\pm0.27$, so the mean RMSE estimate for the No Transfer baseline is outside of this interval, though there is overlap with its own uncertainty bounds. We conclude that transferring knowledge from the source learning task is improving the predictive accuracy. 

However, the main result from the FM1 accuracy metrics is that the Time-Warping method was consistent in terms of accuracy with the standard transfer learning baselines, while there is reason to believe that it is less prone to overfitting to the sparse target data. The test set consisted of 247 observations of FM1 less than or equal to $30\%$ at a single location over 4 months. We cannot strongly conclude that the accuracy for the Freeze Dense/Output Layers method will generalize to other locations and times. The Time-Warping method fit two parameters to the observed FM1 data, which modified 128 out of over 21,000 trainable parameters. Further, the pretrained FM10 RNN had an RMSE of $3.02\pm1.12$ and a bias of $-0.19\pm0.33$, which were calculated on a much larger dataset at 151 different spatial locations. The accuracy metrics for the Time-Warping method are consistent for FM1$\leq30\%$; the mean bias and RMSE for this method fall within the uncertainty bounds estimated in the source learning task.

Figure~\ref{fig:fm1_scatters} shows the relationship between the observed and predicted FM1 values. The figures use the FM1 predictions from a single RNN selected from the 100 trained models, specifically the model whose RMSE was closest to the median RMSE across all replications. We display a single representative model to illustrate the behavior that can be expected from an individual model prediction. The scatterplot on the right side, which filters FM1$\leq30\%$, shows that the RNN tends to overestimate FM1 for drier fuels and underestimate FM1 for wetter fuels. Overall, the plot shows relatively strong agreement between the RNN predictions and the observations.

\begin{figure}[htbp]
\centering
\begin{minipage}[c]{0.8\textwidth}
  \centering
  \includegraphics[width=\textwidth]{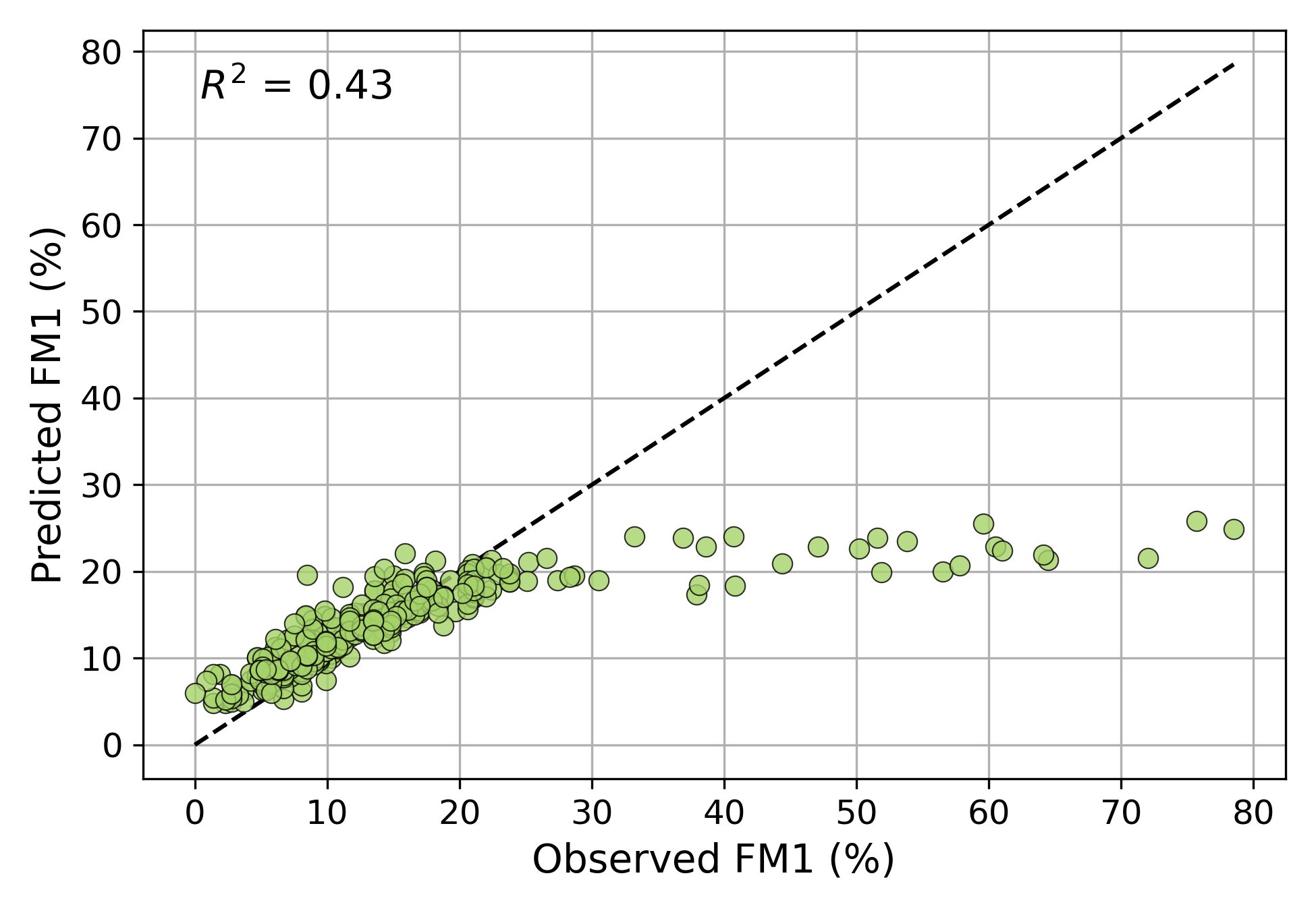}
\end{minipage}

\vspace{0.5cm}

\begin{minipage}[c]{0.8\textwidth}
  \centering
  \includegraphics[width=\textwidth]{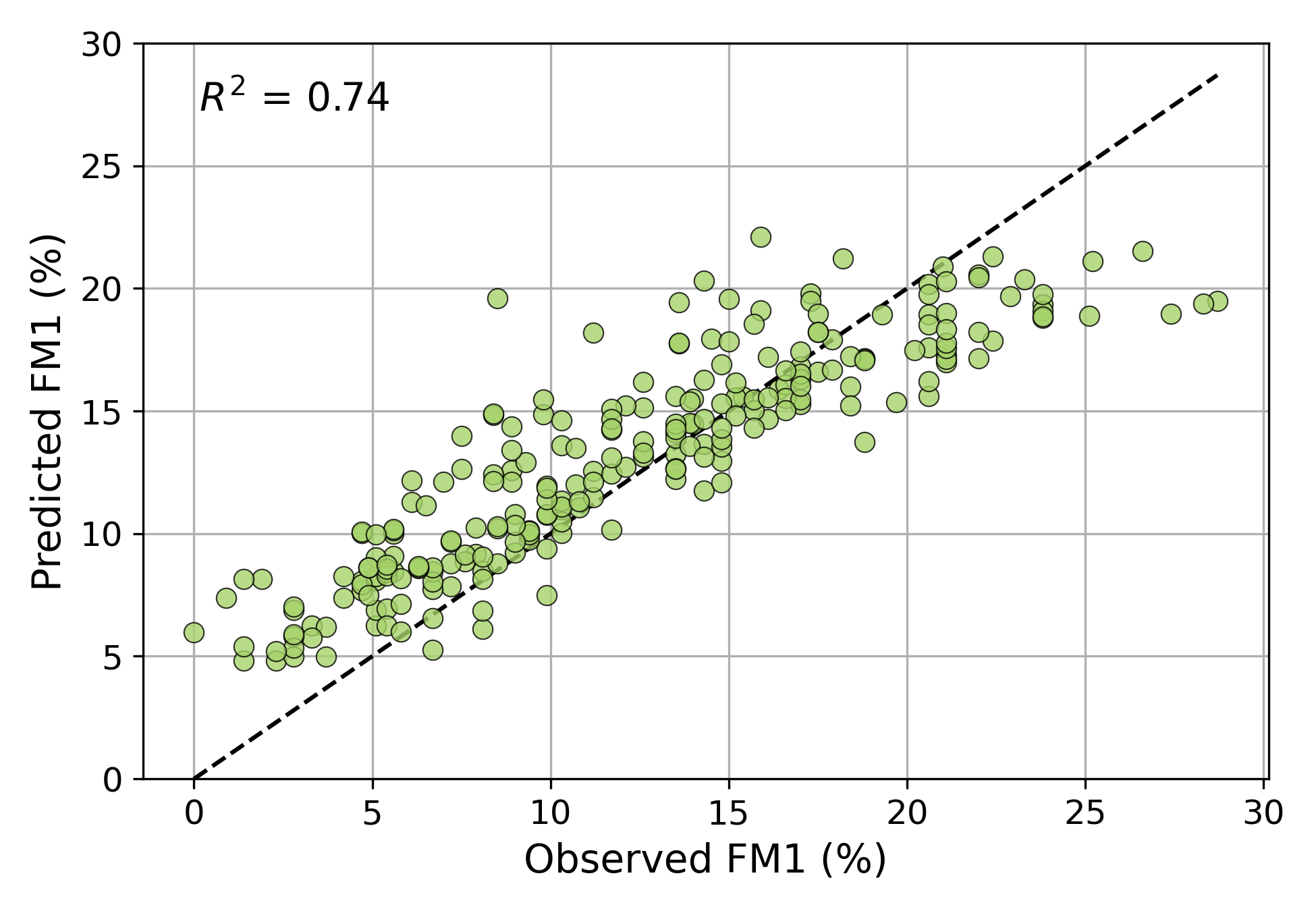}
\end{minipage}

\caption[
Observed vs Predicted FM1 for Time-Warped RNN
]{
Observed vs Predicted FM1 for Time-Warped RNN. The plotted RNN model is the set of weights that corresponds to the median RMSE on the test set out of 100 replications. 
(a) Top: All FM1 observations in the test set~(n=271). The RNN predictions substantially underestimate the FM1 for wet fuels. 
(b) Bottom: FM1 observations filtered to less than or equal to $30\%$~(n=247). The RNN predictions are much more accurate, though the model is overestimating the FM1 for very dry fuels and underestimating the FM1 for wetter fuels. 
}
\label{fig:fm1_scatters}
\end{figure}

\subsection{FM100 Transfer Learning Results}

The testing period for FM100 was from August 1997 through December 1997. A total of 206 observations were collected over this period, and no filtering was done for wet fuels. Table~\ref{tab:fm100_all} shows the accuracy metrics on the test set for the different transfer learning approaches, and Figure~\ref{fig:fm100_rmse} shows a visual comparison of the RMSE statistics. Both tables also include the corresponding metrics for the Nelson model for comparison. The RMSE and bias are interpretable in units of percent FMC. A positive bias indicates underprediction, while a negative bias indicates overprediction.

\begin{table}[htbp]
\centering
\caption[FM100 Accuracy Metrics - Transfer Learning Comparison.]{Accuracy metrics for FM100 RNN predictions compared to field observations in Oklahoma for different transfer learning methods~(n=206). Metrics are reported as mean $\pm$ standard deviation across 100 replications.}
\label{tab:fm100_all}
\begin{tabular}{lccc}
\hline
Method & $R^2$ & Bias (\%) & RMSE (\%) \\
\hline
No Transfer & $0.64 \pm 0.13$ & $-0.19 \pm 0.33$ & $2.60 \pm 0.44$ \\
Full Fine-Tuning & $0.84 \pm 0.06$ & $-0.41 \pm 0.41$ & $1.75 \pm 0.29$ \\
Freeze Recurrent Layer & $0.73 \pm 0.08$ & $-0.34 \pm 0.33$ & $2.25\pm 0.34$ \\
Freeze Dense/Output Layers & $0.75 \pm 0.09$ & $0.12 \pm 0.44$ & $2.15 \pm 0.38$ \\
Time-Warping & $0.70 \pm 0.08$ & $-0.33 \pm 0.72$ & $2.38 \pm 0.28$ \\
Time-Warping + Fine-Tuning & $0.78 \pm 0.10$ & $-0.64 \pm 0.56$ & $2.00 \pm 0.40$ \\
\hline\hline
Nelson Model* & $0.75$ & $-2.1$ & - \\
\hline
\end{tabular}
\end{table}

\begin{figure}[htbp]
\centering
\includegraphics[width=14 cm]{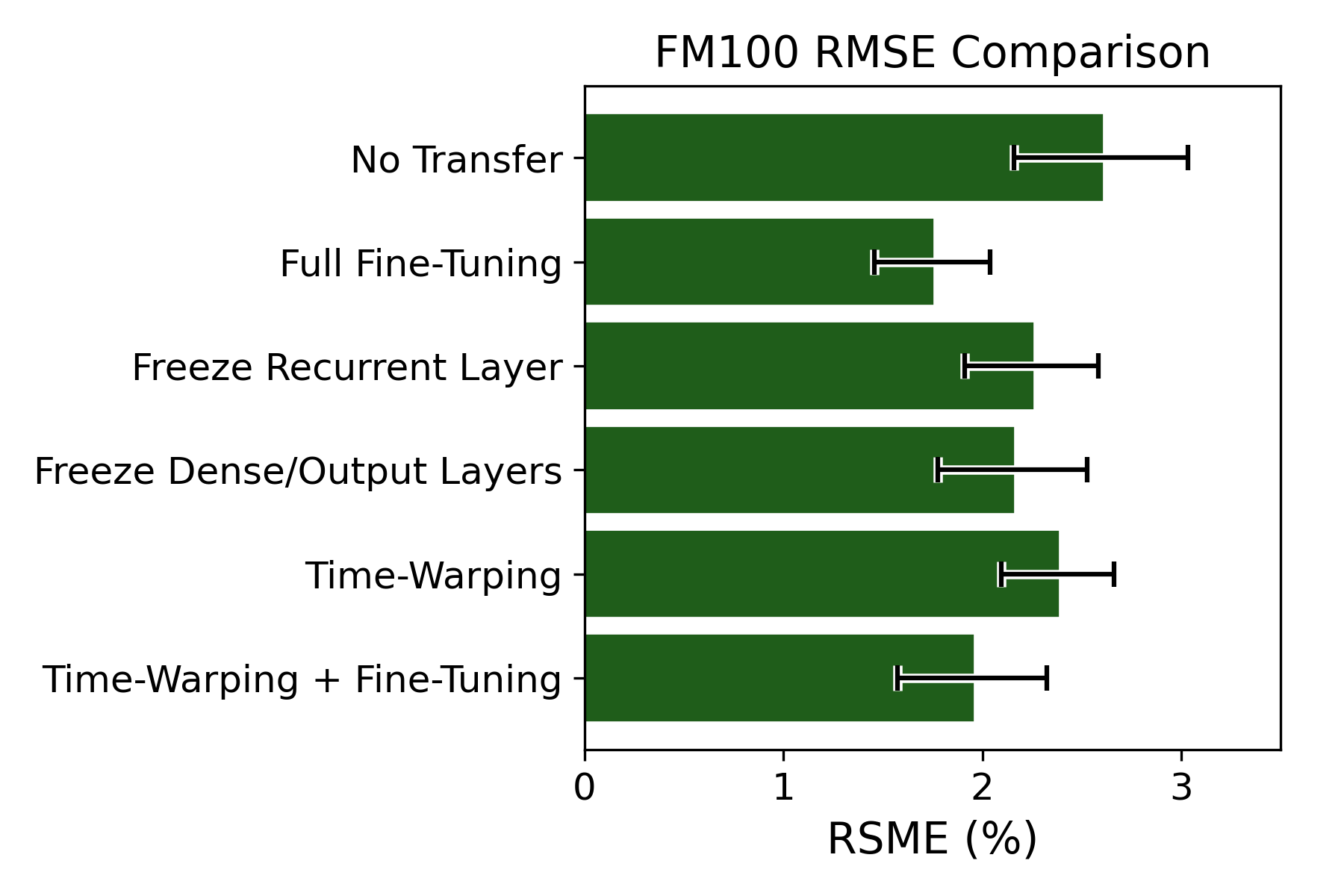}
\caption[FM100 RMSE by Transfer Learning Method.
]{FM100 RMSE comparison for different transfer learning methods. Bars represent the mean RMSE when predicting FM100 on the test set over 100 replications, and brackets show $\pm$ 1 standard deviation of the RMSE. }
\label{fig:fm100_rmse} 
\end{figure}

All methods resulted in higher accuracy for FM100 relative to FM1, likely due to the lower variability and slower dynamics of FM100. All RNN-based techniques appear to be systematically unbiased, with the estimated uncertainty bounds for the bias including zero. The Time-Warping method had a mean RMSE of $2.38\pm0.28$, which represents a modest improvement over the No Transfer baseline, although the mean RMSE for the No Transfer method falls within the uncertainty bounds. The Full Fine-Tuning and Time-Warping plus Fine-Tuning methods achieved the lowest RMSE values, and the uncertainty bounds for the Time-Warping method do not overlap with those estimates. While these results suggest that the fine-tuned methods were more accurate, the mean RMSE of $2.38$ for the Time-Warping method is less than $1\%$ FMC higher than the two most accurate methods, which is a small practical difference. As in the FM1$\leq30\%$ case, the Time-Warping method produces accuracy comparable to standard transfer learning baselines while modifying only a small fraction of the network parameters.

Figure~\ref{fig:fm100_scatter} shows the relationship between the observed and predicted FM100 values. The figures use the FM100 predictions from a single RNN selected from the 100 trained models, specifically the model whose RMSE was closest to the median RMSE across all replications. We display a single representative model to illustrate the behavior that can be expected from an individual model prediction. The scatterplot shows that the RNN, as in the FM1 case, tends to overestimate FM100 for drier fuels and underestimate FM100 for wetter fuels. Overall, the plot shows relatively strong agreement between the RNN predictions and the observations.

\begin{figure}[htbp]
\centering
\includegraphics[width=14 cm]{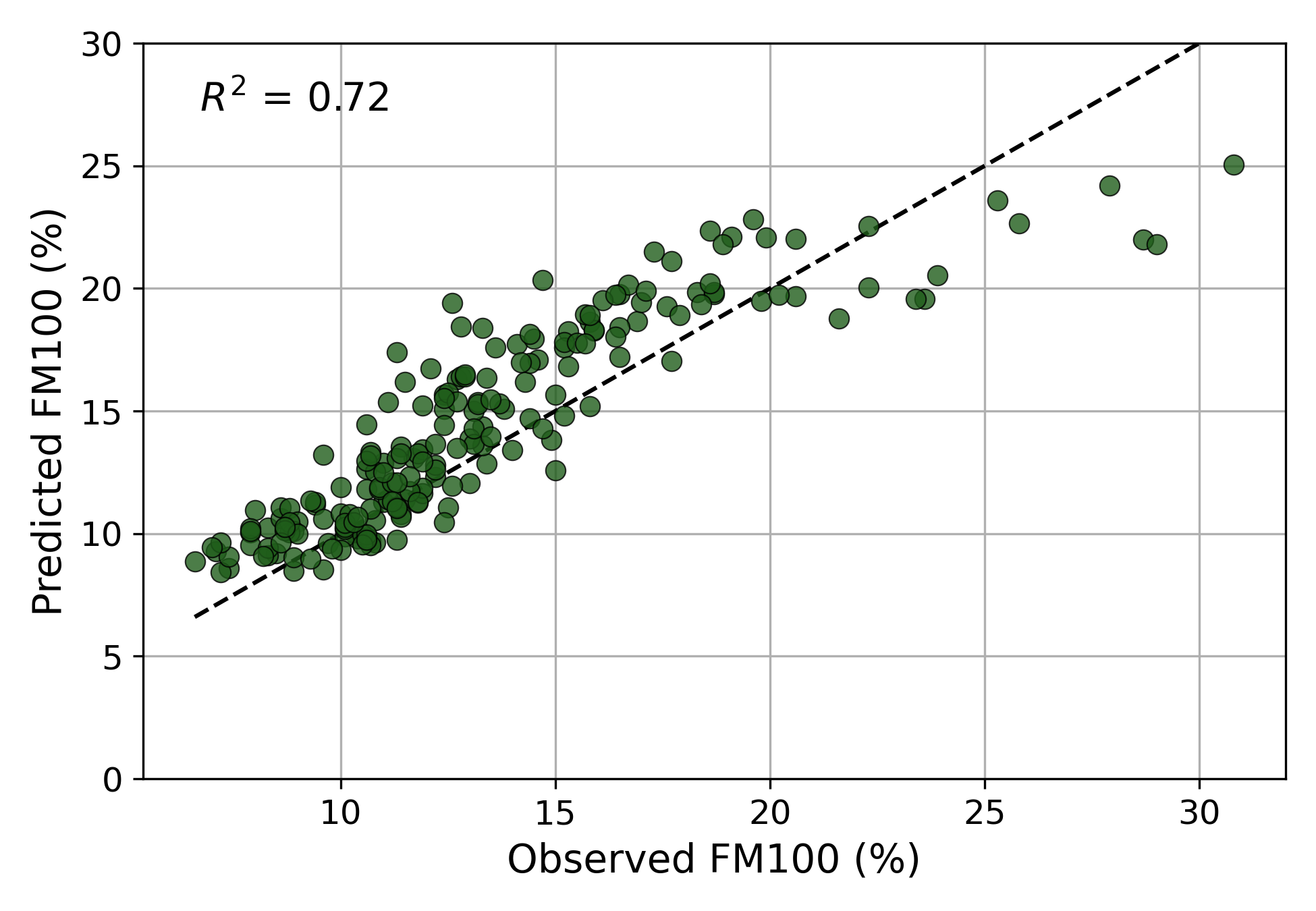}
\caption[Observed vs Predicted FM100 for Time-Warped RNN
]{Observed vs Predicted FM100 for Time-Warped RNN. The plot includes all FM100 observations in the test set~(n=206). The plotted RNN model is the set of weights that corresponds to the median RMSE on the test set out of 100 replications. }
\label{fig:fm100_scatter} 
\end{figure} 

\subsection{FM1000 Transfer Learning Results}

The testing period for FM1000 was from August 1997 through December 1997. A total of 209 observations were collected over this period, and no filtering was done for wet fuels. Table~\ref{tab:fm1000_all} shows the accuracy metrics on the test set for the different transfer learning approaches, and Figure~\ref{fig:fm1000_rmse} shows a visual comparison of the RMSE statistics. Both tables also include the corresponding metrics for the Nelson model for comparison. Both the RMSE and the bias are interpretable in units of percent FMC. A positive bias indicates underprediction, while a negative bias indicates overprediction.

\begin{table}[htbp]
\centering
\caption[FM1000 Accuracy Metrics - Transfer Learning Comparison.]{Accuracy metrics for FM1000 RNN predictions compared to field observations in Oklahoma for different transfer learning methods~(n=209). Metrics are reported as mean $\pm$ standard deviation across 100 replications.}
\label{tab:fm1000_all}
\begin{tabular}{lccc}
\hline
Method & $R^2$ & Bias (\%) & RMSE (\%) \\
\hline
No Transfer & $0.57 \pm 0.12$ & $-0.50 \pm 0.29$ & $2.16 \pm 0.31$ \\
Full Fine-Tuning & $0.59 \pm 0.05$ & $-0.92 \pm 0.21$ & $2.13 \pm 0.14$ \\
Freeze Recurrent Layer & $0.50 \pm 0.09$ & $-0.42 \pm 0.55$ & $2.33\pm 0.20$ \\
Freeze Dense/Output Layers & $0.50 \pm 0.09$ & $0.65 \pm 0.70$ & $2.33 \pm 0.21$ \\
Time-Warping & $0.54 \pm 0.18$ & $-0.81 \pm 0.71$ & $2.22 \pm 0.43$ \\
Time-Warping + Fine-Tuning & $0.65 \pm 0.10$ & $-0.75 \pm 0.36$ & $1.94 \pm 0.27$ \\
\hline\hline
Nelson Model* & $0.56$ & $-1.0$ & - \\
\hline
\end{tabular}
\end{table}

\begin{figure}[htbp]
\centering
\includegraphics[width=14 cm]{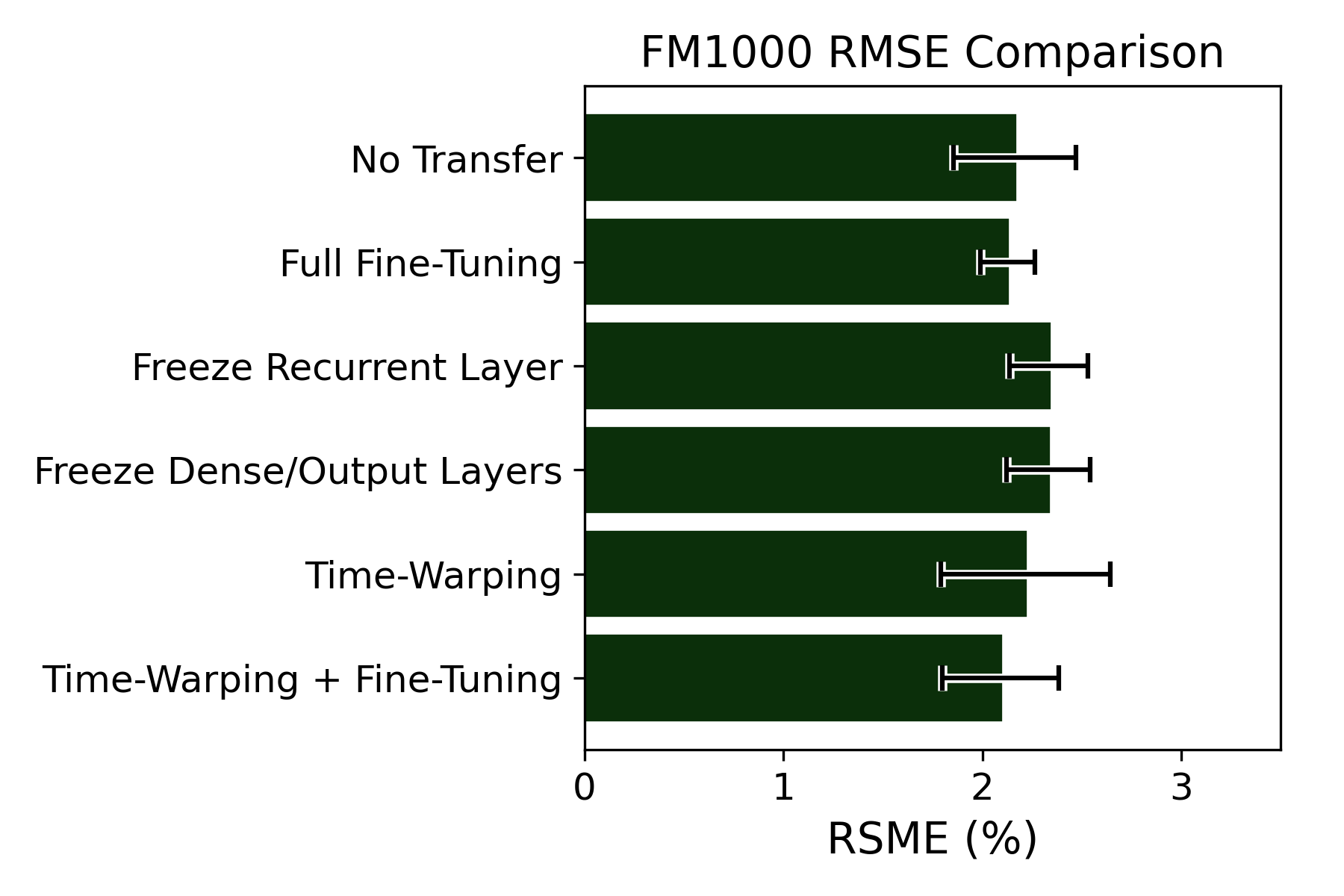}
\caption[FM1000 RMSE by Transfer Learning Method.
]{FM1000 RMSE comparison for different transfer learning methods. Bars represent the mean RMSE when predicting FM1000 on the test set over 100 replications, and brackets show $\pm$ 1 standard deviation of the RMSE. }
\label{fig:fm1000_rmse} 
\end{figure}

The general prediction accuracy was similar for FM1000 and FM100, and the accuracy was better than for FM1. While the estimated $R^2$ values are worse than for the FM100 case, the RMSE values are consistent. The Time-Warping method had an estimated RMSE of $2.22\pm0.43$, which is within the uncertainty bounds for the most accurate transfer learning method, which was Time-Warping plus Fine-Tuning. The estimated bias of $-0.81\pm0.71$ provides evidence that the time-warped models are generally overpredicting FM1000. The $R^2$ of $0.54\pm0.18$ is consistent with the $R^2$ from the Nelson model, again noting the differences in the time period of evaluation.

Figure~\ref{fig:fm1000_scatter} shows the relationship between the observed and predicted FM1000 values. The figures use the FM1000 predictions from a single RNN selected from the 100 trained models, specifically the model whose RMSE was closest to the median RMSE across all replications. We display a single representative model to illustrate the behavior that can be expected from an individual model prediction. The scatterplot shows that the RNN, as in the FM1 case, tends to overestimate FM1000 for drier fuels and underestimate FM1000 for wetter fuels, though the observations of FM1000 greater than $15\%$ are very sparse. Overall, the plot shows moderately strong agreement between the RNN predictions and the observations.

\begin{figure}[htbp]
\centering
\includegraphics[width=14 cm]{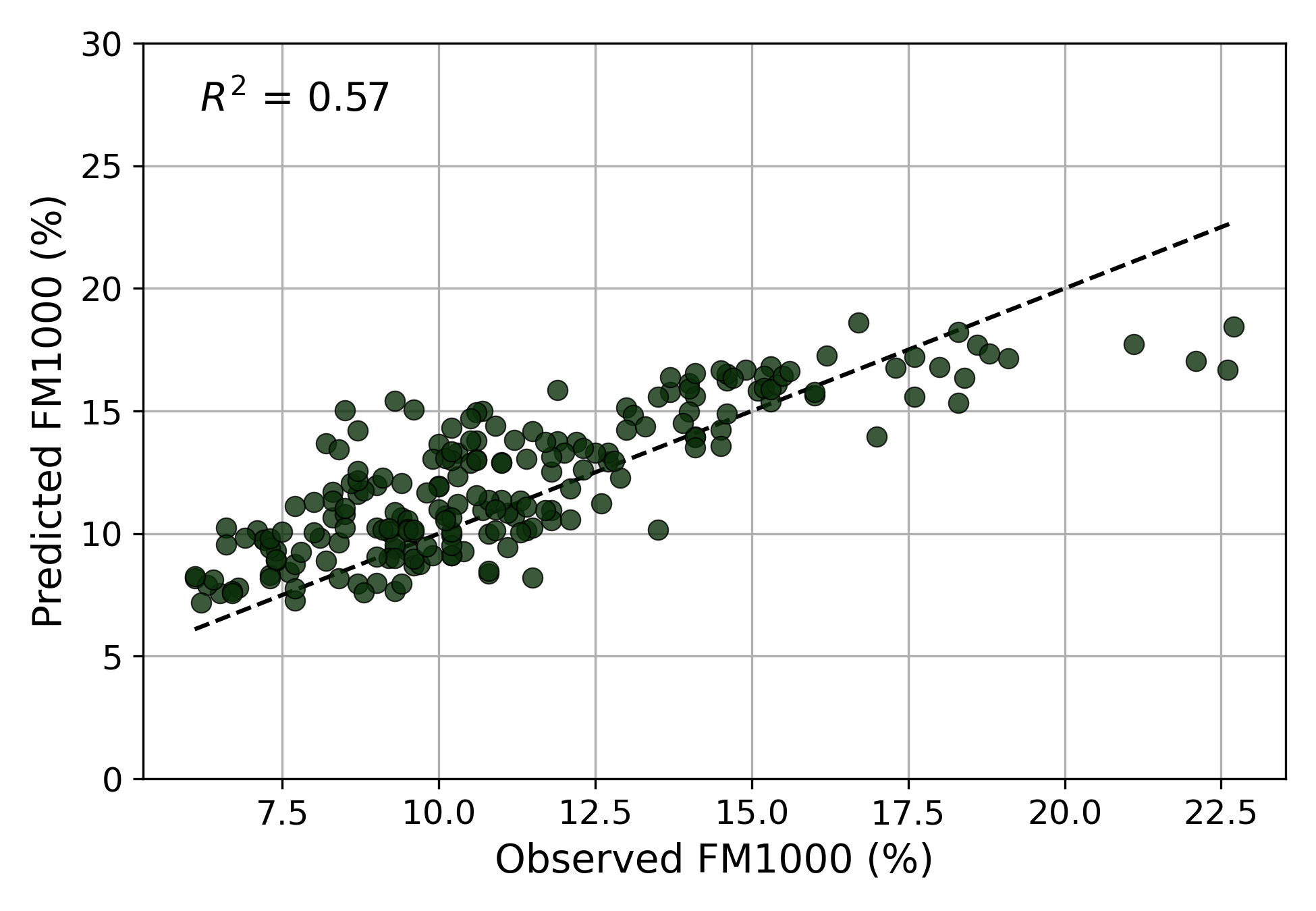}
\caption[Observed vs Predicted FM1000 for Time-Warped RNN
]{Observed vs Predicted FM1000 for Time-Warped RNN. The plot includes all FM1000 observations in the test set~(n=209). The plotted RNN model is the set of weights that corresponds to the median RMSE on the test set out of 100 replications. }
\label{fig:fm1000_scatter} 
\end{figure} 

\clearpage

\subsection{Time-Warping Effect on Learned Dynamics Results}

This section analyzes the learned dynamics of the RNN after applying the Time-Warping transfer learning method, with no fine-tuning. We analyze 100 replications of fitting the time-warping parameters to the data, while holding all other model parameters constant. 

The Time-Warping method led to accurate predictions on the test set for FM1, FM100, and FM1000. In this section, we wish to provide statistical evidence that the temporal dynamics have been changed. We analyze the temporal dependencies of the RNN predictions during the test period for FM1, FM100, and FM1000 using several techniques. The FM1 predictions show much faster change over time than FM100 or FM1000. The two larger fuel classes have similar temporal dependencies. 

Figure~\ref{fig:tscales} plots time series of predictions for FM1, FM100, and FM1000 over 72 hours during the first week of the testing period. The predictions are taken from the replication whose RMSE is closest to the median RMSE across the 100 runs. This choice avoids averaging across replications, which can artificially smooth the predictions and obscure representative model behavior. The FM1 predictions are much more variable throughout the day as the predicted FM1 responds quickly to environmental changes. The FM100 and FM1000 predictions exhibit less variability and slower change over the diurnal cycle. The dynamics for the FM100 and FM1000 are visually similar, though the average FM100 value is wetter than the average FM1000 value over the time period. 

\begin{figure}[htbp]
\centering
\includegraphics[width=14 cm]{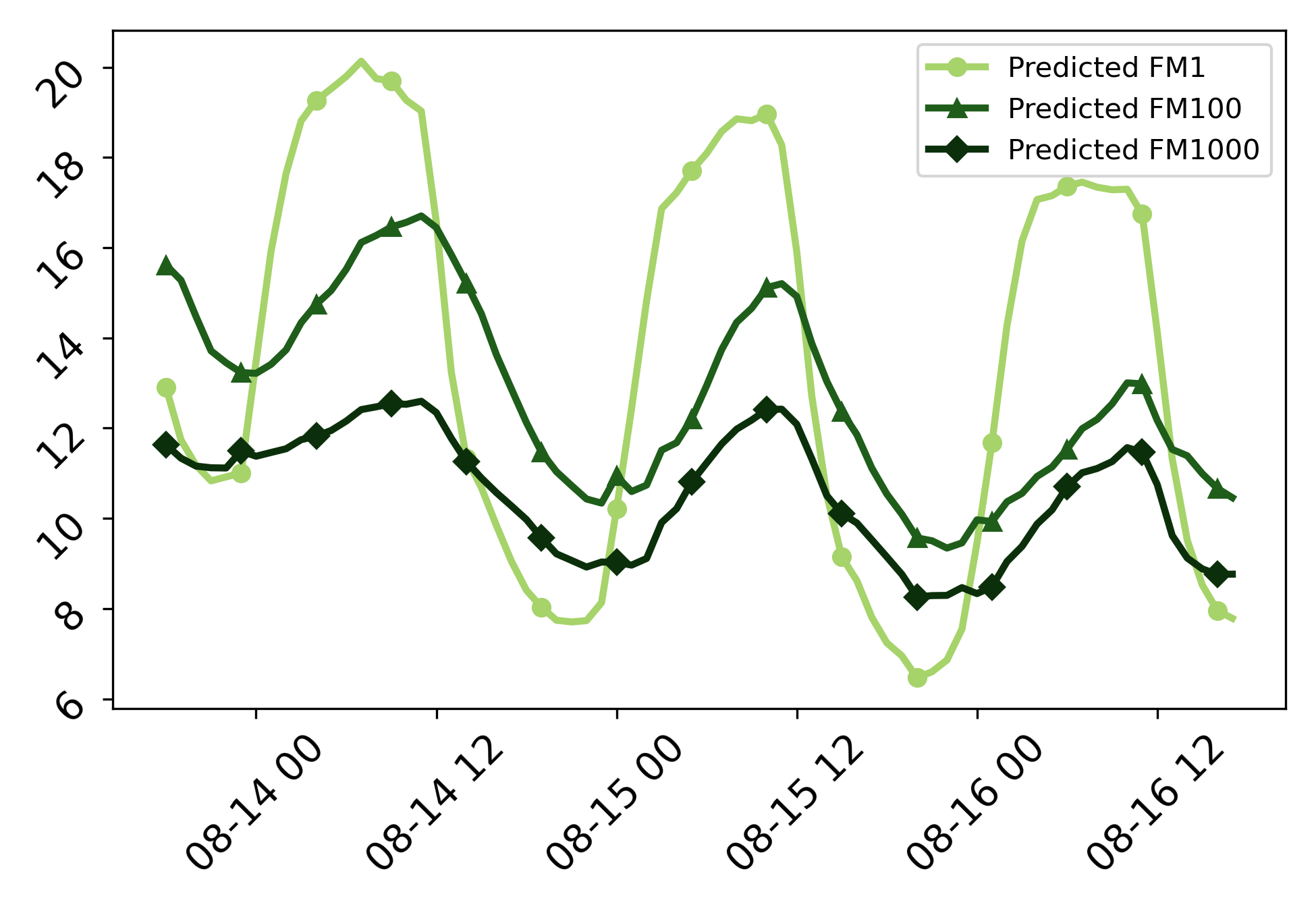}
\caption[Example Time Series of Predictions for FM1, FM100, and FM1000 from the Time-Warped RNN
]{Example Time Series of Predictions for FM1, FM100, and FM1000 from the Time-Warped RNN, with no fine-tuning. The time period plotting is 72 hours during the first week of the test period in August 1997. The predictions are from the models that had the median RMSE across the 100 replications. }
\label{fig:tscales} 
\end{figure} 

Tables~\ref{tab:fm1_timewarp_params}, \ref{tab:fm100_timewarp_params}, and \ref{tab:fm1000_timewarp_params} summarize the best-fit time-warping parameters across 100 replications. For each replication, the two parameters were used to shift the forget and input gate biases, then the FMC was predicted on the training set, and then we selected the parameter combination that led to the best accuracy on the training set. The tables summarize the mean, standard deviation, median, minimum, and maximum values across the replications. The values for the forget gate time-warping parameters show relative convergence for the FM1 and FM100 cases. For FM1000, the parameters are more variable but still show a relatively consistent pattern across the replications for the forget gate bias time-warping parameter. The values for the input gate time-warping parameters are more variable and don't exhibit clear signs of convergence. 

\begin{table}[htbp]
\caption[FM1 Time-Warp Parameters Summary.]{Summary statistics for the time-warp parameters across 100 replications for FM1 with no fine-tuning.}
\label{tab:fm1_timewarp_params}
\centering
\begin{tabular}{lccccc}
\hline
Time-Warping Parameter & Mean & Std & Median & Low & High \\
\hline
$b_f$ shift & $-1.31$ & $0.50$ & $-1.25$ & $-3.33$ & $-0.42$ \\
$b_i$ shift & $1.82$  & $2.19$ & $1.67$  & $-1.67$ & $5.00$ \\
\hline
\end{tabular}
\end{table}

For the FM1 case, the time-warping parameter for the forget gate was $-1.31\pm0.50$ across the 100 replications, and the parameter was negative for all replications. This shows signs of convergence across 100 replications, and it matches the theoretical expectation for the time-lag system. Decreasing the forget gate bias reduces the retention of the previous state, effectively reducing the memory of the system. The time-warping parameter for the input gate was $1.82\pm2.18$, with a median value of $1.67$ indicating that most of the 100 replications resulted in a positive input gate bias shift. The mean and median values are consistent with the theoretical expectation that the dependence on new information should increase when time-warping from an FM10 to an FM1 model. However, the standard deviation is large relative to the mean, and some of the replications resulted in negative values for the parameter. 

\begin{table}[htbp]
\caption[FM100 Time-Warp Parameters Summary.]{Summary statistics for the time-warp parameters across 100 replications for FM100 with no fine-tuning.}
\label{tab:fm100_timewarp_params}
\centering
\begin{tabular}{lccccc}
\hline
Time-Warping Parameter & Mean & Std & Median & Low & High \\
\hline
$b_f$ shift & $1.89$ & $0.61$ & $2.08$ & $0.42$ & $3.33$ \\
$b_i$ shift & $-1.01$ & $2.06$ & $-1.25$ & $-4.17$ & $4.17$ \\
\hline
\end{tabular}
\end{table}

For the FM100 case, the time-warping parameter for the forget gate was $1.89\pm0.61$ across the 100 replications, and the parameter was positive for all replications. Like the FM1 case, this shows signs of convergence across 100 replications, and it matches the theoretical expectation for the time-lag system. Increasing the forget gate bias increases the retention of the previous state, effectively increasing the memory of the system. The time-warping parameter for the input gate bias had a negative mean of $-1.01$ and a negative median of $-1.25$, which are consistent with the theoretical expectation that the dependence on new information should decrease when time-warping from an FM10 to an FM10 model. But again, the standard deviation was large relative to the mean and median, and there were replications where the parameter was positive. 

\begin{table}[htbp]
\caption[FM1000 Time-Warp Parameters Summary.]{Summary statistics for the time-warp parameters across 100 replications for FM1000 with no fine-tuning..}
\label{tab:fm1000_timewarp_params}
\centering
\begin{tabular}{lccccc}
\hline
Time-Warping Parameter & Mean & Std & Median & Low & High \\
\hline
$b_f$ shift & $1.15$ & $1.46$ & $1.25$ & $-4.17$ & $3.75$ \\
$b_i$ shift & $-1.62$ & $3.81$ & $-4.17$ & $-5.00$ & $5.00$ \\
\hline
\end{tabular}
\end{table}

The same general patterns were seen for the FM1000 case, but with greater variability and less clear signs of convergence in the parameters. The mean time-warping parameters were $1.15$ for the forget gate and $-1.62$ for the input gate, which match theoretical expectations of increasing memory retention and decreasing the dependence on new information. But the parameters exhibit high variability relative to the mean and median, and in some replications the parameters have the opposite signs of the theoretical expectations. 

Generally, the signs of the time-warping parameters match the theoretical expectations. The FM1 case resulted in decreasing the forget gate bias and increasing the input gate bias. The FM100 and FM1000 cases showed the opposite. This aligns with the idea that the FM1 model was time-warped to have a shorter memory, and the FM100 and FM1000 models were time-warped to have longer memories. But for the FM1000 case, the forget gate time-warping parameter was highly variable. The input gate time-warping parameter was highly variable in each case. 

Next, we analyze the ACF and PACF of the predictions in the test set. For these plots, we use a single model instead of aggregating the 100 replications, since averaging predictions across replications can smooth the autocorrelation structure and obscure the temporal dependencies present in an individual model. We therefore select the replication with the median test set RMSE for each of the FM1, FM100, and FM1000 cases. The plotted values represent the hourly predictions from the RNNs themselves, not the observed values. The observed FMC values are too sparse in time to analyze the ACF plots at short time scales. Figure~\ref{fig:acf_all} shows the three ACF plots from lags zero to 50, and Figure~\ref{fig:pacf_all} shows the corresponding PACF plots from lags zero to 30. 

\begin{figure}[htbp]
\centering
\includegraphics[width=\textwidth,height=0.28\textheight,keepaspectratio]{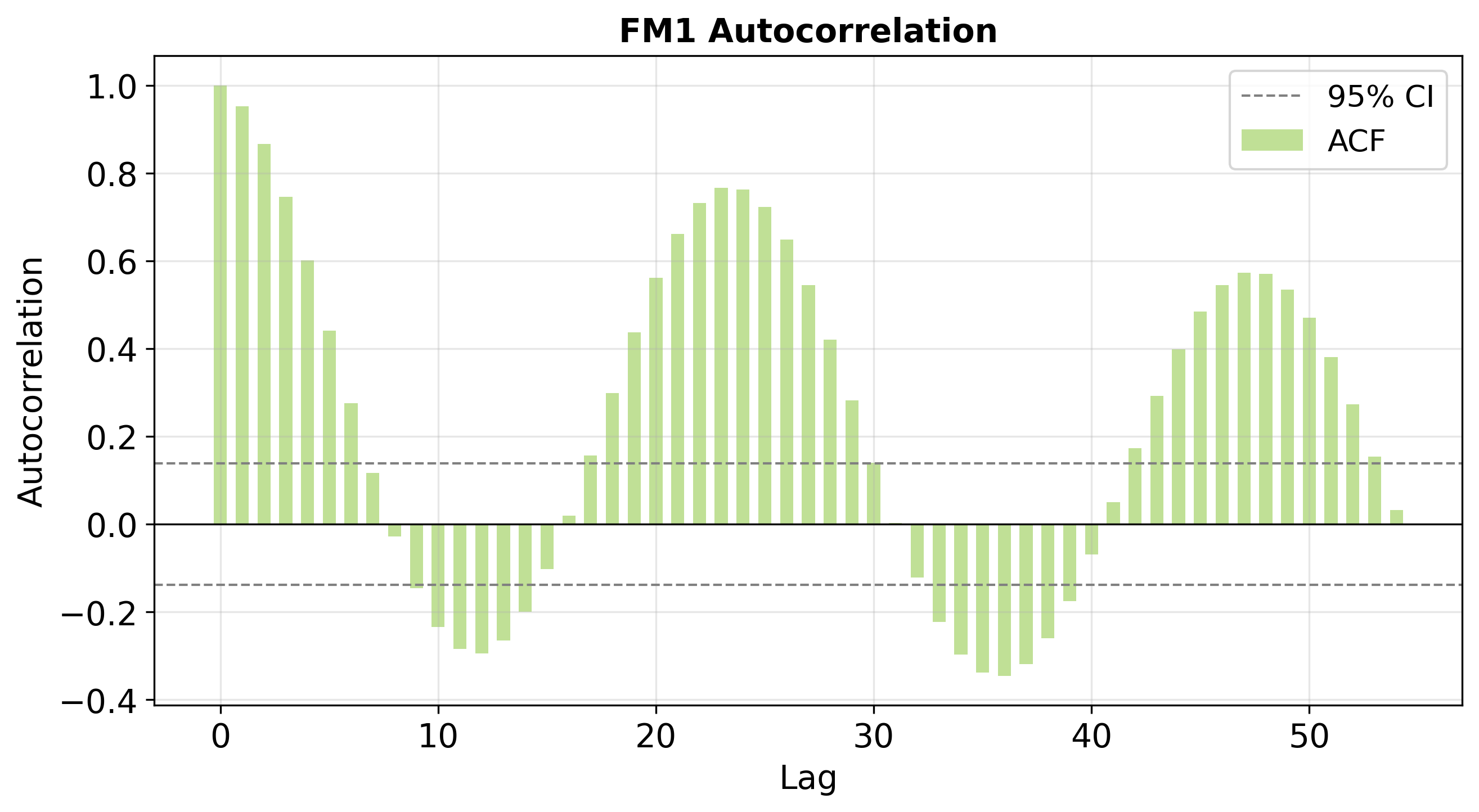}
\vspace{0.6em}
\includegraphics[width=\textwidth,height=0.28\textheight,keepaspectratio]{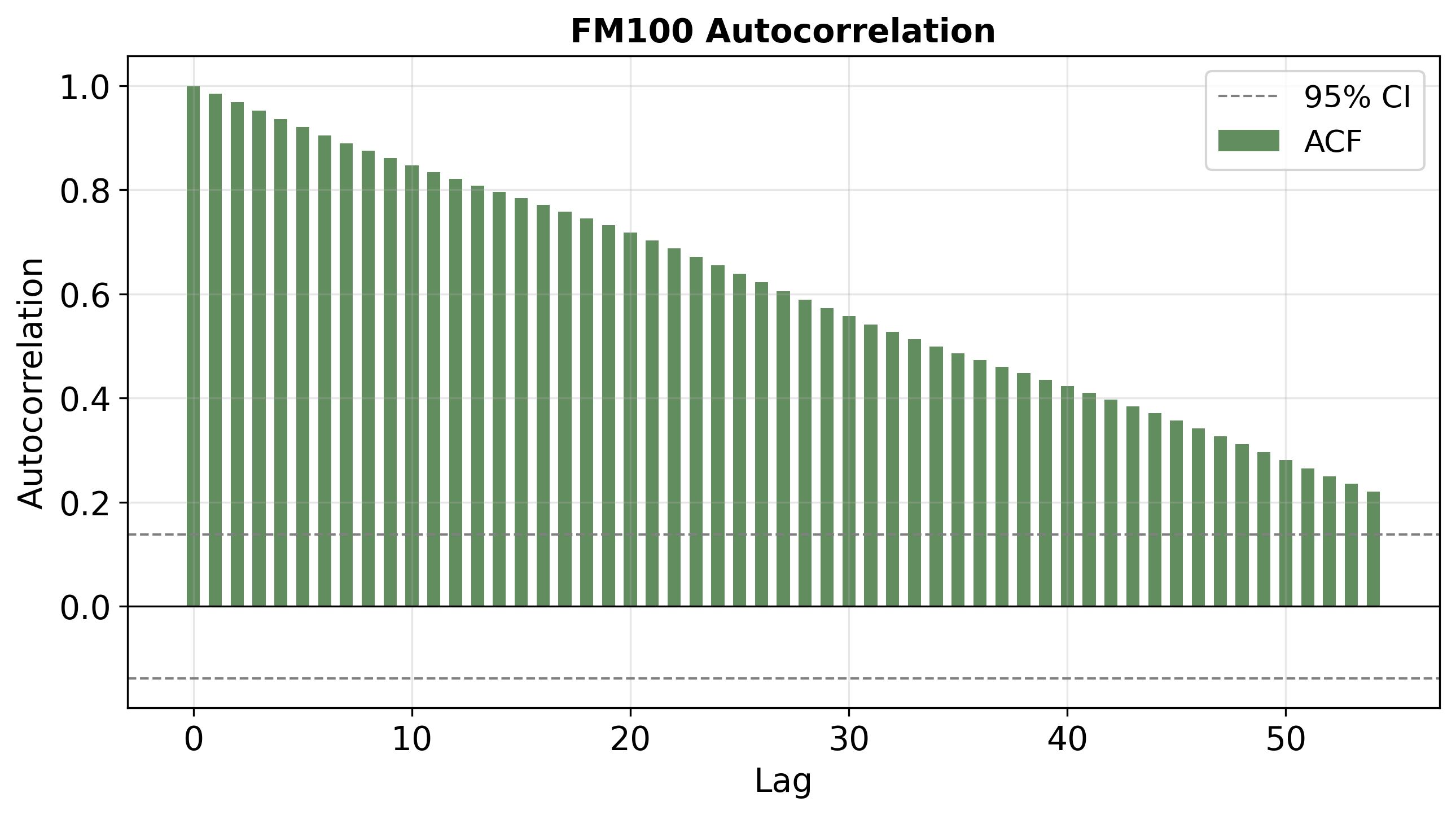}
\vspace{0.6em}
\includegraphics[width=\textwidth,height=0.28\textheight,keepaspectratio]{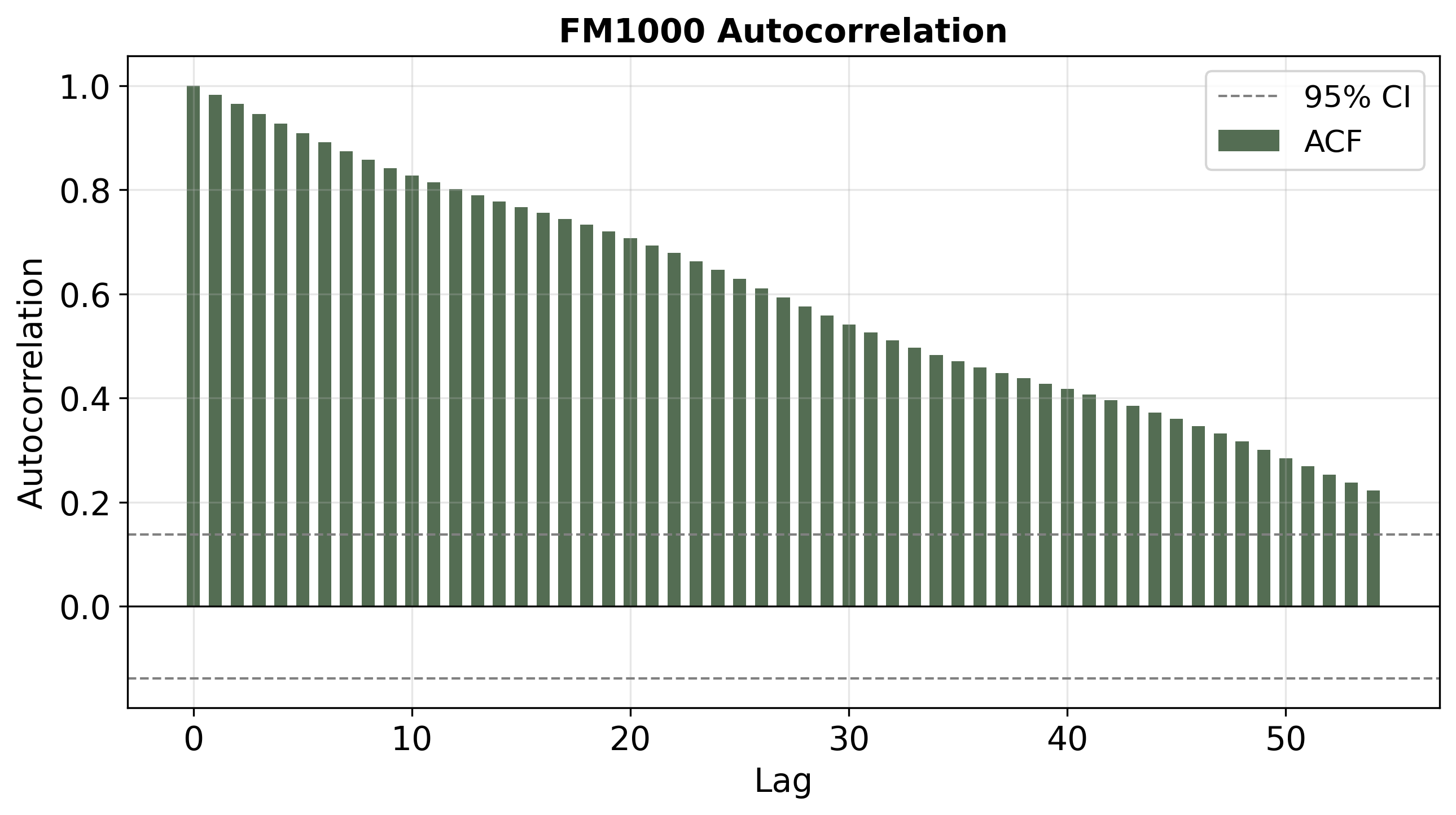}
\caption[ACF of FM1, FM100, and FM1000]{Autocorrelation functions for FM1 (top), FM100 (middle), and FM1000 (bottom) from predictions on the test set. The plots represent the ACF for the single models that had the median RMSE on the test set. The ACF for FM1 decays relatively quickly and has a visible diurnal periodicity, while FM100 and FM1000 show much slower decay consistent with stronger persistence in the system.}
\label{fig:acf_all}
\end{figure}

\begin{figure}[htbp]
\centering
\includegraphics[width=\textwidth,height=0.28\textheight,keepaspectratio]{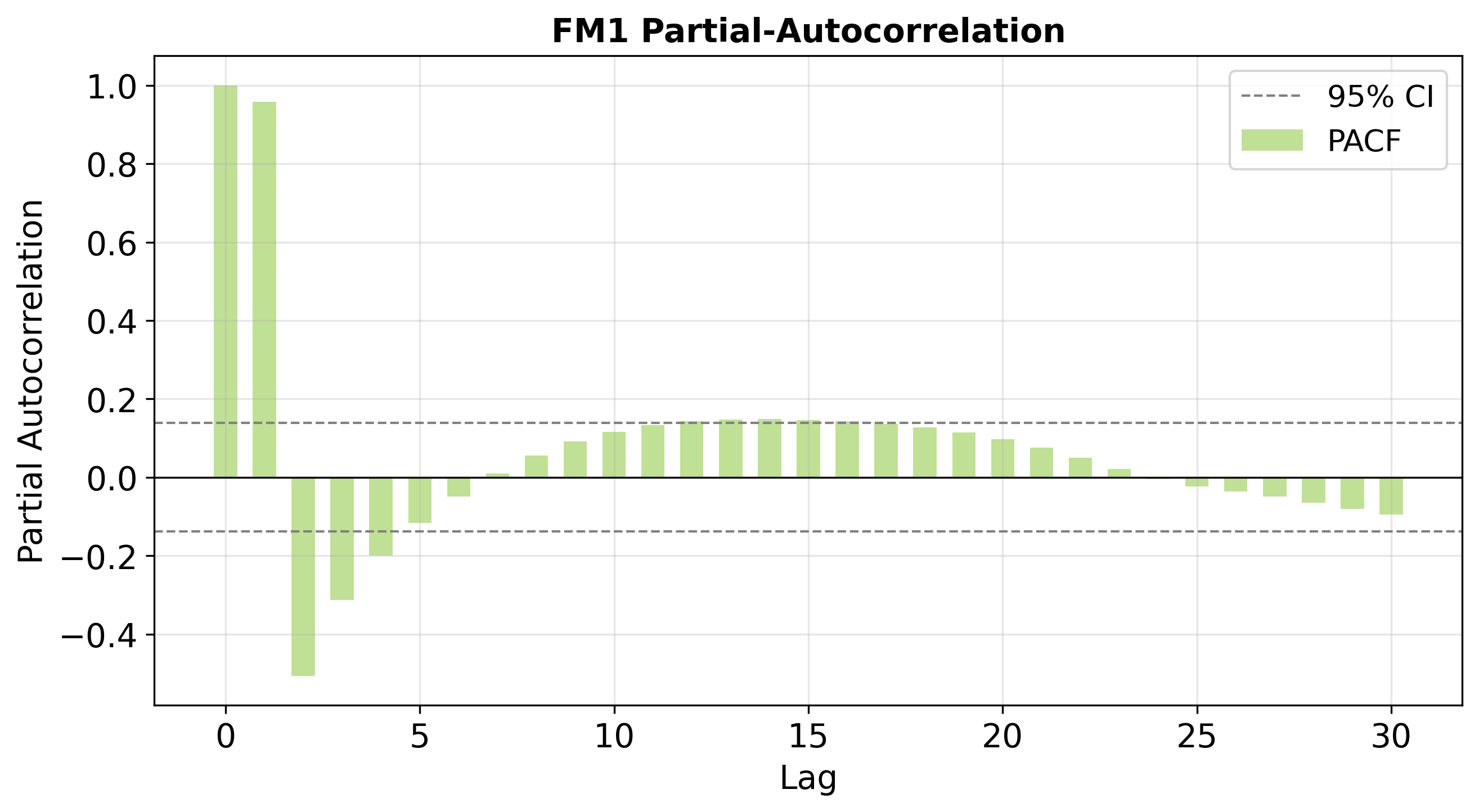}
\vspace{0.6em}
\includegraphics[width=\textwidth,height=0.28\textheight,keepaspectratio]{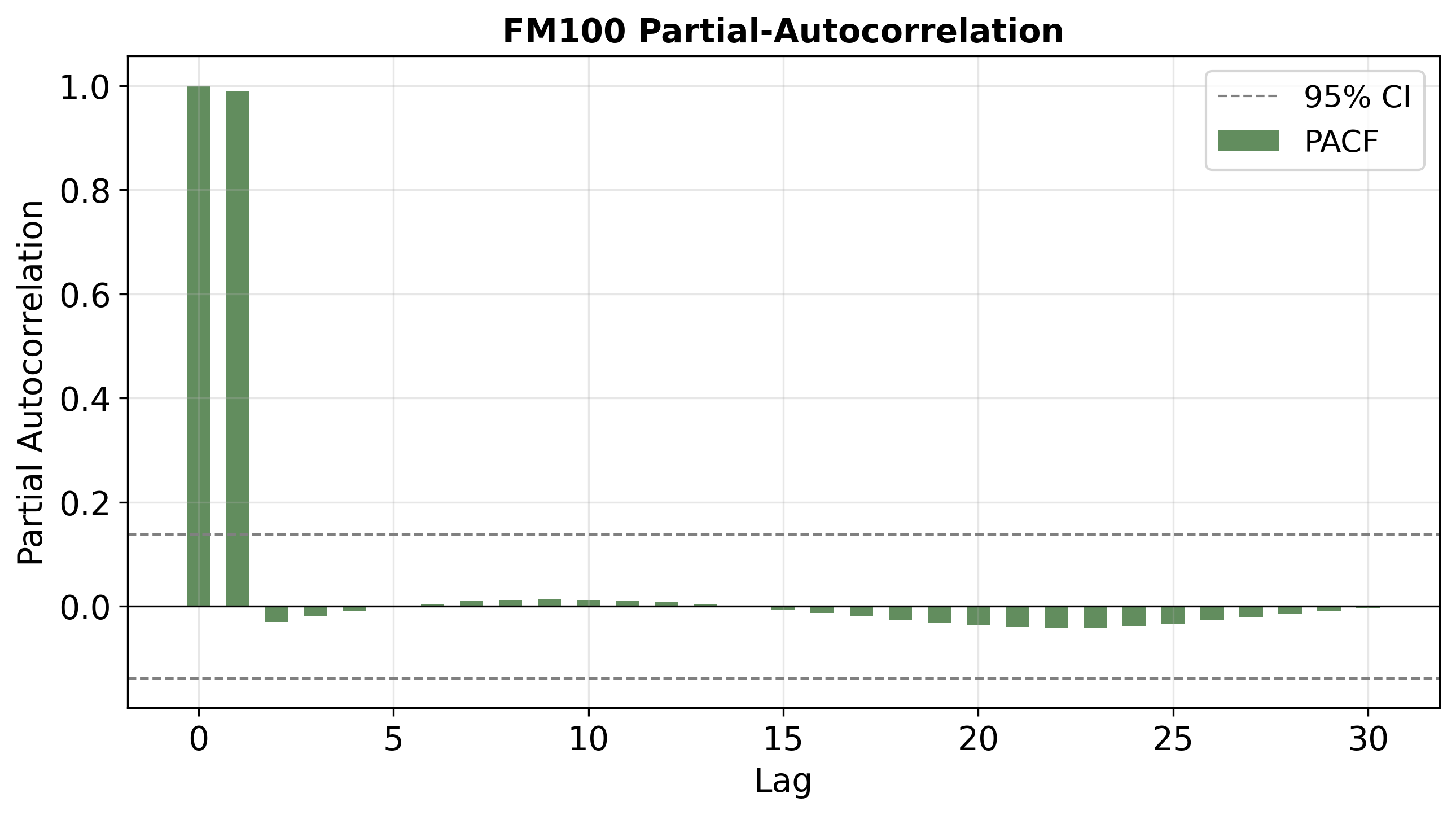}
\vspace{0.6em}
\includegraphics[width=\textwidth,height=0.28\textheight,keepaspectratio]{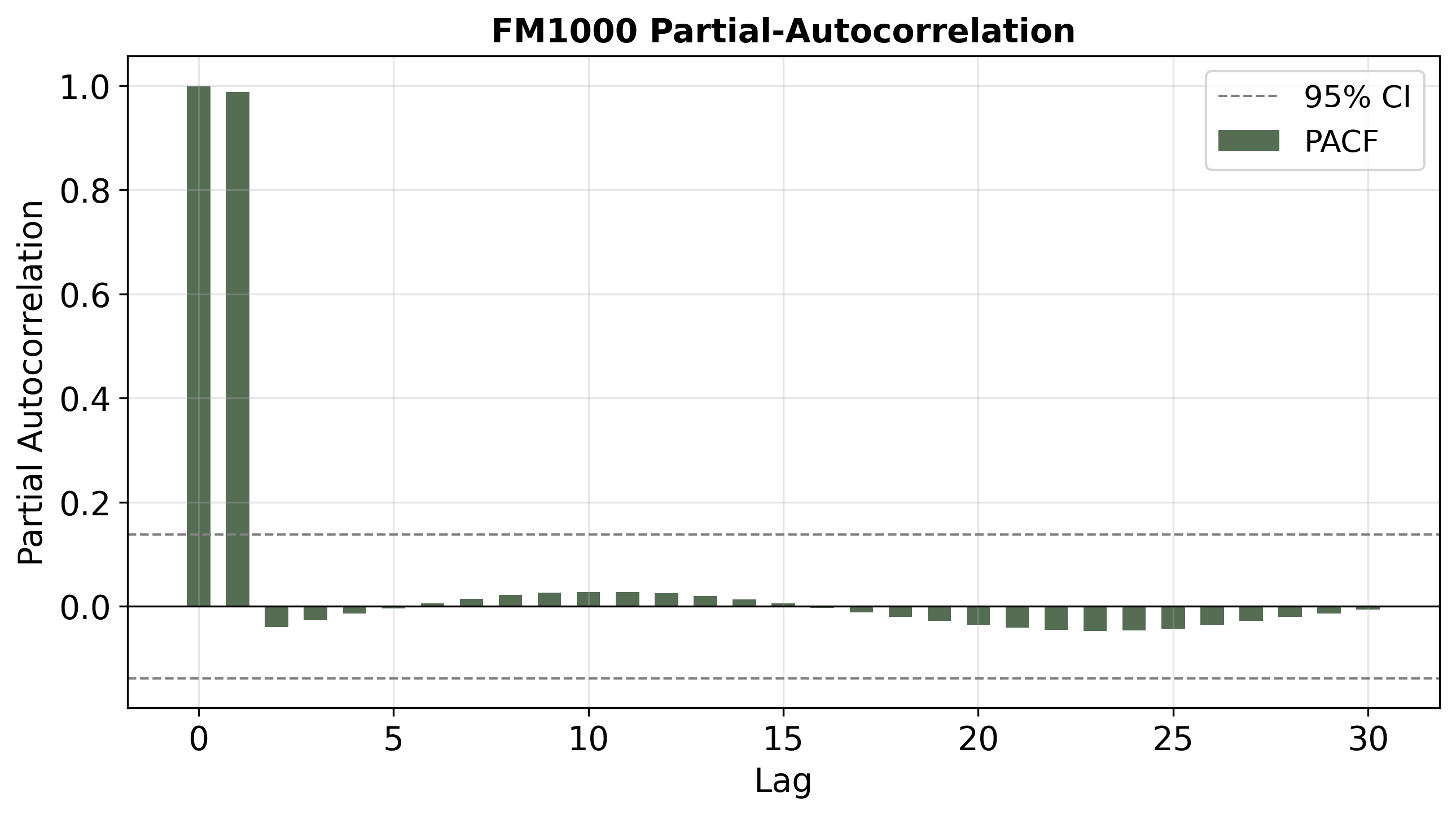}
\caption[PACF of FM1, FM100, and FM1000]{Partial autocorrelation functions for FM1 (top), FM100 (middle), and FM1000 (bottom) from predictions on the test set. The plots represent the ACF for the single models that had the median RMSE on the test set. FM100 and FM1000 show a dominant lag-1 effect consistent with strong persistence, whereas FM1 shows additional short-lag structure and weaker persistence beyond the first lag.}
\label{fig:pacf_all}
\end{figure}

The ACF for the FM1 predictions has positive autocorrelation for short time lags. It shows a clear cyclical pattern that is indicative of diurnal variability in the weather. The FM1 predictions are negatively correlated every 12 hours and positively correlated every 24 hours. The ACFs for the FM100 and FM1000 predictions are very similar and resemble the theoretical ACF of an AR(1) process. The autocorrelation is strongly positive at short time lags and gradually decays toward zero as the lag increases. There are subtle signs of the diurnal variability seen in the FM1 case, but the dependence on the previous state of the system dominates the temporal structure. By contrast, the FM1 case has less dependence on the previous state, since the ACF drops below the significance threshold after about 7 hours. The strong diurnal pattern in the FM1 ACF appears to be driven by diurnal variability in the weather, since there is no physical reason to expect FMC itself to become negatively correlated over time in the absence of changing environmental conditions.

The PACF for the three predictions shows a strong positive correlation at lag 1, which is expected due to the physical persistence of FMC. For the FM1 predictions, the PACF values at lags 2, 3, and 4 are negative and statistically significant. This negative correlation appears to be driven by external forcing from short-term weather variability. Rain events that sharply increase FM1 contribute to negative autocorrelation at short time lags, while increases in wind and solar radiation dry out fuels and produce similar short-term responses. Thus, both the wetting and drying processes driven by weather can contribute to negative autocorrelation at short time lags. The PACFs for FM100 and FM1000 have a strongly positive value at lag 1 and then immediately drop below the significance threshold. This closely matches what would be expected for the PACF of an AR(1) process. While the cyclical patterns in the three PACF plots are mostly below the significance threshold, they still show structure that would not be expected in a white-noise process with no temporal dependence. The ACF values for a white-noise process are expected to be positive or negative with random frequency. The three PACF plots show stable patterns of oscillatory behavior, so we can conclude that there is evidence for cyclical patterns for the three fuel class predictions. 

The ACFs and PACFs provide strong statistical evidence that the FM1 predictions are more strongly dependent on external forcing from weather. Correspondingly, the FM100 and FM1000 predictions are more strongly dependent on memory persistence.

\clearpage

    \chapter{\uppercase{Conclusions}}\label{chapter:conclusions}
\section{Summary}

A method for modifying the learned temporal dynamics of an RNN was proposed in this thesis. The LSTM is a common RNN architecture, but there were no existing methods to modify its temporal dynamics. Shifting the forget and input gate biases to increase or decrease the associated activation can time-warp an LSTM. We proved that a time-lag system can be approximated to any level of accuracy with an LSTM. Then, we proved that if the system is time-warped, increasing or decreasing the gate bias parameters can approximate the time-warped system to any level of accuracy. The theoretical properties of the approximation error were demonstrated numerically using simulations from a time-lag system. 

The Time-Warping method for transfer learning demonstrated a strong empirical performance for predicting the FMC for fuel classes with sparse observations. First, the results from zero-shot transfer learning for FM10 establish that the pretrained RNN could reliably transfer knowledge from the source learning task with FM10 sensors and weather inputs from numerical weather prediction (NWP) to the target task with gravimetric observations and ground-based weather observation. The accuracy of the zero-shot FM10 predictions was within the estimated uncertainty bounds of the pretrained model on the source task, and was strong relative to the state-of-the-art Nelson model. 

The empirical accuracy results for FM1, FM100, and FM1000 show that the Time-Warping method resulted in comparable accuracy to established transfer learning methods, despite only fitting two parameters to data that shifted 128 out of over 21,000 trainable parameters in the RNN. Although the method of freezing the recurrent layer produced similar accuracy metrics for larger fuels, for FM1 freezing the recurrent layer was substantially less accurate than all other methods. This supports the idea that the temporal dynamics within the LSTM recurrent layer need to be modified for the different fuel classes. 

The values of the time-warping parameters were in some cases highly variable, particularly for the input gate time-warping shift. However, the directionality of the time-warping parameters matched theoretical expectations for all fuel classes in terms of the mean and the median. Further, the forget gate bias time-warping parameter demonstrated consistency across replications. Examination of the ACF and PACF functions of the empirical predictions provides additional evidence that the Time-Warping method successfully modified the temporal dynamics of the pretrained RNN.

The Oklahoma field study had under two years of observations from a single spatial location. This is a highly sparse dataset, considering it is being used to represent the FMC of four different fuel classes for the entire CONUS. As a consequence, any model that is fit to this data risks overfitting. The set of learned parameters could be overfit to one spatial location, and the parameters could also be overfit to one particular time period of weather patterns. This issue applies to both ML models that are trained on the data and to physics-based models that have parameters that are calibrated to the data. With the zero-shot transfer learning for FM10, the full time period from March of 1996 through 1997 was used as a test set. For all other models, we split the data into training and testing sets and, in some situations, utilized a validation set. Splitting the data for proper statistical validation resulted in highly sparse data in time. The test set consisted of just four months, from August of 1997 through December of 1997. 

The sparsity of the observed data means that we focus primarily on the accuracy metrics for the zero-shot FM10 model and for the physics-based transfer learning model, where only the time-warping parameters were fit to data. The zero-shot FM10 model was trained with HRRR weather forecast inputs and with fuel sensor response data. Further, the RNN had never been exposed to data from the location of the Oklahoma field study. While the zero-shot FM10 model performed poorly for very wet fuels due to systematic biases in the fuel sensors, the accuracy metrics were at least as strong as the state-of-the-art Nelson model when focusing on FM10 less than or equal to $30\%$. The Nelson model had been specifically calibrated to observations from the Oklahoma field study, so it is advantaged relative to the RNN. This provides strong evidence that the RNN is learning a set of weights that generalize well. The zero-shot FM10 RNN was still accurate when deployed at a new location using different weather inputs and a different methodology for observing the response (sensors versus gravimetric measurements).

The Time-Warping method used observations from the Oklahoma field study to optimize the time-warping parameters to generate predictions for the other fuel classes. So, there is the possibility that this has been overfit to one spatial location and one particular year of weather. However, the amount of fine-tuning to the particular data set was very small compared to the overall model complexity. The pretrained RNN consisted of over 21,000 parameters that had been trained on data from a different time period and different spatial locations than the target learning task. The time-warping method with no fine-tuning modified just 128 of these parameters: 64 bias terms for the forget and input gates. Further, these 128 parameters were shifted uniformly. These particular parameters were selected for modification based on apriori theoretical considerations, not on post-hoc rationalization after seeing the data from the target task. 

All the other transfer learning methods utilized fine-tuning and fit model parameters to the sparse data set. Overfitting in both space and time remains an issue for the fine-tuned models. The RNN from the source learning task was trained with longitude, latitude, and elevation as static predictors. The FM10 model is therefore performing spatial interpolation using the learned patterns from the training data. The Oklahoma field study was one physical location. Fine-tuning allowed many model parameters to update, and there is no way to verify with data from one spatial location whether this came at the expense of the accuracy of the spatial interpolation. In contrast, for the Time-Warped RNN the input connections from the longitude, latitude, and elevation were unchanged, along with most of the other parameters which could implicitly encode knowledge about the spatial patterns. It is possible that modifying the forget and input gate biases deteriorated the ability of the RNN to perform accurate spatial interpolation, but it is more likely to generalize than the other transfer learning methods that fit many parameters to data from one location. The accuracy of the Time-Warping method was consistent with the transfer learning baselines, and there is reason to suspect that the accuracy will generalize to new locations and times better than the fine-tuning methods. 

\section{Contributions}

To the best of our knowledge, this thesis contributes in several ways to time series modeling with RNNs and specifically to the FMC literature. 

This thesis provides the first case of time-warping, as defined by modifying the learned dynamics of a model, being framed as a transfer learning problem. Connecting the traditional concepts of time-warping from dynamic systems to the modern study of transfer learning is an important step towards understanding how knowledge gained from ML models can be generalized across different temporal scales. 

This thesis builds on existing research about how RNNs, and LSTMs specifically, learn parameters that are related to a characteristic time scale. The mathematical propositions proved that there exist situations when the characteristic time scale of an LSTM can be time-warped with arbitrary accuracy by modifying the forget and input gate biases. This was validated numerically with the demonstration using Newton's Law of Cooling. The propositions were proven through restricting the true dynamical system to a class of linear first-order models, and further that the weights of the LSTM were initialized in a way that matches the dynamic system. The mathematical propositions therefore do not guarantee that this method will work for an arbitrary trained LSTM. Nonetheless, identifying a class of dynamical systems where the LSTM can provably be time-warped is a novel result. While the specific values of the time-warping parameters were highly variable for the FMC transfer learning study, the parameters directionally matched theoretical expectations. It is rare within the neural network literature to be able to develop an apriori expectation of what the value of a parameter should be post-training. This thesis develops an expectation of what should happen to specific bias parameters in an LSTM layer, then demonstrates the validity of this apriori expectation in an operationally useful applied setting. 

Next, this thesis provides the first method to post-hoc modify the learned dynamics of a pretrained RNN. We demonstrated empirically that the Time-Warping method performs well compared to established transfer learning methods despite modifying a small fraction of the trainable parameters of the network relative to the other methods. Therefore, there are theoretical reason to believe that the level of accuracy demonstrated by the Time-Warping transfer learning method will generalize to new spatial locations and times. And the established transfer learning methods, while in some instances modestly more accurate, are more likely to have been overfit to a single location over a short period of time, and thus there is reason to suspect that the level of accuracy may not generalize as well. 

For the field of FMC modeling, the Time-Warping method has potential to contribute to the literature and to operational systems. Modern ML methods are able to leverage the large quantities of available FM10 sensor observations to provide accurate forecasts for that particular fuel class, but currently there is limited success in ML applications to forecast the other fuel classes that are sparsely observed. Transfer learning has the capacity to greatly improve the accuracy of the FM1, FM100, and FM1000 forecasts. Further, the Time-Warping method provides the possibility of generating FMC forecasts for sticks of arbitrary sizes by producing zero-shot predictions with different time-warping parameters. The traditional physics-based models are capable of doing this, but the modern ML models of FMC have not allowed for that possibility until now. 

\section{Future Research}\label{sec:future}

Future research will use FMC sampling from different spatial locations to evaluate the predictions of FM1, FM100, and FM1000. The FEMS field samples were not utilized in this thesis due to their low temporal resolution: at most every two weeks versus twice daily for the Oklahoma field study. Since this thesis focused on time-warping, we concluded that the Oklahoma field study was preferable. Future research could utilize the FEMS field studies, which include observations from across CONUS. The rigorous spatiotemporal cross validation method used to evaluate the FM10 RNN could be applied to the time-warped RNN using the FEMS field studies. A time-warped version of the RNN could be compared to the state-of-the-art Nelson model at a broad set of physical locations. This type of validation accounts for spatial uncertainty, but it was not possible in this thesis due to the sparsity of the Oklahoma field study and lack of other controlled studies. 

The other controlled field experiments of observing FMC that we were able to locate were not in CONUS. The Weiss study was in Hawaii~\citep{Weise-2025-DFM}, and other studies have been conducted in Europe, Australia, and East Asia. The pretrained RNN used in this thesis utilized longitude and latitude as predictors to help the model to learn spatial patterns. The spatial domain was locations across CONUS. Combining the longitude and latitude from CONUS with coordinates from distant locations into the same ML training framework raises practical issues related to the curvature of the Earth. One possibility is to retrain the RNN without longitude and latitude. In the paper on the FM10 RNN, the model was slightly less accurate when trained without longitude and latitude as predictors, but it was still relatively accurate. Removing the spatial coordinates could allow for utilizing data from around the world~\citep{Hirschi-2026-RNN}.

The Time-Warping method could be applied to other physical systems in future research. This thesis examines the static and homogeneous time-warping case, meaning that the forget and input gate biases are shifted by a scalar value for all times. A different temporal transfer learning task may necessitate using time-varying or heterogeneous time-warping parameters. That is, the forget and input gate bias terms could be shifted with different values for different LSTM units. And the bias parameters could be shifted by different amounts at different times. This would modify the learned temporal dynamics in a much more complicated way, and would require extensive evaluation to find the right time-warping functions. This thesis provides a theoretical foundation for these more complicated methods, which could be useful for modeling different types of dynamic systems.

\renewcommand\bibname{REFERENCES}
\singlespacing

\nocite{*}
\bibliographystyle{abbrvnat}
\bibliography{references.bib}

\doublespacing
    \ucdappendix
\chapter{\uppercase{Abbreviations and Notation} \label{appendix-notation}}


\begin{description}
\item[\textbf{ML}] Machine Learning
\item[\textbf{RNN}] Recurrent Neural Network
\item[\textbf{LSTM}] Long Short-Term Memory
\item[\textbf{ODE}] Ordinary Differential Equation
\item[\textbf{CNN}] Convolutional Neural Network
\item[\textbf{NLP}] Natural Language Processing
\item[\textbf{RMSE}] Root Mean Squared Error
\item[\textbf{FMC}] Fuel Moisture Content
\item[\textbf{FM1}] 1-hr Fuel Moisture Content
\item[\textbf{FM10}] 10-hr Fuel Moisture Content
\item[\textbf{FM100}] 100-hr Fuel Moisture Content
\item[\textbf{FM1000}] 1000-hr Fuel Moisture Content
\item[\textbf{RAWS}] Remote Automated Weather Station
\item[\textbf{FEMS}] Fire Environment Mapping System
\item[\textbf{NWP}] Numerical Weather Prediction
\item[\textbf{HRRR}] High-Resolution Rapid Refresh
\item[\textbf{CONUS}] Continental United States
\item[\textbf{NFDRS}] National Fire Danger Rating System
\item[\textbf{ACF}] Autocorrelation Function
\item[\textbf{PACF}] Partial-Autocorrelation Function
\item[\textbf{AR($p$)}] Autoregressive Process of Order $p$
\end{description}



\chapter{\uppercase{Additional appendix} \label{appendix-thm}}

\section{Extra Tables and Figures}\label{app:extra}

In this section, we include additional figures that did not fit in-line in the main text of the thesis. 

\begin{figure}[htbp]
    \centering
    \includegraphics[width=\textwidth, height=0.8\textheight, keepaspectratio]{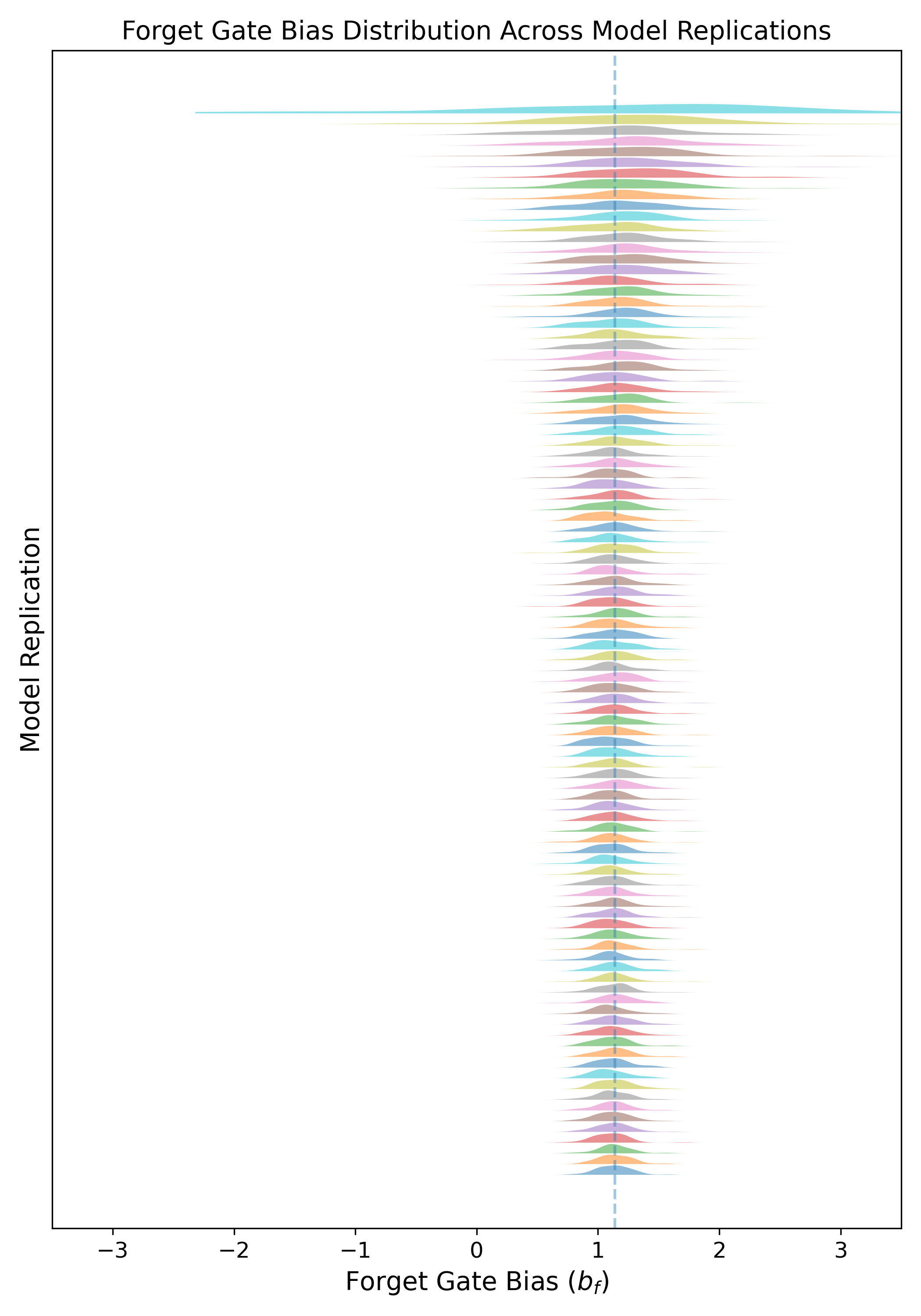}
    \caption[Forget Gate Bias Distribution across Replications]{Forget Gate Bias Distribution across Replications. The distributions are visualized using Gaussian kernel density estimates to provide a smooth representation of the empirical bias distributions for comparison, and the replications are ordered by the sample variance of each distribution.}\label{fig:bf_hist}
\end{figure}

\begin{figure}[htbp]
    \centering
    \includegraphics[width=\textwidth, height=0.8\textheight, keepaspectratio]{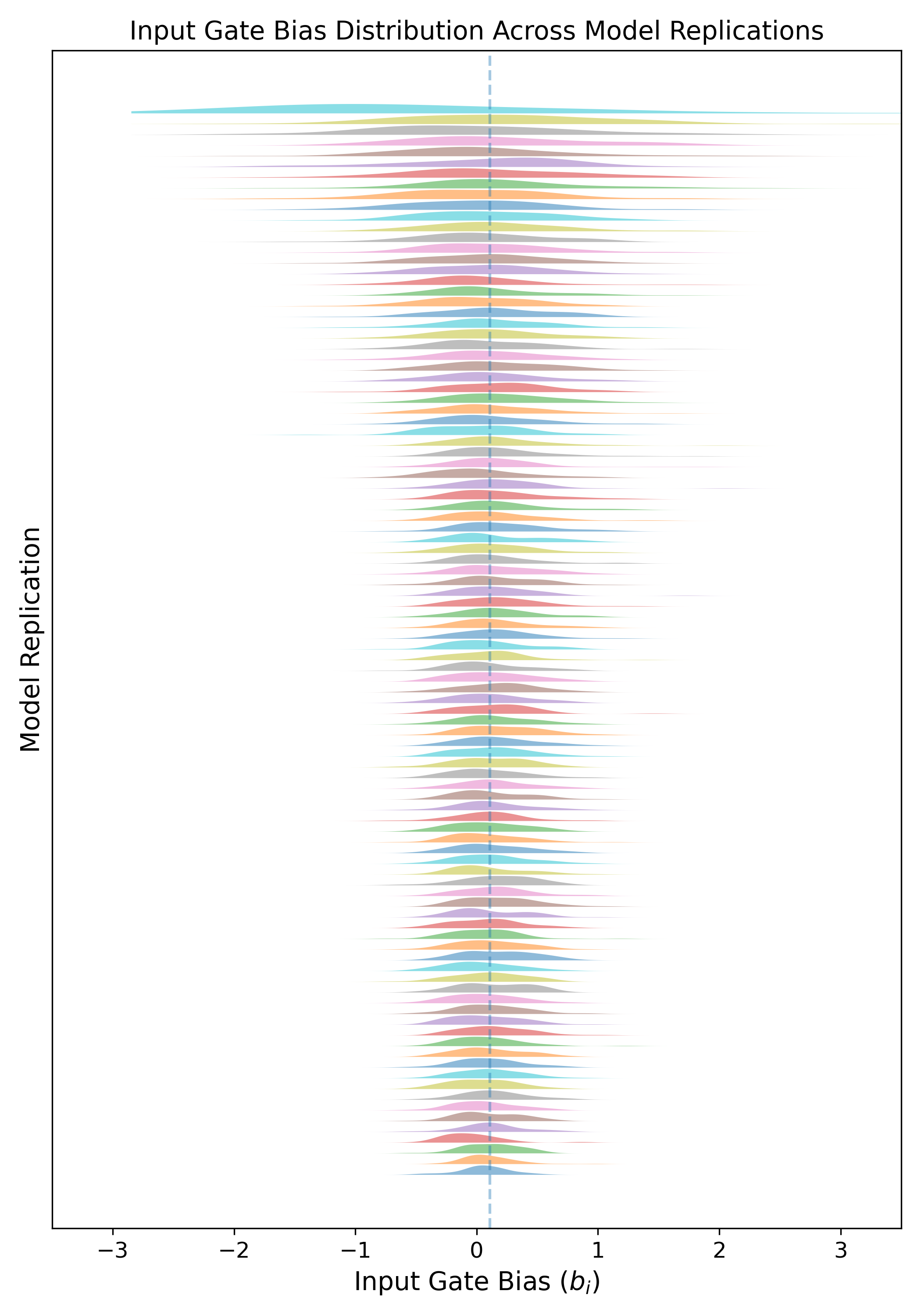}
    \caption[Input Gate Bias Distribution across Replications]{Input Gate Bias Distribution across Replications. The distributions are visualized using Gaussian kernel density estimates to provide a smooth representation of the empirical bias distributions for comparison, and the replications are ordered by the sample variance of each distribution.}\label{fig:bi_hist}
\end{figure}

\begin{figure}[htbp]
    \centering
    \includegraphics[width=\textwidth, height=0.8\textheight, keepaspectratio]{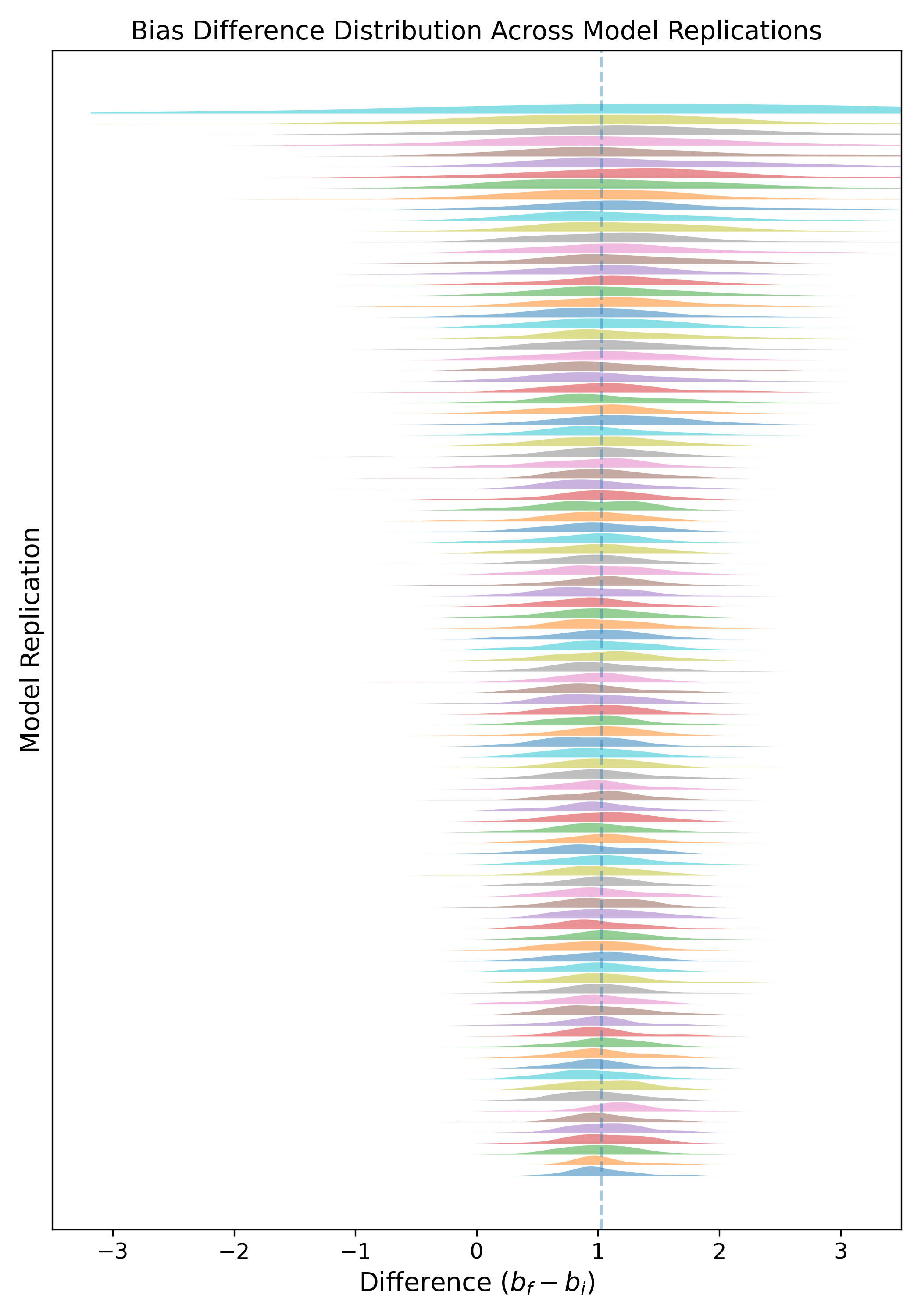}
    \caption[Per-Unit Difference of Forget and Input Gate Biases across Replications.]{Per-Unit Difference of Forget and Input Gate Biases across Replications. The distributions are visualized using Gaussian kernel density estimates to provide a smooth representation of the empirical bias distributions for comparison, and the replications are ordered by the sample variance of each distribution.}\label{fig:diff_hist}
\end{figure}

\clearpage

Two static regression models were fit on the training and validation periods and used to predict FMC for each fuel class in the test period. The models are an XGBoost and a simple linear regression, and they include the same set of predictors as the RNN, which can be found in Table~\ref{tab:all_vars}. The test period was from August 13, 1997 at 12:00 UTC through December 30, 1997 at 22:00 UTC. The results are presented in Tables~\ref{tab:fm1_static}, \ref{tab:fm10_static}, \ref{tab:fm100_static}, and \ref{tab:fm1000_static}. These models provide a baseline of accuracy, and help confirm that the more sophisticated modeling approaches are improving accuracy relative to a simple ML model.

\begin{table}[htbp]
\caption[Accuracy metrics for static FM1 models compared to Oklahoma field observations.]{Accuracy metrics for FM1 predictions during the test period from XGBoost and linear regression model (LM). Comparisons of the models are presented for Oklahoma field observations of all FM1 values and for FM1 $\leq 30\%$.}
\label{tab:fm1_static}
\centering
\begin{tabular}{lcc|cc}
\multicolumn{2}{c}{\textbf{FM1: Static Baselines}} \\
\hline
\multicolumn{1}{c}{} 
& \multicolumn{2}{c}{All observed FM1} 
& \multicolumn{2}{c}{Observed FM1 $\leq$ 30\%} \\
\hline
Metric 
& XGBoost & LM 
& XGBoost & LM \\
\hline
$R^2$      & $0.63$ & $0.56$ & $0.35$ & $0.42$ \\
Bias (\%)  & $0.71$ & $-0.05$ & $-0.92$ & $-2.21$ \\
RMSE (\%)  & $8.15$ & $8.81$ & $5.03$ & $4.73$ \\
\hline
\end{tabular}
\end{table}

\begin{table}[htbp]
\caption[Accuracy metrics for static FM10 models compared to Oklahoma field observations.]{Accuracy metrics for FM10 predictions during the test period from XGBoost and linear regression model (LM). Comparisons of the models are presented for Oklahoma field observations of all FM1 values and for FM1 $\leq 30\%$.}
\label{tab:fm10_static}
\centering
\begin{tabular}{lcc|cc}
\multicolumn{2}{c}{\textbf{FM10: Static Baselines}} \\
\hline
\multicolumn{1}{c}{} 
& \multicolumn{2}{c}{All observed FM10} 
& \multicolumn{2}{c}{Observed FM10 $\leq$ 30\%} \\
\hline
Metric 
& XGBoost & LM 
& XGBoost & LM \\
\hline
$R^2$      & $0.56$ & $0.59$ & $0.49$ & $0.59$ \\
Bias (\%)  & $1.31$ & $0.77$ & $0.31$ & $-0.36$ \\
RMSE (\%)  & $5.80$ & $5.62$ & $3.88$ & $3.47$ \\
\hline
\end{tabular}
\end{table}

\begin{table}[htbp]
\caption[Accuracy metrics for static FM100 models compared to Oklahoma field observations.]{Accuracy metrics for FM100 predictions during the test period from XGBoost and linear regression model (LM) compared against field observations. Results are shown for all observed FM100 values.}
\label{tab:fm100_static}
\centering
\begin{tabular}{lcc}
\multicolumn{2}{c}{\textbf{FM100: Static Baselines}} \\
\hline
\multicolumn{1}{c}{} 
& \multicolumn{2}{c}{All observed FM100} \\
\hline
Metric 
& XGBoost & LM \\
\hline
$R^2$      & $0.57$ & $0.53$ \\
Bias (\%)  & $-0.58$ & $-0.51$ \\
RMSE (\%)  & $2.86$ & $3.00$ \\
\hline
\end{tabular}
\end{table}

\begin{table}[htbp]
\caption[Accuracy metrics for static FM1000 models compared to Oklahoma field observations.]{Accuracy metrics for FM1000 predictions from XGBoost and linear regression model (LM) during the test period compared against field observations. Results are shown for all observed FM1000 values.}
\label{tab:fm1000_static}
\centering
\begin{tabular}{lcc}
\multicolumn{2}{c}{\textbf{FM1000: Static Baselines}} \\
\hline
\multicolumn{1}{c}{} 
& \multicolumn{2}{c}{All observed FM1000} \\
\hline
Metric 
& XGBoost & LM \\
\hline
$R^2$      & $0.38$ & $0.44$ \\
Bias (\%)  & $-0.71$ & $-0.64$ \\
RMSE (\%)  & $2.61$ & $2.47$ \\
\hline
\end{tabular}
\end{table}

\clearpage
\section{Data and Code}\label{app:code}

The code associated with this project was built in Python (Version 3.11.14). The main ML software packages are TensorFlow (Version 2.19.1) and Keras (Version 3.13.2). Two publicly-available GitHub repositories form the code base of the project. The GitHub repository \url{https://github.com/openwfm/ml_fmda} contains materials needed to reproduce the data set used in the source learning task and associated results, which were primarily reported in the publication \citet{Hirschi-2026-RNN}. The data used in this project is derived from public domain resources, including the Amazon AWS data set of historical HRRR forecasts (\url{https://registry.opendata.aws/noaa-hrrr-pds/}) and the historical RAWS data from Synoptic, accessed through the software package SynopticPy (\url{https://github.com/blaylockbk/SynopticPy}). The data used in the project was collected, formatted, and stored using the University of Colorado Denver supercomputing cluster Alderaan. The exact data is available upon request due to privacy restrictions of the university computing system and the substantial storage space needed for this particular project.

The GitHub repository \url{https://github.com/jh-206/fmc_transfer} contains the code needed to reproduce the results related to the transfer learning analysis. 

The Oklahoma field study data described in \citet{Carlson-2007-ANM} was personally delivered in an email correspondence by Dr. Van der Kamp. This data is available upon request. The data includes hourly weather observations from a ground-based station in the Oklahoma Mesonet~\citep{OKMeso-2026-MSI}. The weather data from the study region and time were purchased from Mesonet. We received confirmation in an email correspondence from Mesonet that the data can be freely shared with proper citation, so this data is also available upon request. The wetting and drying equilibria were calculated from the air temperature and relative humidity from the Mesonet data. 

The Mesonet weather data was hourly resolution, with reported observations occurring directly on the hour. The associated FMC measurements included times that specified minutes within an hour. We aligned the times for the RNN predictions and accuracy calculations in the following ways. 

Times were converted from central time to UTC. The RNN was pretrained on FM10 observations from RAWS and weather data from HRRR, both of which report data in UTC time format. The hour of the day and the day of the year were used as derived features in the RNN to help the network learn diurnal and seasonal patterns. Thus, the times from the Oklahoma field study needed to be converted to UTC so the associated hour of day and day of the year features were consistent with the training data in the source domain. 

The weather variables from the nearest hour were assigned to the FMC observations for the training set. This introduces a potential misalignment of up to half an hour, but this error is accounted for when predicting in the test set. Temporal interpolation was used on the RNN predictions to align them with the exact times of the reported FMC observations in the test set. The RNN produces hourly predictions of FMC for the different fuel classes. These predictions were interpolated linearly to the exact time of the FMC observations. 

As an example, consider an observation in the training set that was recorded at 12:35 UTC. During training, it would be aligned with the weather at 1:00 UTC. In other words, the RNN prediction at 1:00 UTC would be compared to the FMC observation at 12:35 UTC to calculate a loss. But an observation in the test set that was recorded at 12:35 UTC would be kept as-is. The RNN predictions at 12:00 UTC and 1:00 UTC would be used with linear interpolation to produce a prediction at 12:35 UTC.

\section{RNN Paper}\label{app:paper}

Portions of this thesis are based on the following peer-reviewed publication:

Hirschi, J.; Mandel, J.; Hilburn, K.; Farguell, A. \textit{A Recurrent Neural Network for Forecasting Dead Fuel Moisture Content with Inputs from Numerical Weather Models}. \textit{Fire} 2026, 9, 26. \url{https://doi.org/10.3390/fire9010026}

This paper presents a recurrent neural network framework for forecasting dead fuel moisture content using inputs from numerical weather prediction models. The methods and results developed in this publication form a subset of the broader modeling framework and analysis presented in this thesis. 

    \newpage \ \newpage 
\end{document}